\documentclass[11pt, a4paper]{article}

\usepackage[margin=1in]{geometry}
\usepackage{amsmath}
\usepackage{amssymb}
\usepackage{amsfonts}
\usepackage{bm}

\usepackage{tikz}
\usetikzlibrary{positioning}
\usetikzlibrary{shapes.geometric, arrows.meta, positioning, calc}

\usepackage[hidelinks]{hyperref}

\usepackage{algorithm}
\usepackage{algorithmic}

\usepackage{tcolorbox}

\usepackage{booktabs}

\usepackage{amsthm}
\newtheorem{example}{Example}
\newtheorem{theorem}{Theorem}
\newtheorem{proposition}{Proposition}
\newtheorem{corollary}{Corollary}

\usepackage{fancyhdr}
\usepackage{enumitem}
\usepackage{float}

\title{Neural Bayesian Sequential Routing}
\author{Yongchao Huang\footnote{[Email: yongchao.huang@abdn.ac.uk] The author welcomes any follow-up work, extensions, and adaptations of these ideas. If this manuscript found useful in future research, appropriate citation would be appreciated. It was developed over many days and nights with the aim of providing a self-contained material for open knowledge sharing, although some (many) errors may still remain after careful review.}}
\date{28/03/2026}

\begin{document}

\maketitle

\begin{abstract}
Human decision-making is inherently sequential and uncertainty-aware, yet standard deep neural networks often rely on static, dense forward computation that provides limited visibility into how evidence is acquired, how uncertainty evolves, or when computation should stop. While conditional architectures such as Mixture-of-Experts (MoE) introduce input-dependent computation, traditional soft-routing mechanisms can suffer from expert imbalance or collapse, and typically do not maintain a temporally evolving belief state. To bridge this gap, we introduce \textbf{Neural Bayesian Sequential Routing (NBSR)}, a dynamic framework that models neural inference as an active, evidence-accumulating traversal over a hierarchical Directed Acyclic Graph (DAG). Operating within a Dirichlet-Categorical conjugate framework, specialized neural experts query a persistent global knowledge oracle to extract strictly positive evidence vectors, which act as pseudo-counts and update a Dirichlet belief state via exact conjugate addition. By coupling this Bayesian belief update with a Gumbel-Softmax Straight-Through estimator, NBSR enables hard, path-dependent routing while preserving surrogate gradients for end-to-end training. The resulting Dirichlet precision and entropy provide native mechanisms for uncertainty quantification, entropy-based early exiting, OOD abstention, and cost-aware evidence acquisition. We provide theoretical guarantees showing that, under strictly positive evidence extraction, total Dirichlet precision increases monotonically along any valid trajectory and the marginal predictive variance is correspondingly bounded, formalizing the intended sequential ``hypothesis sharpening'' behavior; under idealized capacity and optimization assumptions, the terminal Dirichlet expectation recovers the Bayes-optimal conditional distribution. Empirical evaluations on visual categorization, structured medical diagnosis, language modeling, partially observable control, and cost-aware Bayesian experimental design show that NBSR achieves competitive predictive performance while yielding transparent routing traces, path-dependent evidence attribution, uncertainty-aware decision control, and resource-rational inference. Ultimately, NBSR provides a mathematically grounded framework for interpretable, modular, and resource-rational agentic AI.
\end{abstract}

\tableofcontents
\newpage

\section{Introduction} \label{sec:intro}

Human decision-making is naturally sequential, distributional, and hierarchical \cite{kahneman2011thinking, wald1945sequential}. When navigating complex environments or diagnosing intricate problems, humans do not evaluate all possible information simultaneously; rather, we selectively accumulate evidence step-by-step \cite{Bonawitz2006just}. Decisions are made at discrete junctures conditioned on both the current contextual state and the trajectory of prior outcomes. As information flows through this mental decision tree, epistemic variance shrinks, and our broad initial hypotheses ``sharpen'' into narrow, confident conclusions. 

In contrast, standard deep learning paradigms typically rely on ``flat'', monolithic architectures that process all input features \textit{simultaneously} to produce a dense probability distribution over all possible decisions \cite{lecun2015deep, he2016deep}. While highly effective at maximizing semantic accuracy, these approaches lack \textit{cognitive plausibility} \cite{collins2024building}, exhibit poor interpretability in decision-auditing scenarios \cite{lipton2018mythos, rudin2019stop}, and often spend unnecessary computational resources on easily classified, unambiguous inputs \cite{graves2016adaptive, han2021dynamic}. 

To introduce conditional computation, architectures such as Mixture of Experts (MoE) \cite{jacobs1991adaptive, shazeer2017outrageously} utilize \textit{dynamic routing} to delegate inputs to specialized sub-networks. However, these systems inherently fall short of true sequential reasoning. Standard MoE networks typically rely on soft, continuous weighting across parallel paths, which mitigates the computational benefits of strict conditional execution \cite{fedus2022switch}. Further, MoE routing is typically a static, single-step, input-conditional operation\footnote{Once trained, the weights of a standard MoE router are frozen. It does not maintain a memory, it does not update an internal belief state, and it does not sequentially accumulate evidence across time or depth. It simply performs a flat, input-conditional matrix multiplication to partition a batch.}; it does not maintain a belief state, nor does it allow the network to sequentially accumulate evidence or re-evaluate its uncertainty.

To bridge this gap, we introduce \textbf{Neural Bayesian Sequential Routing (NBSR)}. This approach employs a \textit{Bayesian hierarchical decision graph} - a novel neural framework that models complex decision-making as an active, evidence-accumulating routing process. We formulate the decision structure as a \textit{Directed Acyclic Graph} (DAG) where each node contains \textit{a differentiable routing mechanism} and \textit{a neural evidence extractor}. Operating within a \textit{Dirichlet-Categorical conjugate} framework \cite{gelman2013BDA, bishop2006pattern}, the model maintains a persistent belief state over the final outcome space. At each routing step, local neural experts actively query the global data to extract strictly positive evidence vectors. These vectors act as Bayesian pseudo-counts, deterministically updating the Dirichlet concentration parameters and natively mirroring the human cognitive process of dynamic uncertainty reduction.

To enable end-to-end training of this discrete tree structure, we employ a \textit{Gumbel-Softmax relaxation} \cite{jang2017categorical, maddison2017concrete} coupled with the \textit{Straight-Through Estimator} (STE) \cite{bengio2013estimating}. This allows for hard, path-dependent routing during inference (which drastically reduces computational FLOPs) while maintaining smooth surrogate gradient flow to the routers during backpropagation. Furthermore, because NBSR natively tracks epistemic uncertainty, it naturally accommodates extensions into autonomous planning via recurrent memory (POMDP navigation) \cite{kaelbling1998planning} and resource-rational active learning (Bayesian Optimal Experimental Design) \cite{sebastiani2000BOED, foster2019BOED}.

We empirically validate the efficacy, interpretability, and computational efficiency of this NBSR framework across a highly diverse suite of five domains: (1) \textit{visual object categorization} (CIFAR-10); (2) \textit{structured medical diagnosis}, yielding personalized feature attribution; (3) \textit{language modeling} via interpretable syntactic-to-semantic token routing; (4) \textit{partially observable control} (POMDPs) utilizing a recurrent memory state; and (5) \textit{active clinical triage} modeled as a resource-rational BOED agent.
Our core contributions are:
\begin{enumerate}
    \item \textbf{A Bayesian Sequential Routing Framework:} We propose a novel formulation for conditional neural execution where discrete routing decisions sequentially update a Dirichlet belief state via exact conjugate addition, mathematically enforcing the progressive sharpening of decision boundaries.
    \item \textbf{End-to-End Hard Routing:} We adapt the Gumbel-Softmax estimator to train discrete hierarchical decision trees over deep feature extractors, naturally preventing representation collapse while enabling dynamic, entropy-based early exiting to minimize inference costs.
    \item \textbf{Theoretical Foundations:} We provide formal theoretical guarantees for the NBSR framework, proving the strict monotonicity of precision accumulation, bounding the variance reduction, and establishing the asymptotic consistency of the learning objective.
    \item \textbf{Interpretable and Safe Inference:} We demonstrate that our framework provides fully transparent, causal audit trails for every decision, whilst utilizing its native Dirichlet precision metric to successfully trigger out-of-distribution (OOD) safety abstentions \cite{hendrycks2017baseline, sensoy2018evidential} and prevent hallucinations.
    \item \textbf{Resource-Rational Extensions:} We successfully extend the core DAG topology to incorporate temporal memory buffers (NBSR-Mem) for autonomous control, and formulate it as an active \textit{Autoregressive State Machine} (AR-NBSR) capable of executing long-horizon, budget-constrained Bayesian Optimal Experimental Design.
\end{enumerate}

\section{Related Works} \label{sec:related_work}

\paragraph{Bayesian Belief Updating}
Standard deep learning architectures primarily output deterministic point estimates, which fail to capture the underlying uncertainty of the data distribution. To address this, Bayesian Neural Networks (BNNs) \cite{moehrke2024BNNs} introduce probability distributions over the network weights, which results in distributional predictions. The bottleneck lies in shaping the weights distributions based on data. Foundational approaches such as \textit{Bayes by Backprop} \cite{blundell2015weight} utilize variational inference to approximate intractable weight posteriors, while \textit{Monte Carlo Dropout} \cite{gal2016dropout} provides a mathematically grounded, computationally cheaper approximation by performing multiple stochastic forward passes. However, these methods primarily focus on capturing epistemic uncertainty within the model parameters rather than explicitly modeling the sequential accumulation of state-belief. 

Alternatively, classical Bayesian probability relies on exact posterior inference using \textit{conjugate priors}. When the posterior distribution belongs to the same probability family as the prior, the likelihood function updates the prior analytically via closed-form algebraic addition \cite{raiffa1961applied,gelman2013BDA,lambert2018students}. For example, the Beta distribution acts as the conjugate prior for the Binomial likelihood, and the Dirichlet distribution acts as the conjugate prior for the Categorical likelihood. While typically restricted to simple statistical models, our Neural Bayesian Sequential Routing (NBSR) framework leverages this exact conjugacy in a deep learning context. By treating continuous neural outputs as pseudo-observations, NBSR performs sequential belief updating without the prohibitive computational overhead of standard variational BNN ensembles.

\paragraph{Evidential Deep Learning}
Evidential Deep Learning (EDL), pioneered by Sensoy et al. \cite{sensoy2018evidential}, provides an elegant paradigm for uncertainty quantification without requiring the computational complexities of Bayesian sampling. Rooted in Subjective Logic and the Dempster-Shafer Theory of Evidence, EDL reformulates discrete classification tasks. Instead of forcing a neural network to distribute a total probability mass of 1 across all classes via a softmax layer (which may lead to overconfident predictions on unfamiliar data), EDL trains a deterministic neural network to output the concentration parameters of a Dirichlet distribution.

This approach parameterizes a "distribution over distributions" over the categorical simplex. When the network encounters familiar data, it generates high positive evidence for the target class, sharpening the Dirichlet distribution. Conversely, when it encounters out-of-distribution (OOD) data, the network generates zero evidence, gracefully reverting the Dirichlet state to a uniform, maximum-entropy prior (i.e. explicitly outputting an "I do not know" state). EDL has demonstrated state-of-the-art success in OOD detection and adversarial robustness. Our proposed framework fundamentally extends the EDL paradigm; whereas standard EDL relies on a static, monolithic forward pass, NBSR distributes the evidence extraction process across a hierarchical, dynamically routed DAG, transforming evidential inference into an active, resource-rational sequential process.

\paragraph{Hierarchical Mixtures of Experts and Decision Trees}
Classical divide-and-conquer algorithms, such as Classification and Regression Trees (CART) \cite{breiman1984cart}, explicitly partition the input space into nested regions using hard routing decisions. To provide a differentiable, probabilistic alternative, the Hierarchical Mixture of Experts (HME) \cite{jordan1993HME} was introduced. HME structures neural networks as a soft decision tree where internal nodes act exclusively as gating functions and terminal leaves act as local regression or classification experts. The final prediction is computed as a probability-weighted continuous average of the \textit{leaf} experts' outputs. Subsequent works extended HMEs with constructive growing and pruning algorithms to optimize the tree topology during training \cite{waterhouse1995HME}. While these models elegantly introduce conditional computation, they suffer from critical architectural limitations: they treat intermediate nodes purely as \textit{routers} (delaying all decisions to the leaves), they partition the raw input space directly, and their reliance on multiplicative, zero-sum probabilities prevents the native quantification of epistemic uncertainty. Our NBSR framework fundamentally departs from this paradigm. By treating every node (whether internal or terminal) as an active evidence extractor, enforcing hard discrete routing to prevent representation collapse, and sequentially querying a shared, dense \textit{Global Knowledge Oracle} to additively accumulate Dirichlet evidence, NBSR transforms the tree from a simple soft-routing mechanism into an active, uncertainty-aware Bayesian reasoning process\footnote{A comprehensive architectural comparison can be found in Appendix~\ref{app:nbsr_vs_classical_trees}.}.

\section{Preliminaries} \label{sec:preliminaries}

\paragraph{Dirichlet Distribution and Evidential Belief.}
Standard deterministic neural networks produce point-estimate probability vectors (e.g. via a softmax layer) that lack a native notion of uncertainty, frequently leading to overconfident predictions on unfamiliar data. In contrast, evidential deep learning \cite{sensoy2018evidential} formulates discrete predictive tasks, a $K$-class image classification for example, as a Bayesian process by predicting the \textit{concentration parameters} $\bm{\alpha} = [\alpha_1, \dots, \alpha_K]$ of a Dirichlet distribution\footnote{Extended mathematical details regarding the Dirichlet distribution are provided in Appendix~\ref{app:dirichlet_dist}.}. As the multivariate generalization of the Beta distribution, the Dirichlet distribution acts as a continuous probability density over the categorical probability simplex, essentially serving as a ``distribution over distributions''.

Drawing from \textit{Subjective Logic}, this formalizes the notion of belief assignments into a framework where predictions are treated as subjective opinions. The positive parameters $\alpha_k > 0$ represent the accumulated evidence for each respective class. Specifically, the parameters are linked to the learned neural evidence $e_k \ge 0$ via the relation $\alpha_k = e_k + 1$. When a network extracts no supporting evidence ($e_k = 0$ for all $k$), the parameters default to $\alpha_k = 1$, yielding a uniform prior distribution $\bm{\alpha} = [1, \dots, 1]$. This state explicitly represents total epistemic uncertainty, allowing the model to mathematically express ``I do not know'' when faced with ambiguous or out-of-distribution inputs.

Under this joint formulation, to focus on the probability $p_k$ of a specific class $k$, we integrate out the probabilities of the other $K-1$ classes. The resulting marginal distribution for class $k$ mathematically collapses into a Beta distribution:
$$ p_k \sim \text{Beta}(\alpha_k, \alpha_0 - \alpha_k) $$
where $\alpha_0 = \sum_{i=1}^K \alpha_i$ represents the total precision. Rooted in the framework of Subjective Logic \cite{josang2016subjective, sensoy2018evidential}, epistemic uncertainty $u$ is explicitly defined as inversely proportional to this total evidence:
\begin{equation} \label{eq:uncertainty_Dirichlet_cc} \tag{cc.Eq.\ref{eq:uncertainty_Dirichlet}}
    u = \frac{K}{\alpha_0}
\end{equation}
Consequently, $\alpha_0$ acts as a direct inverse measure of structural uncertainty. To obtain a concrete, scalar probability score for evaluating the network (such as for computing the Negative Log-Likelihood), we calculate the mathematical mean of this marginal distribution. The resulting Bayesian expected marginal is:
\begin{equation} \label{eq:expected_marginal_cc} \tag{cc.Eq.\ref{eq:expected_marginal}}
    \mathbb{E}[p_k] = \frac{\alpha_k}{\alpha_0}
\end{equation}
By utilizing these expected marginals rather than deterministic softmax outputs, our framework explicitly binds predictive confidence to total accumulated evidence.

\paragraph{Differential Entropy of the Dirichlet distribution.}
To quantify the total structural uncertainty of the belief state, we rely on the differential entropy of the Dirichlet distribution, denoted as $\mathcal{H}(\bm{\alpha})$. In our framework, this entropy serves as a mathematically rigorous metric for epistemic confidence and is computed analytically as:
\begin{equation} \label{differential_entropy_Dirichlet_cc} \tag{cc.Eq.\ref{eq:differential_entropy_Dirichlet}}
    \mathcal{H}(\bm{\alpha}) = \log \text{B}(\bm{\alpha}) + (\alpha_0 - K)\psi(\alpha_0) - \sum_{k=1}^K (\alpha_k - 1)\psi(\alpha_k)
\end{equation}
where $\text{B}(\bm{\alpha})$ is the multivariate Beta function and $\psi(\cdot)$ denotes the digamma function (the logarithmic derivative of the Gamma function). As evidence accumulates ($\alpha_0 \to \infty$), the variance of the distribution strictly shrinks, driving the differential entropy $\mathcal{H}(\bm{\alpha})$ downward towards $-\infty$ in the limit of absolute certainty\footnote{As formalized later in Theorem \ref{theorem:1}, while the total precision strictly increases and the variance strictly shrinks at every sequential step, the differential entropy $\mathcal{H}(\bm{\alpha})$ is not mathematically guaranteed to decrease strictly monotonically. Specifically, an asymmetrical evidence update that strongly contradicts the current prior can cause a momentary entropy spike as the distribution shifts its center of mass, though it ultimately continues its asymptotic collapse towards $-\infty$.}. We exploit this informational dynamic by using $\mathcal{H}(\bm{\alpha})$ as a temporal penalty during training, and use a threshold ($\eta$) for dynamic early exiting during inference.

\paragraph{Differentiable Discrete Routing.}
A core challenge in hierarchical neural networks is routing data through discrete, conditional paths. Let $\mathbf{z} = (z_1, \dots, z_K)$ denote the routing logits at a given decision node. To encourage structural exploration during training, the network must sample a routing path stochastically. However, standard categorical sampling and the discrete \texttt{argmax} operation are fundamentally non-differentiable, severing the backpropagation of error gradients. To resolve this, we leverage the \textit{Gumbel-Softmax continuous relaxation} \cite{jang2017categorical, maddison2017concrete}, which serves as a \textit{re-parameterization trick} for discrete distributions. By injecting independent Standard Gumbel noise $g_k \sim \text{Gumbel}(0,1)$ into the routing logits, we cleanly decouple the stochasticity from the deterministic network parameters. This enables us to compute a smooth, differentiable approximation vector $\bm{\pi}$ using a temperature-scaled softmax:
\begin{equation} \label{eq:Gumbel-Softmax_cc} \tag{cc.Eq.\ref{eq:Gumbel-Softmax}}
  \pi_k = \frac{\exp((z_k + g_k) / \tau)}{\sum_{i=1}^K \exp((z_i + g_i) / \tau)} 
\end{equation}
where the temperature hyperparameter $\tau > 0$ controls the smoothness of the distribution. As $\tau \to 0$, the continuous output $\bm{\pi}$ asymptotically approaches a discrete one-hot vector. 

To enforce strict conditional execution (saving computational FLOPs on unvisited branches) while preserving end-to-end trainability, we couple this relaxation with the \textit{Straight-Through Estimator} (STE) \cite{bengio2013estimating}. During the forward pass, we discretize $\bm{\pi}$ into a hard one-hot vector $\mathbf{a}_{\text{hard}}$ via the \texttt{argmax} operator. During the backward pass, the non-differentiable discretization is bypassed, and gradients flow directly through the continuous relaxation $\bm{\pi}$. Mathematically, this surrogate gradient operation is implemented natively in modern autograd engines as\footnote{The \texttt{.detach()} operator severs a tensor from the computational graph, preventing gradients from propagating through it. During the forward pass, the detached term functionally cancels out $\bm{\pi}$, simplifying the equation to exactly $\mathbf{a}_{\text{hard}}$. During the backward pass, the gradient of the detached term evaluates to zero, ensuring the learning signal bypasses the non-differentiable \texttt{argmax} step and flows exclusively through the continuous approximation $+ \bm{\pi}$.}:
\begin{equation} \label{eq:Gumbel-Softmax-detach_cc} \tag{cc.Eq.\ref{eq:Gumbel-Softmax-detach}}
    \mathbf{a}_{\text{out}} = (\mathbf{a}_{\text{hard}} - \bm{\pi})\text{.detach()} + \bm{\pi}
\end{equation}
This mechanism guarantees that the forward computation remains strictly discrete and path-dependent, while the backward pass provides the dense, continuous gradients necessary to update the upstream gating parameters (extended theoretical details see Appendix \ref{app:gumbel_softmax_trick}).

\section{Methodology} \label{sec:method}

Our framework decomposes the decision-making process into \textbf{three} sequential phases: \textit{Global Feature Extraction}, \textit{Graph-Based Routing}, and \textit{Bayesian Evidence Accumulation}. Below, we detail the exact network architecture and logic, using a $K$-dimensional outcome space as our primary illustrative model (e.g. $K=10$ for CIFAR-10 categorization).

\paragraph{Notations}
To ensure clarity throughout the methodological and theoretical formulations, we summarize the key mathematical notations used in this paper in Table \ref{tab:notations}.

\begin{table}[H]
\centering
\caption{Summary of Key Mathematical Notations}
\label{tab:notations}
\renewcommand{\arraystretch}{1.2}
\small
\begin{tabular}{cp{11.5cm}}
\hline
\textbf{Notation} & \textbf{Description} \\ \hline
\multicolumn{2}{l}{\textit{Graph Topology \& Traversal}} \\
$\mathcal{G} = (\mathcal{V}, \mathcal{E})$ & Directed Acyclic Graph with computational nodes $\mathcal{V}$ and routing edges $\mathcal{E}$. \\
$v_t$, $a_t$ & The node visited and the discrete routing action taken at step $t$. \\
$\mathcal{A}(v_t)$ & The set of valid discrete actions (outgoing edges) originating from node $v_t$. \\
$\mathcal{T}$ & A specific sampled trajectory (path) through the graph for a given input $x$. \\
$L$ & The maximal pre-specified depth of the routing tree super-graph. \\
$T$ & The terminal node (or terminal step) of a specific trajectory $\mathcal{T}$, where $T \le L$. \\
$\mathcal{S}_t$ & The joint decision state at step $t$, incorporating context, belief, and memory history. \\
\multicolumn{2}{l}{\textit{Data \& Representations}} \\
$x, y^*$ & The input data instance and its corresponding ground-truth target class. \\
$K$ & The dimensionality of the outcome space (e.g. total number of categories). \\
$\mathbf{h}_x$ & The Global Knowledge Oracle (dense feature vector in $\mathbb{R}^d$) generated by backbone $f_{\phi}$. \\
$\boldsymbol{\varepsilon}_t$ & Optional contextual side information available at step $t$. \\
$\mathbf{m}_t$ & The recurrent memory buffer (e.g. GRU hidden state) encoding past trajectory observations. \\
\multicolumn{2}{l}{\textit{Neural Components}} \\
$f_{\phi}$ & The shared global feature backbone, parameterized by $\phi$, mapping input $x$ to oracle state $\mathbf{h}_x$. \\
$\text{MLP}_{\theta_{v_t}}$ & The Router network at node $v_t$, parameterized by $\theta_{v_t}$, outputting routing logits $\mathbf{z}_t$. \\
$f_{v_t}$ & The Expert network at node $v_t$, parameterized by weights $\mathbf{W}_{v_t}$ and bias $\mathbf{b}_{v_t}$. \\
$\mathbf{e}_t$ & The strictly positive, continuous evidence vector ($\in \mathbb{R}^K_+$) extracted by the expert $f_{v_t}$. \\
$\tau$ & The Gumbel-Softmax temperature controlling the continuous relaxation of the router. \\
$\Theta$ & The complete set of all trainable parameters within the end-to-end framework. \\
\multicolumn{2}{l}{\textit{Bayesian State, Uncertainty \& Optimization}} \\
$\bm{\alpha}_t$ & The Dirichlet concentration parameter vector (belief state) at step $t$, where $\bm{\alpha}_t \in \mathbb{R}^K_+$. \\
$\alpha_0^{(t)}$ & The precision (total sum of concentration parameters) of the belief state: $\sum_k \alpha_{t,k}$. \\
$\mathbf{p}$ & The Categorical probability distribution over the $K$ classes. \\
$u$ & Epistemic uncertainty derived from Subjective Logic ($u = K/\alpha_0^{(t)}$), utilized for OOD abstention. \\
$\mathcal{H}(\bm{\alpha}_t)$ & The differential entropy of the Dirichlet belief state. \\
$D_{\text{KL}}$ & The Kullback-Leibler divergence representing the sequential belief shift ($D_{\text{KL}}(\bm{\alpha}_{t+1} \parallel \bm{\alpha}_t)$). \\
$\eta$ & The entropy confidence threshold utilized for dynamic early exiting. \\
$c(v_t), \gamma$ & The asymmetric measurement cost of expert $v_t$, and its associated budgetary penalty scalar. \\
$\lambda$ & The tunable hyperparameter scaling the temporal entropy penalty in the loss function. \\ \hline
\end{tabular}
\end{table}

\subsection{Pipeline Overview}

The neural Bayesian sequential routing pipeline transforms raw input data into a series of \textit{discrete, uncertainty-aware decisions} that progressively refine a \textit{global belief state}. This process mirrors human cognitive strategies, where an initial broad hypothesis is ``sharpened'' through the acquisition of specific evidence sequentially.

The decision-making architecture consists of a \textit{Directed Acyclic Graph (DAG)} $\mathcal{G} = (\mathcal{V}, \mathcal{E})$, where nodes $\mathcal{V}$ represent specific decision \textit{states} and edges $\mathcal{E}$ represent the possible \textit{actions} or \textit{paths}. Within this graph, we utilize two distinct neural components: a \textit{Router Network} to navigate the graph and an \textit{Expert Network} to extract information.

The data flow follows \textit{three} logical stages:

\begin{enumerate}
    \item \textbf{Global Encoding} (\textit{The Global Knowledge Oracle}): 
    the input $x$ (e.g. an image or patient profile) is first processed by a shared neural backbone to produce a \textit{static} feature representation $\mathbf{h}_x \in \mathbb{R}^d$. This vector acts as a \textit{Persistent Knowledge Oracle} - a constant reference point for the entire system. Unlike standard feed-forward networks where data is transformed and ``forgotten'' layer-by-layer, $\mathbf{h}_x$ is broadcast to all nodes in the DAG. This ensures that every subsequent decision has access to the full contextual state of the original input\footnote{This in some sense serves as a 'residual connection'.}.

    \item \textbf{Sequential Traversal} (\textit{The Routing Mechanism}): 
    starting at the root node $v_0$, the model traverses a path $\mathcal{T}$ through the graph. At each node $v_t$, a \textit{Router Network} $\text{MLP}_{\theta_{v_t}}$ is activated. The router's task is to decide which path to take next. It takes as input the concatenation of the global features $\mathbf{h}_x$ and the current \textit{Bayesian Belief State} $\bm{\alpha}_t$, as well as any optional side information\footnote{This side info can be e.g. rules learned on a high level of abstraction of the dataset (e.g. physical constraints), or with different levels of fidelity. For simplicity, we ignore this extra info in later work.} $\boldsymbol{\varepsilon}_t$. 
    
    To make a selection, the router outputs a discrete action $a_t$ using a \textit{Gumbel-Softmax relaxation}. This specific mathematical technique allows the model to choose a single, ``hard'' path during inference (improving efficiency) while still allowing gradients to flow during training so the model can learn from its mistakes. This ensures the path taken is conditioned not just on the raw data, but on what the model has already ``learned'' or concluded in previous steps.

    \item \textbf{Evidence Accumulation} (\textit{The Bayesian Update}): 
    upon reaching the selected node $v_{t+1}$, a local \textit{Expert Network} $f_{v_{t+1}}$ (parameterized by weights $\mathbf{W}$ and bias $\mathbf{b}$) interrogates (queries) the Oracle $\mathbf{h}_x$. The expert extracts a strictly positive \textit{Evidence Vector} $\mathbf{e}_t \in \mathbb{R}^K_+$.
    
    This evidence is added to the current belief state via \textit{Conjugate Addition}: $\bm{\alpha}_{t+1} = \bm{\alpha}_t + \mathbf{e}_t$. This update rule is rooted in Bayesian statistics, specifically using the \textit{Dirichlet Distribution}. Every addition of evidence ``sharpens'' the model's confidence, narrowing the probability distribution over the $K$ possible outcomes. This process repeats, adding more evidence and reducing entropy, until a \textit{Confidence Threshold} $\eta$ is met or a leaf node is reached, at which point a final prediction is rendered.
\end{enumerate}

A depth-2 computational graph illustrating the NBSR procedure (both the discrete forward execution path and the continuous backward gradient flow) is shown in Fig.~\ref{fig:NBSR_pipeline}.

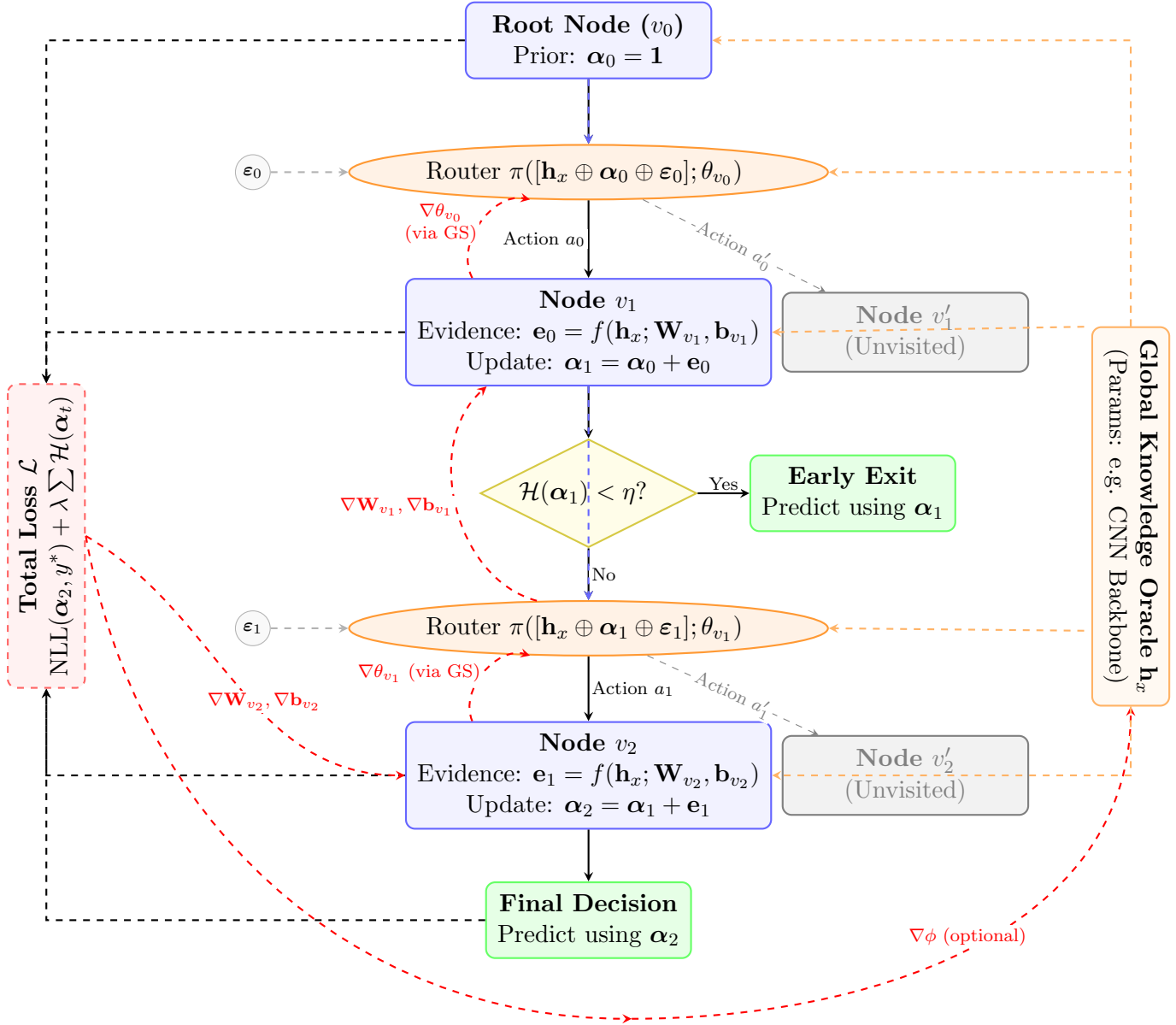
\begin{figure}[htbp]
    \centering
    \begin{tikzpicture}[
        >=stealth,
        node distance=1.5cm and 2cm,
        state/.style={rectangle, draw=blue!60, rounded corners, fill=blue!5, align=center, inner sep=5pt, minimum width=3.8cm, thick},
        router/.style={ellipse, draw=orange!80, fill=orange!10, align=center, inner sep=3pt, thick},
        test/.style={diamond, draw=yellow!80!black, fill=yellow!10, align=center, inner sep=2pt, aspect=2, thick},
        leaf/.style={rectangle, draw=green!60, rounded corners, fill=green!10, align=center, inner sep=5pt, minimum width=2.5cm, thick},
        loss/.style={rectangle, draw=red!60, rounded corners, fill=red!5, align=center, inner sep=5pt, minimum width=3.2cm, thick, dashed},
        edge label/.style={fill=white, inner sep=1pt, font=\scriptsize},
        grad label/.style={fill=white, inner sep=1pt, font=\scriptsize\color{red}},
        data/.style={circle, draw=gray!60, fill=gray!5, inner sep=2pt, font=\scriptsize},
        oracle_style/.style={rectangle, draw=orange!60, rounded corners, fill=orange!5, align=center, inner sep=5pt, thick, rotate=90}
    ]

    \node[state, opacity=0, text opacity=0] (v0_dummy) at (0, 0) { \textbf{Root Node ($v_0$)} \\ Prior: $\bm{\alpha}_0 = \mathbf{1}$ };
    \node[router, opacity=0, text opacity=0, below=1.0cm of v0_dummy] (r0_dummy) { Router $\pi([\mathbf{h}_x \oplus \bm{\alpha}_0 \oplus \boldsymbol{\varepsilon}_0]; \theta_{v_0})$ };

    \node[state, below left=1.2cm and 0.2cm of r0_dummy] (v1) {
        \textbf{Node $v_1$} \\
        Evidence: $\mathbf{e}_0 = f(\mathbf{h}_x; \mathbf{W}_{v_1}, \mathbf{b}_{v_1})$ \\
        Update: $\bm{\alpha}_1 = \bm{\alpha}_0 + \mathbf{e}_0$
    };

    \node[state, right=0.15cm of v1, fill=gray!10, draw=gray, text=gray] (v1_alt) {
        \textbf{Node $v'_1$} \\
        (Unvisited)
    };

    \node[router, above=1.2cm of v1] (r0) { Router $\pi([\mathbf{h}_x \oplus \bm{\alpha}_0 \oplus \boldsymbol{\varepsilon}_0]; \theta_{v_0})$ };
    \node[state, above=1.0cm of r0] (v0) { \textbf{Root Node ($v_0$)} \\ Prior: $\bm{\alpha}_0 = \mathbf{1}$ };
    \node[data, left=1.2cm of r0] (side_info0) {$\boldsymbol{\varepsilon}_0$};

    \node[test, below=0.8cm of v1] (check1) { $\mathcal{H}(\bm{\alpha}_1) < \eta$? };

    \node[leaf, right=0.8cm of check1] (exit1) {
        \textbf{Early Exit} \\
        Predict using $\bm{\alpha}_1$
    };

    \node[router, below=0.8cm of check1] (r1) { Router $\pi([\mathbf{h}_x \oplus \bm{\alpha}_1 \oplus \boldsymbol{\varepsilon}_1]; \theta_{v_1})$ };
    \node[data, left=1.2cm of r1] (side_info1) {$\boldsymbol{\varepsilon}_1$};

    \node[state, below=1cm of r1] (v2) {
        \textbf{Node $v_2$} \\
        Evidence: $\mathbf{e}_1 = f(\mathbf{h}_x; \mathbf{W}_{v_2}, \mathbf{b}_{v_2})$ \\
        Update: $\bm{\alpha}_2 = \bm{\alpha}_1 + \mathbf{e}_1$
    };

    \node[state, right=0.15cm of v2, fill=gray!10, draw=gray, text=gray] (v2_alt) {
        \textbf{Node $v'_2$} \\
        (Unvisited)
    };

    \node[leaf, below=0.8cm of v2] (exit2) {
        \textbf{Final Decision} \\
        Predict using $\bm{\alpha}_2$
    };

    \path (v0.north) -- (exit2.south) coordinate[midway] (graph_center);

    \node[loss, rotate=90, below left=-1.5cm and 9cm of graph_center] (loss) {
        \textbf{Total Loss $\mathcal{L}$} \\
        $\text{NLL}(\bm{\alpha}_2, y^*) + \lambda\sum\mathcal{H}(\bm{\alpha}_t)$
    };

    \node[oracle_style, rotate=180, below left=3.5cm and -9cm of graph_center] (oracle) {
        \textbf{Global Knowledge Oracle} $\mathbf{h}_x$ \\
        (Params: e.g. CNN Backbone)
    };

    \draw[->, thick] (v0) -- (r0);
    \draw[->, thick] (r0) -- node[edge label, left] {Action $a_0$} (v1);
    \draw[->, dashed, draw=gray, text=gray] (r0) -- node[edge label, sloped] {Action $a'_0$} (v1_alt);

    \draw[->, thick] (v1) -- (check1);
    \draw[->, thick] (check1) -- node[edge label, above] {Yes} (exit1);
    \draw[->, thick] (check1) -- node[edge label, right] {No} (r1);

    \draw[->, thick] (r1) -- node[edge label, right] {Action $a_1$} (v2);
    \draw[->, dashed, draw=gray, text=gray] (r1) -- node[edge label, sloped] {Action $a'_1$} (v2_alt);

    \draw[->, thick] (v2) -- (exit2);

    \draw[->, thick, dashed] (v0.west) -| (loss.east);
    \draw[->, thick, dashed] (v1.west) -| (loss.east);
    \draw[->, thick, dashed] (v2.west) -| (loss.west);
    \draw[->, thick, dashed] (exit2.west) -| (loss.west);

    \draw[->, thick, draw=gray!60, dashed] (side_info0) -- (r0);
    \draw[->, thick, draw=gray!60, dashed] (side_info1) -- (r1);

    \draw[->, thick, draw=blue!60, dashed] (v0) -- (r0);
    \draw[->, thick, draw=blue!60, dashed] (v1) -- (r1);

    \draw[->, thick, dashed, red] (loss.south) to[out=-30, in=180] node[grad label, above, pos=0.6] {$\nabla \mathbf{W}_{v_2}, \nabla \mathbf{b}_{v_2}$} (v2.west);

    \draw[->, thick, dashed, red] (v2) to[bend left=65] node[grad label, left] {$\nabla \theta_{v_1}$ (via GS)} (r1);
    \draw[->, thick, dashed, red] (r1) to[bend left=62] node[grad label, left] {$\nabla \mathbf{W}_{v_1}, \nabla \mathbf{b}_{v_1}$} (v1);
    \draw[->, thick, dashed, red] (v1) to[bend left=65] node[grad label, left, align=center] {$\nabla \theta_{v_0}$ \\ (via GS)} (r0);

    \draw[->, thick, dashed, red] (loss.south) to[out=-80, in=180] (-5,-15);
    \draw[->, thick, dashed, red] (-5,-15) to[out=0, in=-90] node[grad label, above] {$\nabla \phi$ (optional)}(oracle.east);


    \draw[->, thick, dashed, draw=orange!60] (oracle.west) |- (v0.east);
    \draw[->, thick, dashed, draw=orange!60] (oracle.west) |- (r0.east);
    \draw[->, thick, dashed, draw=orange!60] (2.0,-4.3) -- (v1.east);
    \draw[->, thick, dashed, draw=orange!60] (2.0,-9.0) -- (r1.east);
    \draw[->, thick, dashed, draw=orange!60] (oracle) |- (v2.east);

    \end{tikzpicture}
    \caption{Computational graph of the NBSR process. The total number of sequential layers (maximum tree depth) and the number of expert nodes at each layer (branching width) are predefined hyperparameters that establish the maximal capacity of the super-graph. The \textbf{Global Knowledge Oracle} is positioned in the outer loop, while the \textbf{Total Loss} resides in the inner loop. The oracle provides a persistent feature set $\mathbf{h}_x$ that is queried by both routers and experts. \textbf{Forward Pass:} Path selection is conditioned on current belief and oracle data. \textbf{Backward Pass:} Gumbel-Softmax (GS) allows end-to-end gradient updates.}
    \label{fig:NBSR_pipeline}
\end{figure}

\subsection{The Decision Graph}
We formulate the sequential decision-making process as a traversal over a Directed Acyclic Graph\footnote{While the primary formulation describes a strict feed-forward DAG, this topology can be temporally flattened into an \textit{Autoregressive State Machine} by sharing a single router and expert pool across time. As we will demonstrate in Section~\ref{subsec:nbsr_boed}, this cyclical extension prevents combinatorial explosion in long-horizon sequential planning tasks.} (DAG), denoted as $\mathcal{G} = (\mathcal{V}, \mathcal{E})$. Here, $\mathcal{V}$ represents the set of computational \textit{nodes} (\textit{vertices}, acting as local experts and routing gates), and $\mathcal{E}$ represents the directed \textit{edges} corresponding to routing actions. The process originates at a designated root node $v_0$, which acts purely as an initial routing gate without an evidence-extracting expert. From there, it sequentially traverses the graph. At each node $v_t \in \mathcal{V}$, a \textit{conditional routing policy} determines the discrete outgoing edge $a_t \in \mathcal{A}(v_t)$, where $\mathcal{A}(v_t)$ is the action set at node $v_t$, to traverse, leading to the subsequent node $v_{t+1}$. This traversal governs the flow of information and concludes when a final prediction is rendered over the $K$-dimensional outcome space.

While the maximal topology of the graph $\mathcal{G}$ (i.e. the total number of available layers and expert nodes) is pre-specified as a fixed super-graph, the effective architecture is highly dynamic. The network autonomously performs implicit pruning during training; if an expert node fails to provide discriminative evidence, the upstream routers learn to assign near-zero probability to its incoming edges, effectively removing it from the computational flow. Further, the inference depth is dynamically truncated on a per-sample basis via the entropy-based early exiting mechanism.

By initializing a sufficiently large \textit{super-graph}, we establish a high-capacity, low-bias foundation. Importantly, the aforementioned dynamic mechanisms act as adaptive regularizers: implicit routing pruning organically restricts the \textit{effective width}, while entropy-based early exiting curtails the \textit{effective depth}. Together, they dynamically rein in the variance, allowing the network to strike an optimal, sample-specific balance between representational capacity and generalization. Precise discussions regarding the bias-variance trade-off, induced by the graph topology, are presented in a later theoretical Section \ref{subsec:bias_variance_trade_off}.

\subsection{Global Feature Representation: The Knowledge Oracle} \label{subsec:the_global_knowledge_oracle}
To provide a stable informational foundation for the tree, the input $x$ (e.g. visual or tabular data) is embedded \textit{once} into a dense feature space to create a \textit{Global Knowledge Oracle}, denoted as $\mathbf{h}_x$. Formally, given a shared neural backbone $f_{\phi}$ parameterized by $\phi$, the oracle is computed as:
\begin{equation}
    \mathbf{h}_x = f_{\phi}(x)
\end{equation}
This shared representation acts as a persistent knowledge base that encapsulates the totality of the external data available for a given instance. For example, for visual tasks, we employ a truncated ResNet-18 architecture as $f_{\phi}$, with the final fully connected classification layer removed, to extract a dense, continuous feature representation $\mathbf{h}_x \in \mathbb{R}^d$, where e.g. $d = 256$. Rather than transforming the data sequentially layer-by-layer, this global oracle $\mathbf{h}_x$ is broadcast to the entire graph, allowing all active nodes in the routing tree to query it directly to forge node-specific, local knowledge. Note that, unlike pixel-level models such as random forest or CART decision trees\footnote{A comparison between NBSR and decision trees is made in Appendix.\ref{app:nbsr_vs_classical_trees}.} \cite{breiman1984cart}, which look at one raw feature at each hierarchical splitting node, the experts in NBSR partition the semantic space (e.g. "Setosa vs. Non-Setosa" in Iris classification), not the raw input space.

\subsection{The Bayesian State and Initialization}
The ultimate objective is to map the input to one of $K$ final outcomes (e.g. object categories, clinical endpoints). We model the uncertainty over these outcomes using a \textit{Dirichlet distribution}\footnote{The Dirichlet distribution is a multivariate continuous probability distribution parameterized by a positive concentration vector $\bm{\alpha}$. As the conjugate prior to the Categorical distribution, it essentially represents a ``distribution over distributions,'' making it mathematically ideal for modeling structural uncertainty over a $K$-dimensional probability simplex. Details about Dirichlet distribution can be found in Appendix.~\ref{app:dirichlet_dist}.}, parameterized by a concentration vector $\bm{\alpha}_t \in \mathbb{R}^K_+$; this Dirichlet distribution is sequentially refined along the selected tree trajectory. If we index the node depth by $t$, at the root node ($t=0$), the belief state is initialized to a uniform, maximum-entropy prior:
\begin{equation}
    \bm{\alpha}_0 = \mathbf{1} \in \mathbb{R}^K
\end{equation}
which implies no preference over the $K$ decisions at the beginning.

\subsection{The Differentiable Router Network}
At step $t$, located at node $v_t$, the model must select a single discrete branch $a_t \in \mathcal{A}(v_t)$ to traverse deeper into the tree. The routing decision is conditioned on both the input features and the current shape of the model's uncertainty. 

The router is instantiated as a \textit{Multi-Layer Perceptron}\footnote{The router architecture can be chosen as any reasonable network; here we choose MLP as an example.} (MLP). It accepts as input a concatenation of the global feature vector $\mathbf{h}_x$ with the current belief state $\bm{\alpha}_t$, alongside any optional side information $\boldsymbol{\varepsilon}_t$:
\begin{equation} \label{eq:router}
    \mathbf{z}_t = \text{MLP}_{\theta_{v_t}}([\mathbf{h}_x \oplus \bm{\alpha}_t \oplus \boldsymbol{\varepsilon}_t])
\end{equation}

From a practical implementation standpoint, the router network at any given node $v_t$ must have a statically defined output dimension. In standard deep learning frameworks such as PyTorch \cite{pytorch2019} or TensorFlow \cite{tensorflow2015}, the final linear layer of this MLP is hard-coded at initialization to output exactly $|\mathcal{A}(v_t)|$ routing logits $\mathbf{z}_t$, corresponding directly to the total number of possible discrete actions (child nodes) branching from $v_t$.

\begin{example}[MLP Router Architecture for CIFAR-10 Classification]
For a typical visual categorization task with a $K$-dimensional outcome space and a convolutional backbone extracting $d$-dimensional features, the router is parameterized as follows:
\begin{itemize}
    \item \textit{Input Dimension:} $d + K + |\boldsymbol{\varepsilon}_t|$ (e.g. $256 + 10 = 266$ for standard CIFAR-10 without side information).
    \item \textit{Hidden Layer:} linear projection to 128 dimensions, followed by a ReLU activation.
    \item \textit{Output Layer:} linear projection to $|\mathcal{A}(v_t)|$ dimensions (the number of outgoing branches from node $v_t$).
\end{itemize}
\end{example}

To achieve hard routing while preserving end-to-end differentiability, we apply the \textit{Gumbel-Softmax trick} \cite{jang2017categorical, maddison2017concrete} to the raw routing logits $\mathbf{z}_t$. Specifically, we first draw i.i.d. noise samples $g_k$ from a \textit{standard Gumbel distribution} to perturb the logits, which enables \textit{stochastic exploration}. We then compute a continuous routing probability vector $\bm{\pi}$, where the probability of selecting branch $k$ is given by (\ref{eq:Gumbel-Softmax_cc}):
\begin{equation} \label{eq:gumbel_softmax}
    \pi_k = \frac{\exp((z_{t,k} + g_k) / \tau)}{\sum_{i=1}^{|\mathcal{A}(v_t)|} \exp((z_{t,i} + g_i) / \tau)}
\end{equation}
Here, $\tau > 0$ is a temperature parameter that controls the smoothness of the distribution; as $\tau \to 0$, the continuous softmax output asymptotically approaches a discrete one-hot vector. 

To enforce strict conditional execution, ensuring the model only pays the computational cost for a single path, we utilize the \textit{Straight-Through Estimator}\footnote{The Straight-Through Estimator (STE) is a technique used to train neural networks with discrete, non-differentiable operations. It works by performing the discrete operation (such as \texttt{argmax} or thresholding) during the forward pass, but treating it as a differentiable identity function, or routing gradients through its smooth continuous proxy, during the backward pass. This prevents the gradients from zeroing out, enabling standard backpropagation through discrete bottlenecks.} (STE). During the forward training pass, the continuous vector $\bm{\pi}$ is discretized using a hard \texttt{argmax}, effectively activating only a \textit{single} subsequent node $v_{t+1}$. During the backward pass, the non-differentiable \texttt{argmax} operation is bypassed, and gradients flow smoothly through the continuous approximation $\bm{\pi}$ directly into the router's parameters (further mathematical details are provided in Appendix \ref{app:gumbel_softmax_trick}).

\subsection{Evidence Extraction and Belief Update}
Once a target node $v_{t+1}$ is activated, its resident ``expert'' network extracts local evidence from the input to update the belief state. Conceptually, this process treats the global feature representation as a \textit{persistent knowledge base}, driven by three interacting components:

\begin{itemize}[label=-]
    \item \textit{The Global (Knowledge) Oracle ($\mathbf{h}_x$):} serves as the static, comprehensive reference of the input instance, as defined previously in Section.\ref{subsec:the_global_knowledge_oracle}.
    
    \item \textit{The Context-Aware Query ($\mathbf{z}_t$):} at step $t$, the router asks: ``based on what I already know (i.e. the current state $\bm{\alpha}_t$) and what the oracle holds ($\mathbf{h}_x$), which expert should I consult next?''
    
    \item \textit{The Information Extractor ($f_{v_{t+1}}$):} the chosen expert network $f_{v_{t+1}}$ acts as a specific, selective querying function\footnote{\textbf{Static vs. Dynamic Experts:} our framework currently employs \textit{Static Experts} with implicit context. Because the router already used the belief state $\bm{\alpha}_t$ to select node $v_{t+1}$, the simple act of arriving at this node is implicitly conditioned on $\bm{\alpha}_t$. The matrix $\mathbf{W}_{v_{t+1}}$ represents a fixed, specialized filter; it does not need explicit knowledge of $\bm{\alpha}_t$ to extract its specific feature. This static design ensures faster inference, easier training, and conceptual clarity. An alternative is a \textit{Dynamic Expert} framework formulated as $f(\mathbf{h}_x, \bm{\alpha}_t; \mathbf{W}, \mathbf{b})$. In this paradigm, $\bm{\alpha}_t$ acts as an explicit query in an \textit{attention} mechanism, while $\mathbf{h}_x$ acts as the key/value, allowing the expert to dynamically shift its weights based on the model's exact current uncertainty. We leave the exploration of dynamic attention-based experts to future work.} $\mathbf{e}_t = f(\mathbf{h}_x; \mathbf{W}_{v_{t+1}}, \mathbf{b}_{v_{t+1}})$. It interrogates the oracle $\mathbf{h}_x$ from a highly specialized perspective (e.g. ``look specifically for bird-like features'') and extracts only the relevant localized evidence $\mathbf{e}_t$.
\end{itemize}

Mathematically, the evidence extractor typically consists of a \textit{linear} mapping from the shared feature vector $\mathbf{h}_x$ to the $K$-dimensional outcome space. Here, the trainable parameters $\mathbf{W}_{v_{t+1}}$ and $\mathbf{b}_{v_{t+1}}$ dictate precisely what information to extract from the global oracle given the specific context of the current node. To satisfy the strict positivity requirement of the Dirichlet concentration parameters, the output must be passed through a bounding activation function:

\begin{equation} \label{eq:evidence_extraction_activation}
    \mathbf{e}_t = \text{Activation}(\mathbf{W}_{v_{t+1}} \mathbf{h}_x + \mathbf{b}_{v_{t+1}})
\end{equation}

The choice of this activation function dictates the \textit{epistemic capacity} of the expert. For terminal leaf experts, this is typically the unbounded \textit{Softplus} function\footnote{The Softplus function is a smooth, differentiable activation function defined as $\text{Softplus}(x) = \ln(1 + e^x)$, acting as a continuous approximation of the ReLU activation function. It produces positive outputs ranging from 0 to $\infty$, commonly used to constrain output values to be positive in neural networks, such as in regression tasks, and has a derivative equal to the sigmoid function.} ($\ln(1 + e^x)$), which allows the network to inject massive evidence to crystallize a final, highly confident decision. However, to mathematically enforce a hierarchical taxonomy, intermediate routing nodes can utilize a scaled \textit{Sigmoid} activation. This places a strict upper bound on the evidence an intermediate node can extract, acting as an \textit{epistemic confidence budget} that ensures broad, early-stage hypotheses remain appropriately hesitant (a dynamic visually demonstrated in our toy experiment in Section~\ref{subsec:toy_iris}).

The resulted $\mathbf{e}_t$ thus serves as an \textit{information gain} to the current nodal decision-maker. The belief state is then deterministically updated via conjugate addition\footnote{In Bayesian statistics, the Dirichlet distribution is the conjugate prior for the Categorical distribution. Traditionally, observing discrete categorical counts $\mathbf{c}$ updates a prior $\text{Dir}(\bm{\alpha})$ to a posterior $\text{Dir}(\bm{\alpha} + \mathbf{c})$. Our framework generalizes this property by treating the neural network's continuous, positive evidence vectors $\mathbf{e}_t$ as pseudo-counts. According to Bayes' theorem, the posterior distribution over the class probabilities $\mathbf{p}$ is given by $P(\mathbf{p} \mid \mathbf{e}_t) \propto P(\mathbf{e}_t \mid \mathbf{p}) P(\mathbf{p} \mid \bm{\alpha}_t) \propto \prod_{k=1}^K p_k^{(\alpha_{t,k} + e_{t,k}) - 1}$. This guarantees that the posterior remains a Dirichlet distribution parameterized by $\bm{\alpha}_t + \mathbf{e}_t$, effectively performing a differentiable, continuous analog of exact Bayesian updating. See Appendix \ref{app:dirichlet_dist} for extended mathematical details.}, accumulating the evidence gathered by traversing the chosen path:
\begin{equation}
    \bm{\alpha}_{t+1} = \bm{\alpha}_t + \mathbf{e}_t
\end{equation}

\begin{example}[CIFAR-10 Evidence Extractor]
Following the previous visual categorization setup in Example 1, the specific architecture for the local evidence extractor at each terminal node $v_{t+1}$ is configured as follows:
\begin{itemize}
    \item \textit{Input layer:} receives the $d$-dimensional global feature vector $\mathbf{h}_x$ (e.g. $d = 256$ from the truncated ResNet-18 backbone).
    \item \textit{Transformation:} a single linear fully-connected layer, parameterized by a weight matrix $\mathbf{W}_{v_{t+1}} \in \mathbb{R}^{K \times d}$ and bias $\mathbf{b}_{v_{t+1}} \in \mathbb{R}^K$.
    \item \textit{Activation:} an unbounded Softplus activation function is applied to the raw logits to guarantee mapping into the strictly positive domain ($\mathbb{R}_+^K$).
    \item \textit{Output dimension:} a $K$-dimensional evidence vector $\mathbf{e}_t$ (e.g. $K = 10$ representing the isolated evidence for each categorical outcome).
\end{itemize}
\end{example}

\subsection{Training Objective and Dynamics}
The routing tree is trained end-to-end to \textit{maximize the expected likelihood of the correct class}. For a single data instance $x$ with ground truth outcome $y^*$, we optimize the \textit{Negative Log-Likelihood} (NLL) of the final Dirichlet distribution $\bm{\alpha}_T$ \textit{at the terminal leaf node}. This is augmented with an intermediate entropy penalty $\mathcal{H}(\cdot)$ to explicitly enforce the progressive sharpening of the decision boundary along the sampled trajectory $\mathcal{T}$:

\begin{equation} \label{eq:training_obj}
    \mathcal{L}(x, y^*) = - \log \left( \frac{\alpha_{T, y^*}}{\sum_{k=1}^K \alpha_{T, k}} \right) + \lambda \sum_{v_t \in \mathcal{T}} \mathcal{H}(\bm{\alpha}_t)
\end{equation}
where $\lambda$ is a tunable hyperparameter controlling the routing efficiency and the aggressiveness of the uncertainty reduction. This loss design forces the model to maximize the expected probability of the ground-truth class, while sharpening our belief along the traversed tree path.

During end-to-end training over a dataset $\mathcal{D}$, the model minimizes \textit{the overall empirical risk} $\mathcal{R}$, defined as the average loss across all instances. The optimal parameter set $\Theta^*$ is obtained by solving:
\begin{equation} \label{eq:empirical_risk}
    \Theta^* = \arg\min_{\Theta} \mathcal{R}(\Theta) = \arg\min_{\Theta} \frac{1}{|\mathcal{D}|} \sum_{(x_i, y^*_i) \in \mathcal{D}} \mathcal{L}(x_i, y^*_i; \Theta)
\end{equation}
where $\Theta$ encompasses all trainable parameters optimized by this objective within the framework, specifically: the shared global feature backbone\footnote{As discussed later in Section.\ref{subsec:algorithm_NBSR}, training the backbone is optional; it can be either pre-trained (frozen) or trainable.} ($\phi$), which generates the oracle state $\mathbf{h}_x$; the router parameters ($\theta_{v_t}$), which dictate the sampled path $\mathcal{T}$; and the local evidence extractor weights and biases ($\mathbf{W}_{v_t}, \mathbf{b}_{v_t}$), which compute the evidence vectors $\mathbf{e}_t$ that sequentially construct the belief states $\bm{\alpha}_t$.

Training the tree structure admits a highly efficient, sample-specific gradient flow. Because of the discrete routing decisions made during the forward pass, gradients only propagate through the active trajectory $\mathcal{T}$. Consequently, \textit{unvisited experts} receive zero gradients and remain unmodified by that specific instance. Conversely, \textit{active experts} along the traversed path are directly optimized to extract more discriminative evidence from the global oracle. Simultaneously, the \textit{active routers} learn to adjust their conditional routing probabilities via the continuous backward pass of the Gumbel-Softmax estimator; if a chosen path yields a high loss, the router penalizes the corresponding action logit, which inherently boosts the probability of selecting an alternative, unvisited branch in the future. This dynamic interplay of selective evidence extraction and adaptive routing is precisely what allows the tree to autonomously specialize, naturally evolving the computational graph into a highly structured, conditional \textit{Mixture-of-Experts} (MoE).

\subsection{Inference, Early Exiting, and Epistemic Abstention}
During inference, the stochasticity introduced by the Gumbel noise during training is removed. At each node $v_t$, the router selects the deterministic path by taking the strict \texttt{argmax} of the routing logits $\mathbf{z}_t$. 

A core advantage of our Bayesian formulation is the natural affordance for \textbf{dynamic early exiting}. Because the concentration vector $\bm{\alpha}_t$ explicitly quantifies the model's uncertainty, we can halt the routing process dynamically. If, at step $t$, the differential entropy of the belief state $\mathcal{H}(\bm{\alpha}_t)$ falls below a predefined confidence threshold $\eta$, the traversal can terminate immediately (like \textit{early stopping}). The expected prediction is then derived from the current Dirichlet distribution: $\mathbb{E}[p_k] = \frac{\alpha_{t,k}}{\sum_i \alpha_{t,i}}$. This mechanism ensures that unambiguous inputs require shallower graph traversals, reducing inference FLOPs while preserving performance.

Further, this formulation intrinsically safeguards against hallucination. By explicitly tracking the total precision $\alpha_0$, the framework provides a native \textbf{Out-Of-Distribution (OOD) abstention} mechanism. If the epistemic uncertainty $u = \frac{K}{\alpha_0}$ remains above a critical safety threshold even after full graph traversal, the network recognizes its own lack of extracted evidence and gracefully abstains from prediction, ensuring robust safety in high-stakes environments.

\subsection{Training the NBSR Tree} \label{subsec:algorithm_NBSR}

We present the end-to-end training process in Algo.\ref{alg:nbsr_training}, which explicitly details both the forward inference trajectory and the backward gradient flow, demonstrating how conditional execution dynamically restricts the computational cost. Importantly, the backward pass highlights that the tree organically prunes its gradient updates: only the routers and experts along the actively sampled trajectory $\mathcal{T}$ receive learning signals for any given instance.

\begin{algorithm}[H]
\caption{End-to-End Training of Neural Bayesian Sequential Routing Tree}
\label{alg:nbsr_training}
\begin{algorithmic}[1]
\REQUIRE Training dataset $\mathcal{D}$, initialized parameters $\Theta = \{\phi, \theta, \mathbf{W}, \mathbf{b}\}$, learning rate $\eta$
\REQUIRE Max depth $L$, branching factor $\bar{W}$ (action space), confidence threshold $\eta$, entropy penalty $\lambda$
\WHILE{not converged}
    \STATE Sample a minibatch $\mathcal{B} \subset \mathcal{D}$
    \STATE Initialize batch loss $\mathcal{R} \leftarrow 0$
    
    \FOR{each instance $(x, y^*) \in \mathcal{B}$}
        \STATE \textit{// 1. Forward Pass \& Trajectory Sampling}
        \STATE $t \leftarrow 0$, $v_t \leftarrow v_0$ (Root), $\bm{\alpha}_0 \leftarrow \mathbf{1} \in \mathbb{R}^K$
        \STATE $\mathcal{T}^{(x)} \leftarrow \{v_0\}$
        \STATE $\mathbf{h}_x \leftarrow f_{\phi}(x)$ \COMMENT{Global Knowledge Oracle}
        
        \WHILE{$t < L$}
            \IF{$\mathcal{H}(\bm{\alpha}_t) < \eta$}
                \STATE \textbf{break} \COMMENT{Uncertainty sufficiently resolved; exit early}
            \ENDIF
            
            \STATE $\mathbf{z}_t \leftarrow \text{MLP}_{\theta_{v_t}}([\mathbf{h}_x \oplus \bm{\alpha}_t \oplus \boldsymbol{\varepsilon}_t])$
            \STATE $a_t \sim \text{Gumbel-Softmax}(\mathbf{z}_t)$ \COMMENT{Sample discrete path via Straight-Through Estimator}
            
            \STATE $v_{t+1} \leftarrow \text{TargetNode}(v_t, a_t)$
            \STATE $\mathcal{T}^{(x)} \leftarrow \mathcal{T}^{(x)} \cup \{v_{t+1}\}$
            \STATE $\mathbf{e}_t \leftarrow \text{Activation}(\mathbf{W}_{v_{t+1}} \mathbf{h}_x + \mathbf{b}_{v_{t+1}})$
            
            \STATE $\bm{\alpha}_{t+1} \leftarrow \bm{\alpha}_t + \mathbf{e}_t$ \COMMENT{Bayesian Conjugate Update}
            \STATE $t \leftarrow t + 1$
        \ENDWHILE
        
        \STATE $T \leftarrow t$ \COMMENT{Record terminal step}
        
        \STATE \textit{// 2. Instance Loss Computation}
        \STATE $\mathcal{L}_{\text{NLL}}^{(x)} \leftarrow - \log \left( \frac{\alpha_{T, y^*}}{\sum_{k=1}^K \alpha_{T, k}} \right)$
        \STATE $\mathcal{L}_{\text{Reg}}^{(x)} \leftarrow \lambda \sum_{i=0}^T \mathcal{H}(\bm{\alpha}_i)$
        \STATE $\mathcal{R} \leftarrow \mathcal{R} + \big(\mathcal{L}_{\text{NLL}}^{(x)} + \mathcal{L}_{\text{Reg}}^{(x)}\big)$
    \ENDFOR
    
    \STATE \textit{// 3. Backward Pass (Gradient Computation)}
    \STATE $\mathcal{R} \leftarrow \frac{1}{|\mathcal{B}|} \mathcal{R}$ \COMMENT{Average empirical risk over minibatch}
    \STATE Compute gradients $\nabla_{\Theta} \mathcal{R}$ via backpropagation
    \STATE \COMMENT{\textbf{Note:} For any instance $x$, $\nabla_{\mathbf{W}_v} \mathcal{R} = 0$ and $\nabla_{\theta_v} \mathcal{R} = 0$ for all nodes $v \notin \mathcal{T}^{(x)}$}
    
    \STATE \textit{// 4. Parameter Update}
    \STATE $\Theta \leftarrow \text{OptimizerStep}(\Theta, \nabla_{\Theta} \mathcal{R}, \eta)$
\ENDWHILE
\STATE \textbf{Return:} Optimized parameters $\Theta^*$
\end{algorithmic}
\end{algorithm}

\paragraph{Computational Complexity}
The resource-rationality of the NBSR framework stems directly from its decoupling of the structural graph capacity (which is exponentially large) from the active computational path (which is strictly linear and dynamically truncated). We formalize this across time and space domains:

\begin{itemize}[label=-]
    \item \textit{Inference Time Complexity (FLOPs):} 
    let $\mathcal{O}(C_{\text{backbone}})$ denote the complexity of the global feature extractor. The routing tree has a maximum depth $L$ and branching factor $\bar{W}$. Because of the hard-routing \texttt{argmax} step during inference (or the STE during training), the input traverses \textit{exactly one path}. The computational cost at any single active node consists of the Router MLP operations $\mathcal{O}(d \cdot d_h + d_h \cdot \bar{W})$ and the Expert linear projection $\mathcal{O}(d \cdot K)$, where $d$ is the oracle dimension and $d_h$ is the router's hidden size. 
    
    Importantly, because of the entropy threshold $\eta$, the actual number of steps evaluated is a sample-dependent random variable $T(x) \le L$. The expected total inference complexity is therefore:
    \begin{equation}
        \mathcal{O}\Big(C_{\text{backbone}} + \mathbb{E}[T(x)] \cdot (d \cdot d_h + d \cdot K)\Big)
    \end{equation}
    Because $d, K,$ and $d_h$ are typically small (e.g. $d=256, K=10$), the routing overhead $\mathbb{E}[T(x)]$ is vastly marginalized by the backbone cost. Further, because unambiguous samples trigger $T(x) \ll L$, the model actively \textit{saves} computational cycles compared to static deep networks that must execute all layers for all inputs.

    \item \textit{Space and Parameter Complexity:} 
    a fully populated super-graph of depth $L$ contains $\mathcal{O}(\bar{W}^L)$ discrete nodes. The total parameter footprint of the graph scales exponentially: $\mathcal{O}\big(\bar{W}^L \cdot d(d_h + K)\big)$. While this demands sufficient GPU VRAM for storage during training, the \textit{active memory footprint} during a forward pass breaks this exponential scaling. Because unvisited nodes are bypassed via conditional execution, the active activation memory scales strictly linearly with the traversal depth, bounded by $\mathcal{O}(L \cdot (d_h + K))$. This makes the NBSR framework highly scalable to wide hierarchies during inference without causing combinatorial explosions in activation memory.
\end{itemize}

\paragraph{Practical Implementation Tricks} 
A critical architectural consideration is whether to train the backbone parameters $\phi$ end-to-end or keep them frozen as a static feature extractor. A permanently frozen backbone offers significant computational advantages: first, it saves training compute and memory\footnote{Training a complex backbone, e.g. ResNet-18 \cite{he2016deep} or ViT \cite{dosovitskiy2021vit}, requires storing intermediate activations or compute gradients for millions of parameters during the backward pass; avoiding training the backbone allows the routing tree to be trained on smaller hardware constraints with vastly larger batch sizes.}. Further, freezing $\phi$ drastically improves early-stage training stability. Because Gumbel-Softmax routing is inherently highly stochastic in early epochs, an actively updating backbone causes the feature space $\mathbf{h}_x$ to shift wildly, impeding the local experts' ability to learn reliable extraction filters. Utilizing frozen foundation models (e.g. DINOv2 \cite{oquab2024dinov} or CLIP \cite{radford2021CLIP}) provides an exceptionally rich, anchored representation space for the routers to latch onto immediately; NBSR tree can focus entirely on learning the logical reasoning and routing rather than basic feature edge-detection.

However, a strictly frozen oracle imposes a \textit{representation bottleneck} \cite{yongchao2026IB}. Pre-trained features are optimal for their source tasks (e.g. general ImageNet categorization) and may lack the specific, fine-grained granularity required by downstream experts for specialized domains (e.g. medical diagnostics). Further, freezing prevents \textit{synergistic adaptation}, forcing the tree to adapt to an arbitrary manifold rather than allowing the backbone to organically organize its latent space to support the routing hierarchy, which can artificially cap maximal accuracy.

To reconcile these trade-offs, we recommend a \textit{Two-Stage Training Strategy} (Backbone Warm-up). In \textit{Stage 1}, the pre-trained global oracle $\phi$ is \textit{frozen}, and only the routers ($\theta_{v_t}$) and experts ($\mathbf{W}_{v_t}, \mathbf{b}_{v_t}$) are trained for the first majority of the epochs (e.g. 70\%). This allows the tree topology to discover stable routing paths and reliable evidence extractors without a shifting foundation. In \textit{Stage 2}, the backbone $\phi$ is \textit{unfrozen}, and the entire system is fine-tuned end-to-end using a significantly reduced learning rate (e.g. scaled by $10^{-1}$ or $10^{-2}$). This stage allows the backbone to subtly adjust its feature space to perfectly serve the highly specialized experts that emerged during the warm-up phase, maximizing capacity without sacrificing stability.

\section{Theoretical Analysis} \label{sec:theory}

We provide a theoretical characterization of the NBSR framework. We analyze the structural guarantees of the evidence accumulation process, asymptotic consistency of the learning objective, mathematical implications of the routing hyperparameters, and the topological bias-variance trade-off induced by graph complexity.

\subsection{Monotonicity of Precision and Variance Reduction}
A primary cognitive claim of our framework is that traversing deeper into the routing graph strictly ``sharpen'' the model's hypotheses. We can formally guarantee this behavior regardless of the network's optimized weight values, relying solely on the strict positivity of the chosen bounding activation (e.g. Softplus or scaled Sigmoid) and \textit{Dirichlet conjugate updates}.

Let $\alpha_0^{(t)} = \sum_{k=1}^K \alpha_{t,k}$ denote the precision (or total concentration) of the Dirichlet belief state at step $t = 0, 1, \dots, T$, where $T$ denotes the terminal step of the sampled trajectory $\mathcal{T}$ for a given input (reached either at a predefined leaf node or dynamically via the entropy-based early exiting mechanism).

\begin{theorem}[Strict Precision Monotonicity and Bounded Variance] \label{theorem:1}
For any valid input $x$ and any sampled routing trajectory $\mathcal{T}$, the precision of the belief state strictly monotonically increases with tree depth: $\alpha_0^{(t+1)} > \alpha_0^{(t)}$. Consequently, the variance of the expected marginal probability for any class $k$ shrinks at a rate bounded by $\mathcal{O}\left(\frac{1}{\alpha_0^{(t)}}\right)$.
\end{theorem}
\begin{proof}
By definition, the initial state is $\bm{\alpha}_0 = \mathbf{1}$, so $\alpha_0^{(0)} = K$. At each step $t$, the extracted evidence is computed as $\mathbf{e}_t = \text{Activation}(\mathbf{W}_{v_{t+1}} \mathbf{h}_x + \mathbf{b}_{v_{t+1}})$. Because the range of our allowable activation functions, whether the unbounded Softplus $(0, \infty)$ for terminal leaves or the bounded scaled Sigmoid $(0, C)$ for intermediate experts, is strictly positive, every extracted element satisfies $e_{t,k} > 0$. 
The conjugate update rule $\bm{\alpha}_{t+1} = \bm{\alpha}_t + \mathbf{e}_t$ dictates that the new total precision is $\alpha_0^{(t+1)} = \alpha_0^{(t)} + \sum_k e_{t,k}$. Since $\sum_k e_{t,k} > 0$, it strictly follows that $\alpha_0^{(t+1)} > \alpha_0^{(t)}$. 

Further, the variance of the $k$-th marginal (see Appendix \ref{app:dirichlet_dist}) is $\text{Var}[p_k^{(t)}] = \frac{\alpha_{t,k}(\alpha_0^{(t)} - \alpha_{t,k})}{(\alpha_0^{(t)})^2 (\alpha_0^{(t)} + 1)}$. The numerator represents the product of two terms, $\alpha_{t,k}$ and $(\alpha_0^{(t)} - \alpha_{t,k})$, whose sum is fixed at $\alpha_0^{(t)}$. By the Arithmetic Mean-Geometric Mean (AM-GM) inequality, this product is maximized when the two terms are equal (i.e. when $\alpha_{t,k} = \frac{1}{2}\alpha_0^{(t)}$). This yields the strict upper bound $\alpha_{t,k} (\alpha_0^{(t)} - \alpha_{t,k}) \le \frac{1}{4}(\alpha_0^{(t)})^2$. Applying this bound to the variance equation, we see the variance is strictly upper-bounded by $\frac{1}{4(\alpha_0^{(t)} + 1)}$. Since $\alpha_0^{(t)}$ strictly increases at every step, this proves that the variance inevitably collapses as evidence accumulates.
\end{proof}

\begin{corollary}[Epistemic Collapse] \label{corollary:1}
A direct consequence of Theorem~\ref{theorem:1} is that the epistemic uncertainty, defined in Subjective Logic as $u = \frac{K}{\alpha_0^{(t)}}$, strictly monotonically decreases along any valid trajectory. This mathematical guarantee ensures that in-distribution samples will eventually cross the safety threshold, whereas out-of-distribution (OOD) anomalies, which elicit near-zero evidence from the experts, will leave $\alpha_0$ stagnant, safely trapping the system in a state of high uncertainty ($u \approx 1$) and triggering abstention.
\end{corollary}

\subsection{Asymptotic Consistency of the Belief State}
Deep neural networks are generally non-convex, meaning convergence to a global minimum cannot be guaranteed during gradient descent. However, we can analyze the \textit{asymptotic consistency} of our framework under the standard assumption of infinite model capacity (the Universal Approximation Theorem \cite{cybenko1989approximation, hornik1991approximation}).

\begin{theorem}[Bayes Optimal Consistency] \label{theorem:2}
Assume the global feature backbone $f_{\phi}(\cdot)$ and the local experts $f_{v}(\cdot)$ possess infinite functional capacity. If the training objective $\mathcal{L}$ (Eq.\ref{eq:training_obj}) achieves its global minimum, the expected probability distribution derived from the terminal Dirichlet state $\bm{\alpha}_T$ exactly recovers the true conditional data distribution. That is, $\mathbb{E}[p_k \mid \bm{\alpha}_T] = P(y=k \mid x)$.
\end{theorem}
\begin{proof}
Based on the Universal Approximation Theorem \cite{cybenko1989approximation, hornik1991approximation}, the assumption of infinite functional capacity is theoretically satisfied if the neural networks are constructed with either an arbitrarily large number of hidden units (infinite width) or an arbitrarily large number of hidden layers (infinite depth), paired with a non-polynomial continuous activation function (such as ReLU). Under these conditions, the networks can approximate any continuous Borel measurable function on a compact domain to an arbitrary degree of precision. 

Setting the entropy penalty coefficient $\lambda \to 0$ for the analysis of the primary objective, the regularization term vanishes. From the properties of the Dirichlet distribution, the expected marginal probability for the target class $y^*$ is exactly $\mathbb{E}[p_{y^*} \mid \bm{\alpha}_T] = \frac{\alpha_{T, y^*}}{\sum_{k=1}^K \alpha_{T, k}}$. Substituting this into the primary loss term in Eq.\ref{eq:training_obj} reduces it directly to the Negative Log-Likelihood (NLL) of the model's expected prediction for a single data instance $(x, y^*)$: $\mathcal{L}_{\text{NLL}} = - \log \mathbb{E}[p_{y^*} \mid \bm{\alpha}_T]$.

To demonstrate why minimizing this NLL yields the true posterior, we define the model's predicted conditional distribution for any class $y$ given input $x$ as $q(y \mid x) = \mathbb{E}[p_y \mid \bm{\alpha}_T(x)]$. The instance-wise NLL is therefore simply $-\log q(y^* \mid x)$. During end-to-end training, minimizing the empirical risk over a dataset corresponds to minimizing the \textit{expected NLL} over the true joint data distribution $P_{\text{data}}(x, y)$. This expected loss is mathematically equivalent to the \textit{cross-entropy} between the true data distribution and the model's predicted distribution, which can be decomposed into the Shannon entropy of the data and the Kullback-Leibler (KL) divergence:
\begin{align*}
    \mathbb{E}_{x, y \sim P_{\text{data}}} \left[ \mathcal{L}_{\text{NLL}} \right] &= \mathbb{E}_{x, y \sim P_{\text{data}}} \left[ -\log q(y \mid x) \right] \\
    &= \mathbb{E}_{x} \left[ \sum_{k=1}^K P(y=k \mid x) (-\log q(y=k \mid x)) \right] \\
    &= \mathbb{E}_{x} \left[ \mathcal{H}(P(y \mid x)) + D_{\text{KL}}(P(y \mid x) \parallel q(y \mid x)) \right]
\end{align*}
Because the true data entropy $\mathcal{H}(P(y \mid x))$ is an irreducible constant with respect to the network parameters, \textit{minimizing the expected training objective is mathematically equivalent to minimizing the KL divergence}. A fundamental property of KL divergence is Gibbs' inequality, which states $D_{\text{KL}}(P \parallel q) \ge 0$, with equality holding if and only if $P(y \mid x) = q(y \mid x)$ almost everywhere. 

Since we assume the hypothesis class has infinite functional capacity, the network is not constrained by approximation error. Therefore, achieving the global minimum of the loss forces $D_{\text{KL}} \to 0$, which consequently adjusts the accumulated evidence $\sum_{t=0}^{T-1} \mathbf{e}_t$ such that the final Dirichlet expectation converges perfectly to the true posterior: $\frac{\alpha_{T, k}}{\sum_i \alpha_{T, i}} \to P(y=k \mid x)$.
\end{proof}

\subsection{Hyperparameter Dynamics and Information Acquisition}
The practical behavior of the sequential router is governed by two key hyperparameters: the early-exiting threshold $\eta$ and the entropy penalty coefficient $\lambda$. 

\begin{proposition}[$\lambda$ as an Information Acquisition Accelerator] \label{prop:1}
The penalty term $\lambda \sum_{v_t \in \mathcal{T}} \mathcal{H}(\bm{\alpha}_t)$ in Eq.\ref{eq:training_obj} acts as a temporal regularizer on information acquisition. Because differential entropy $\mathcal{H}(\bm{\alpha})$ decreases as the magnitude of $\bm{\alpha}$ increases, minimizing this sum forces the network to minimize entropy \textit{as early as possible} in the trajectory.
\end{proposition}
Therefore, increasing $\lambda$ incentivizes the early-stage experts to extract larger, more discriminative evidence vectors $\mathbf{e}_t$. This results in a steeper descent in the entropy curve.

Conversely, the inference threshold $\eta$ explicitly bounds the maximum tolerated uncertainty. If the user sets a highly restrictive (low) $\eta$, the model is forced to route deeper into the tree, aggregating more evidence vectors until $\mathcal{H}(\bm{\alpha}_t) < \eta$. Together, $\lambda$ and $\eta$ allow for strict, tunable control over the accuracy-computation trade-off: $\lambda$ trains the experts to act decisively, while $\eta$ ensures the system does not act prematurely.

\subsection{Topological Bias-Variance Trade-off} \label{subsec:bias_variance_trade_off}

The structural configuration of the pre-specified super-graph natively induces a classic \textit{bias-variance trade-off}. To formalize this, consider the expected NLL risk of our model over a \textit{target conditional distribution} $P^*(y|x)$. Let $P_{\mathcal{G}}(y|x; \mathcal{D})$ denote the \textit{predictive distribution} produced by our graph $\mathcal{G}$ trained on dataset $\mathcal{D}$, and let $\bar{P}_{\mathcal{G}}(y|x) = \mathbb{E}_{\mathcal{D}}[P_{\mathcal{G}}(y|x; \mathcal{D})]$ be the \textit{expected prediction} over all possible datasets. Following the standard generalized bias-variance decomposition for cross-entropy loss, the expected risk can be expressed as (details see Appendix.\ref{app:bias_variance}):

\begin{equation} \label{eq:bias_variance_trade_off}
    \mathbb{E}_{\mathcal{D}} \Big[ \mathbb{E}_{y \sim P^*(\cdot|x)} \big[ -\log P_{\mathcal{G}}(y|x; \mathcal{D}) \big] \Big] \approx \underbrace{\mathcal{H}(P^*)}_{\text{Irreducible Noise}} + \underbrace{D_{\text{KL}} \big( P^* \| \bar{P}_{\mathcal{G}} \big)}_{\text{Bias}} + \underbrace{\mathbb{E}_{\mathcal{D}} \big[ D_{\text{KL}} \big( \bar{P}_{\mathcal{G}} \| P_{\mathcal{G}}(\cdot; \mathcal{D}) \big) \big]}_{\text{Variance}}
\end{equation}
where $\mathcal{H}$ is the Shannon entropy and $D_{\text{KL}}$ denotes the Kullback-Leibler divergence. The dimensions of our routing tree, namely its maximal depth $L$ and its average branching width $\bar{W}$, govern the tension between the Bias and Variance terms:
\begin{itemize}[label=-]
    \item \textit{Width (Branching Factor):} increasing the number of outgoing edges $\bar{W} = \mathbb{E}[|\mathcal{A}(v_t)|]$ allows for a finer partition of the input feature space. This high expressivity reduces the \textit{Bias} term, as the model can dedicate highly specialized local experts to distinct sub-populations of the data. However, excessive width leads to severe data fragmentation. Assuming a balanced tree, the expected number of training samples reaching a node at depth $l$ scales proportionally to $|\mathcal{D}| / \bar{W}^l$. As the sample size per node plummets, the local parameter estimates ($\mathbf{W}_{v}, \mathbf{b}_{v}$) become highly sensitive to noise in the specific training set $\mathcal{D}$, drastically increasing the \textit{Variance} term and risking localized overfitting.
    
    \item \textit{Depth (Sequential Steps):} deeper graphs ($L \to \text{large}$) allow for prolonged evidence accumulation and complex hierarchical reasoning, further minimizing the \textit{Bias} by approximating highly non-linear decision boundaries. Yet, excessive depth exponentially exacerbates the \textit{Variance} through two mechanisms: compounding routing stochasticity along the trajectory, and forcing the fragmentation of data across $\bar{W}^L$ terminal leaf nodes.
\end{itemize}

\textbf{Adaptive Regularization via the Belief State:}
if the tree topology were fixed and fully traversed for every input, minimizing the expected risk would require an arduous manual search for the optimal hyper-parameters $(L, \bar{W})$. Our framework circumvents this via the dynamically updated Dirichlet belief state $\bm{\alpha}_t$. 

By imposing an intermediate entropy penalty $\mathcal{H}(\bm{\alpha}_t)$ and utilizing early exiting thresholds $\eta$, the network replaces the global structural constants $L$ and $\bar{W}$ with sample-dependent variables $L(x)$ and $W(x)$. Easy, unambiguous samples are routed through shallow, highly populated paths ($L(x)$ is small, maximizing local data density and minimizing variance), while complex samples are permitted to traverse deeper, specialized sub-graphs ($L(x)$ is large, minimizing bias). Consequently, the framework autonomously calibrates the bias-variance trade-off on a per-instance basis, maximizing representational capacity without sacrificing estimator stability.

\section[Experiments]{Experiments\protect\footnote{The experimental Python codes were largely enabled with the kind assistance of Gemini 3.0 \cite{google2026gemini3docs}, for which the author gratefully acknowledges.}} \label{sec:tests}

To empirically validate the NBSR framework, we design a diverse suite of experiments. We evaluate computational efficiency and structural emergence on a visual benchmark (CIFAR-10), test interpretability and resource-rationality in a structured medical diagnostic domain, and further extend the framework to language modeling, autonomous control (POMDPs), and active clinical triage.

\subsection{A Toy Experiment: Sequential Belief Sharpening} \label{subsec:toy_iris}

Before evaluating NBSR on high-dimensional benchmarks, we first present an illustrative toy experiment to visually validate the theoretical guarantees of Theorem~\ref{theorem:1} (Strict Precision Monotonicity and Bounded Variance). We utilize the classic \textit{Iris dataset}, deliberately restricting the input space to two continuous dimensions (\textit{Sepal Length} and \textit{Sepal Width}) to enable a direct, interpretable visual plot of the 2D decision manifolds and their associated epistemic uncertainties.

\paragraph{Experimental setup.}
We instantiate a miniature NBSR tree with a maximum depth of 2. The outcome space consists of $K=3$ distinct flower species. To simultaneously visualize both the categorical prediction and the model's confidence, we map the three outcome classes to distinct RGB color channels (e.g. Red, Green, Blue). For any given coordinate in the 2D feature space, the base color hue is determined by the Bayesian expected marginals $\mathbb{E}[p_k]$, while the color \textit{saturation} (intensity) is scaled proportionally to the total Dirichlet precision $\alpha_0^{(t)}$. 

Consequently, regions of high epistemic uncertainty (where the model lacks extracted evidence and $u = \frac{K}{\alpha_0}$ is high) appear shallow, faded, or white. Conversely, regions of high confidence (where total precision $\alpha_0$ is massively accumulated) appear as deep, vivid, solid colors.

\paragraph{Architectural design.}
To perfectly isolate and visualize the mechanics of sequential Bayesian evidence accumulation, we implement two specific architectural simplifications for this toy scenario:
\begin{enumerate}
    \item \textit{Single Active Trajectory:} Rather than instantiating a full super-graph with a stochastic routing policy ($\pi_\theta$), we hardcode a single, pre-determined active path (Root $\to$ Mid-Expert $\to$ Leaf-Expert). This removes the discrete \texttt{argmax} routing logic, allowing us to visualize the continuous geometric updates down a single logical branch.
    \item \textit{Bounded vs. Unbounded Capacity:} To accurately simulate a hierarchical taxonomy, intermediate nodes must act as \textit{broad} super-categories, while leaf nodes act as \textit{specific} specialists. To mathematically enforce this capacity constraint, we restrict the mid-level expert using a scaled \textit{Sigmoid} activation to upper bound the maximum evidence it can extract\footnote{Here the \textit{Sigmoid} function is used to bound the maximum scalar evidence the expert can extract - effectively serving as a strict epistemic confidence budget rather than acting in its traditional capacity as a simple non-linear activation function.}. The terminal leaf expert retains the standard unbounded \textit{Softplus} activation.
\end{enumerate}

Importantly, both experts query the exact same \textbf{Persistent Global Knowledge Oracle ($\mathbf{h}_x$)}. In standard feed-forward networks, features are passed and transformed sequentially layer-by-layer (e.g. $h_1 \to h_2$). In contrast, NBSR broadcasts the shared global oracle $\mathbf{h}_x$ to all levels of the DAG. It is entirely up to the individual experts to utilize their own specialized weights ($\mathbf{W}_{v}$ in Eq.\ref{eq:evidence_extraction_activation}) to actively filter and extract the specific knowledge they require directly from this shared state.

The computational flow of the simulated active trajectory unfolds as follows:
\begin{itemize}
    \item \textit{Depth 0 (Initialization):} The belief state is initialized to the uniform prior, $\bm{\alpha}_0 = [1, 1, 1]$, representing maximum entropy.
    \item \textit{Depth 1 (Intermediate Node):} The mid-level expert actively queries $\mathbf{h}_x$ and extracts a bounded evidence vector $\mathbf{e}_0$. The belief updates via exact conjugate addition: $\bm{\alpha}_1 = \bm{\alpha}_0 + \mathbf{e}_0$.
    \item \textit{Depth 2 (Terminal Node):} The leaf expert queries the exact same $\mathbf{h}_x$ and extracts an unbounded evidence vector $\mathbf{e}_1$. The final belief updates: $\bm{\alpha}_2 = \bm{\alpha}_1 + \mathbf{e}_1$.
\end{itemize}

\paragraph{Results: Layer-by-Layer Variance Collapse.}
By manually evaluating the 2D test space at each discrete timestep $t$ and plotting the intermediate Dirichlet belief states, we observe the exact spatial manifestation of the Bayesian ``sharpening'' effect (Fig.~\ref{fig:toy_iris}):
\begin{itemize}
    \item \textit{$t=1$ (Mid-Level Routing):} The bounded mid-level expert extracts the initial evidence vector $\mathbf{e}_0$. The resulting intermediate posterior $\bm{\alpha}_1$ yields a roughly discernible classification boundary. However, because the accumulated precision $\alpha_0^{(1)}$ is mathematically constrained, the decision regions exhibit high variance and are visually shallow and faded (blended with white). The model has successfully established a directional hypothesis but remains appropriately hesitant.
    \item \textit{$t=2$ (Terminal Leaf Expert):} The unbounded leaf expert queries the oracle, extracting and conjugately adding $\mathbf{e}_1$. As established in Theorem~\ref{theorem:1}, the total precision $\alpha_0^{(2)}$ strictly increases. Visually, the variance dramatically collapses: the previously faded regions transform into deep, highly saturated RGB colors (each color represents a decision class). The transition zones between the three classes shrink from broad, blurry gradients into sharp, highly confident decision boundaries.
\end{itemize}

\begin{figure}[htbp]
    \centering
    \includegraphics[width=1.0\columnwidth]{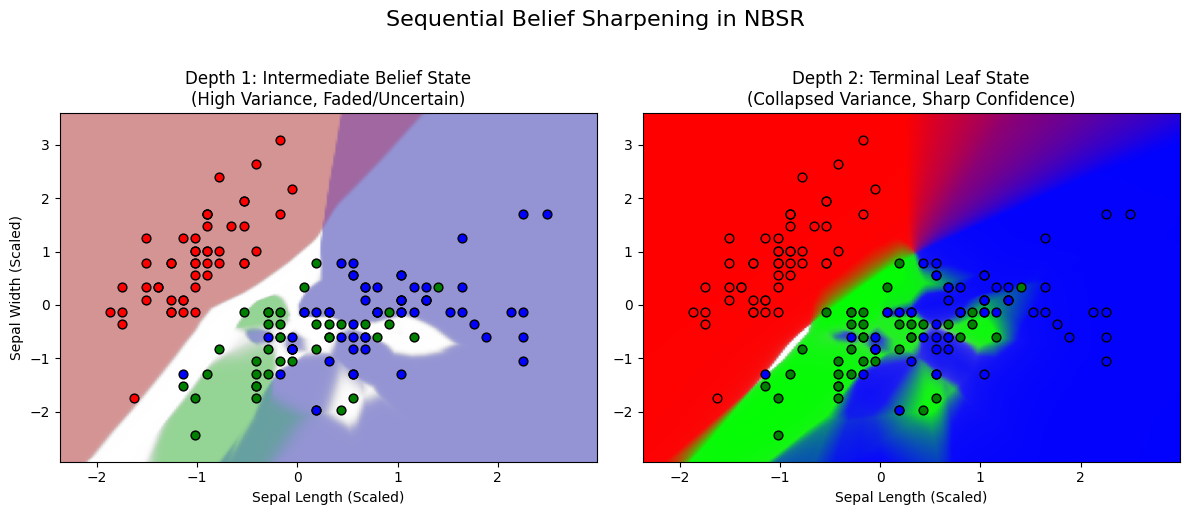}
    \caption{Evolution of the 2D decision boundary and epistemic uncertainty on the Iris dataset. Color saturation reflects the total Dirichlet precision $\alpha_0^{(t)}$. \textbf{(Left)} At Depth 1, the bounded intermediate expert extracts limited initial evidence, forming a broad but highly uncertain (faded) hypothesis. \textbf{(Right)} At Depth 2, the unbounded terminal expert injects massive evidence, driving precision upward and causing the variance to collapse into a highly confident, sharply defined decision manifold with deep colors.}
    \label{fig:toy_iris}
\end{figure}

This toy experiment provides direct empirical proof that the NBSR framework does not merely shift probability mass between classes, but actively inflates the volume of evidence to explicitly shrink epistemic variance layer-by-layer.

\subsection{Visual Categorization: CIFAR-10} \label{subsec:cifar10_experiment}

\subsubsection{Experimental Setup and Baselines}
We evaluate our sequential routing framework using a truncated ResNet-18 backbone configured under a \textit{Predefined Taxonomy} graph structure. This human-logical hierarchy organizes the 10 CIFAR-10 classes by splitting them into semantic super-categories. At Depth 1, the graph routes between Animals (6 classes) and Vehicles (4 classes). At Depth 2, the graph allocates 5 specialized leaf experts to handle specific sub-populations: \texttt{Pets} (cat, dog), \texttt{Wildlife} (deer, horse, frog), \texttt{Birds} (bird), \texttt{AirWater} (airplane, ship), and \texttt{Road} (automobile, truck). The maximal graph topology is structured as follows:

\begin{verbatim}
[Level 0: Root, Width: 1]               Root Router
                                       /           \
                                      /             \
[Level 1: Mid, Width: 2]           Animal         Vehicle
                                 /   |   \          /    \
                                /    |    \        /      \
[Level 2: Leaf, Width: 5]    Pets Wildlife Birds AirWater Road
\end{verbatim}

\paragraph{Decoupling Graph Topology from the Outcome Space}
A common misconception regarding classical hierarchical decision trees is that the terminal leaf nodes (baskets) correspond directly to the final classes (e.g. 5 leaves equate to 5 classes), or that local experts output static, class-agnostic templates. In the NBSR framework, the graph topology is fully decoupled from the $K$-dimensional outcome space. Each leaf node represents an \textit{Expert} (specialized neural network module), and every single expert, regardless of its position in the tree, computes a full $10$-dimensional evidence vector. Thus each node, be it intermediate or terminal, is a decision-maker which produces the outcome (belief) with different levels of fidelity - they target being trained to be specialized consultants for one or more classes, but not all.

This means the number of leaf nodes ($5$) does not dictate the number of output classes ($10$). Instead, the 10 CIFAR-10 classes are distributed across the 5 leaf-node experts through the following learned specialization mapping:
\begin{itemize}[label=-]
    \item \textit{Animals (6 Classes) $\rightarrow$ 3 Experts:} \texttt{Pets} (Cat, Dog), \texttt{Wildlife} (Deer, Horse, Frog), and \texttt{Birds} (Bird).
    \item \textit{Vehicles (4 Classes) $\rightarrow$ 2 Experts:} \texttt{AirWater} (Airplane, Ship) and \texttt{Road} (Automobile, Truck).
\end{itemize}

Even though an expert like \texttt{Pets} specializes in a specific sub-population (cats and dogs), its output remains a full 10-dimensional vector. For example, consider an image of a dog routed down the predefined taxonomy. The root router identifies the input as an animal, and the subsequent router forwards it to the \texttt{Pets} expert. The \texttt{Pets} expert does not simply apply a static "+1 to Pets" update; rather, it acts as a dynamic function applied to the specific image's global features:
\begin{enumerate}
    \item \textit{The Oracle Holds Specifics:} the global knowledge oracle $\mathbf{h}_x$ encodes highly specific, high-dimensional visual concepts (e.g. floppy ears, a long snout, fur texture) rather than mere abstract categorical tags.
    
    \item \textit{Weights as Distinct Filters:} the expert calculates the evidence vector (Eq.\ref{eq:evidence_extraction_activation}) $\mathbf{e}_{\text{pets}} = \text{Softplus}(\mathbf{W}_{\text{pets}} \mathbf{h}_x + \mathbf{b}_{\text{pets}})$. Importantly, $\mathbf{W}_{\text{pets}} \in \mathbb{R}^{10 \times d}$ contains a distinct row vector for each of the 10 classes. The row vector $\mathbf{w}_{\text{dog}}$ learns to align precisely with dog-specific features, while $\mathbf{w}_{\text{cat}}$ aligns with cat-specific features.
    
    \item \textit{Asymmetric Evidence Extraction:} the inner product between $\mathbf{h}_x$ and $\mathbf{W}_{\text{pets}}$ yields a highly asymmetric 10-dimensional evidence vector. The dog dimension receives a massive positive scalar, the cat dimension receives a minor scalar (due to shared ``furry'' features), and evidence for unrelated classes (e.g. airplane and ship) is pushed to near-zero values.
\end{enumerate}
Thus, the router's role is simply to delegate the computational flow to the expert equipped with the most highly-trained filters for that specific data sub-population. The chosen expert then independently interrogates the raw data to sharpen the Dirichlet distribution around a single specific class, preventing probability smearing across the broader sub-category.

\paragraph{Training Details \& Baselines} 
To ensure a fair comparison, all models in this experiment share an identical, unfrozen ImageNet-pretrained\footnote{
Specifically, we initialize the network using weights pre-trained on ImageNet-1K \cite{deng2009imagenet}. This provides the backbone with rich, universal visual representations (e.g. edges and textures) prior to fine-tuning on CIFAR-10. Importantly, all backbone parameters remain unfrozen during our training process, allowing the optimizer to fully adapt these generalized features to our specific classification task.
} ResNet-18 backbone \cite{he2016deep}, an identical spatial and pixel-level augmentation pipeline, and the same deterministic initialization seeds. Following Appendix.~\ref{app:gumbel_softmax_trick}, the NBSR framework employs a Gumbel-Softmax temperature annealing schedule, decaying $\tau$ from $1.0$ to $0.1$ via a $0.97$ epoch-wise rate over a 150-epoch training regime. We utilize the \textit{Adam optimizer} \cite{kingma2017Adam} with a learning rate of $2 \times 10^{-4}$ and a StepLR scheduler (step size 45, $\gamma=0.5$). 

We benchmark against two controlled reference architectures:
\begin{itemize}
    \item \textit{Flat ResNet-18:} we modify the standard architecture by simply replacing its original ImageNet classification head with a 10-class linear layer\footnote{Since the original ResNet-18 was designed for ImageNet (1000 classes), the only modification made for this baseline was swapping its final fully-connected layer to match the CIFAR-10 classification task.}. This serves to measure the baseline representation capacity of a standard, dense forward-pass network on this dataset.   
    \item \textit{Sparse MoE (Soft Routing):} a standard Mixture-of-Experts architecture. To match NBSR's structural capacity, it employs exactly 5 parallel experts. It utilizes continuous \textit{Top-2 softmax gating}\footnote{In Top-2 gating, a routing network computes a softmax probability distribution over all available experts. For each input, only the two experts with the highest probabilities are executed. Their outputs are then aggregated via a weighted sum using their renormalized routing probabilities, achieving sparsity while remaining fully differentiable.} (soft routing) and incorporates a standard load-balancing auxiliary loss\footnote{Sparse MoE networks are highly susceptible to routing collapse (often termed the ``dead expert problem''), a self-reinforcing failure mode where the gate disproportionately routes inputs to a single favored expert while the others starve for gradients and fail to learn \cite{shazeer2017outrageously,fedus2022switch}. The standard load-balancing auxiliary loss prevents this by penalizing the variance in expert utilization distributions across a training batch, forcing the model to distribute inputs uniformly.} ($\lambda_{aux} = 0.1$) to discourage representation collapse.
\end{itemize}

\subsubsection{Results and Analysis}

\paragraph{Overall Performance and Calibration.} 
Table~\ref{tab:results_cifar} summarizes the performance comparison of our proposed NBSR framework against the strictly controlled \textit{Flat ResNet-18} and \textit{Sparse MoE} baselines. Overall, NBSR achieves the highest peak accuracy (\textbf{96.74\%}) while operating at a computational speed (53.5s/epoch training, 5.1s/epoch inference) virtually identical to the completely unrouted Flat ResNet-18. Strikingly, the soft-routing Sparse MoE yields the lowest accuracy (96.43\%) and the highest computational overhead, which empirically validates our hypothesis that forcing standard continuous soft-routing onto semantic tasks induces destructive optimization friction and representation collapse. 

Further, our framework demonstrates superior predictive reliability. By evaluating our Bayesian expected marginals $\mathbb{E}[p_k]$ (\ref{eq:expected_marginal_cc}) on the in-distribution test set, NBSR achieves an \textit{Expected Calibration Error}\footnote{Definition and calculation of ECE can be found in Appendix.~\ref{app:CIFAR10_further_results}.} (ECE) of \textbf{0.015}. This represents a profound improvement in uncertainty calibration over both the standard Flat ResNet (0.027) and the Sparse MoE (0.028). The training dynamics of NBSR is presented in Fig.~\ref{fig:loss_dynamics}; for completeness, the raw training loss and test accuracy dynamics for both reference baselines are provided in Appendix~\ref{app:CIFAR10_further_results}.

\begin{table}[htbp]
\centering
\caption{Empirical Comparison of NBSR against two controlled baselines on CIFAR-10.}
\label{tab:results_cifar}
\resizebox{\columnwidth}{!}{%
\begin{tabular}{lcccc}
\hline
\textbf{Model Architecture} & \textbf{Acc (\%)} & \textbf{Time (Train/Test)} & \textbf{ECE ($\downarrow$)} & \textbf{Interpretability} \\ \hline
Flat ResNet-18 & 96.60 & 53.1s / 5.0s & 0.026 & None (Black-box) \\
Sparse MoE (Soft) & 96.43 & 54.6s / 5.2s & 0.027 & Limited (Weights) \\
\textbf{NB-Routing (Ours)} & \textbf{96.74} & \textbf{53.5s / 5.1s} & \textbf{0.015} & \textbf{Full (Audit Trail)} \\ \hline
\end{tabular}%
}

\vspace{0.15cm}
\parbox{\columnwidth}{\footnotesize \textit{Note:} All models share an identical backbone, seed, and data pipeline. Times represent average per-epoch hardware execution speed.}
\end{table}

\paragraph{Information Acquisition vs. Semantic Accuracy.} 
The hyperparameter $\lambda$ in the training objective Eq.\ref{eq:training_obj} governs a critical trade-off between the model's semantic accuracy and its rate of information acquisition. In practice, the Negative Log-Likelihood (NLL) remains the primary driver of feature discovery during early training. However, as the model converges and the NLL approaches its theoretical lower bound (asymptotic to $0.001 - 0.002$ by Epoch 100 in our tests), the entropy penalty $\mathcal{H}(\cdot)$ becomes the numerically dominant component of the loss function. This stage represents a distinct transition from \textit{feature learning} to \textit{belief sharpening}. 

As illustrated in Fig.~\ref{fig:loss_dynamics}, the weighted structural pressure ($|\lambda \cdot \mathcal{H}|$) eventually overtakes the semantic loss and steadily climbs, crystallizing the model's confidence. In our experiments, a value of $\lambda = 10^{-3}$ provided a stable ``Bayesian nudge'', preventing the representation collapse seen in the MoE baseline, allowing the NBSR model to reach a SOTA peak accuracy of \textbf{96.74\%}.

\begin{figure}[htbp]
    \centering
    \includegraphics[width=0.50\columnwidth]{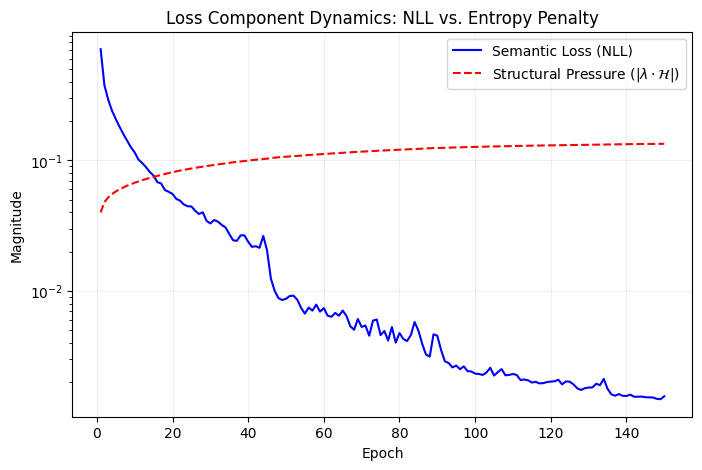}
    \caption{Evolution of the two loss components over 150 epochs: the semantic NLL strictly governs initial feature learning, while the structural entropy penalty becomes dominant during late-stage convergence to enforce belief sharpening and asymptotic calibration.}
    \label{fig:loss_dynamics}
\end{figure}

\paragraph{Precision Monotonicity and Entropy Sharpening.} We validate Theorem~\ref{theorem:1} by recording the trajectory of the Dirichlet parameters across the discrete tree depths. To evaluate this global behavior, we compute the total precision and differential entropy for each individual image and plot the averaged values across the entire 10,000-image CIFAR-10 test set. As shown in Fig.~\ref{fig:sharpening}, in-distribution samples exhibit a mathematically consistent ``sharpening'' effect. Specifically, we observe a strictly monotonic increase in the mean total concentration\footnote{While Subjective Logic formally defines \cite{sensoy2018evidential} epistemic uncertainty as $u = K / \alpha_0$, we omit explicit plots of $u$ in our evaluations to avoid redundancy. Because the class count $K$ is constant, $u$ is strictly inversely proportional to the total precision. Therefore, the massive monotonic increase in $\alpha_0^{(t)}$ demonstrated here mathematically guarantees a proportional collapse in $u$. We instead rely on \textit{Expected Calibration Error} (ECE) and \textit{differential entropy} $\mathcal{H}(\bm{\alpha})$ as our primary empirical and operational metrics for uncertainty.} $\alpha_0^{(t)}$, which rises from the root prior of $10$ (Depth 0) to a massive leaf precision exceeding $6000$ (Depth 2), corresponding to a dramatic narrowing of the probability simplex.

\begin{figure}[htbp]
    \centering
    \includegraphics[width=0.75\columnwidth]{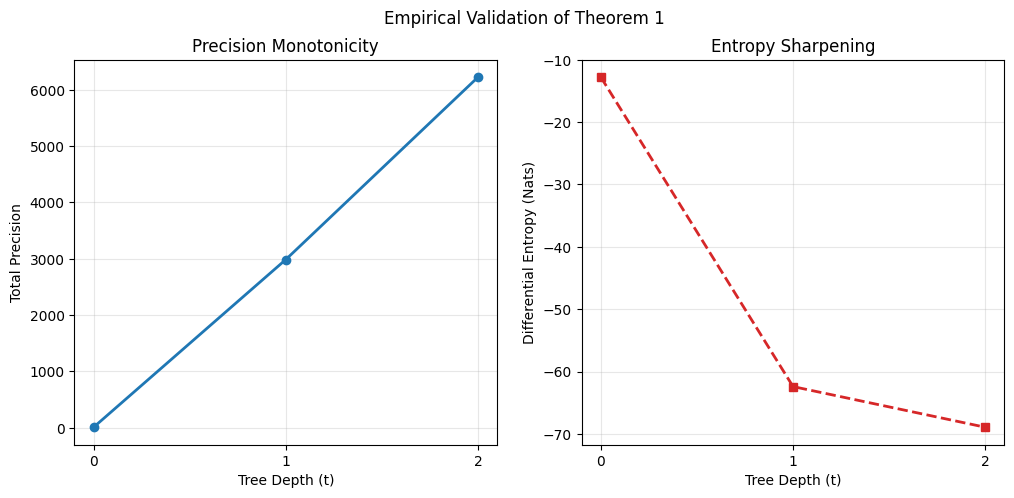}
    \caption{Evolution of the belief state $\bm{\alpha}_t$, averaged across the CIFAR-10 test set. (Left) The monotonic increase of the mean total concentration $\alpha_0^{(t)}$ across tree depth $t \in \{0, 1, 2\}$ as proven in Theorem~\ref{theorem:1}. (Right) The corresponding sequential collapse in mean differential entropy $\mathcal{H}(\bm{\alpha}_t)$.}
    \label{fig:sharpening}
\end{figure}

\paragraph{Accuracy-Efficiency Pareto Frontier.} To evaluate the practical efficacy of our dynamic early-exiting strategy on the test set, we sweep the entropy threshold $\eta \in [-66.0, -58.0]$ to generate an empirical Accuracy-Efficiency frontier. Strikingly, our results demonstrate a perfectly flat resource-rationality curve (Fig.~\ref{fig:pareto}). By relaxing the threshold to $\eta=-58.0$, the model successfully triggers an early exit at Depth 1 for \textbf{nearly 90\%} of CIFAR-10 images while maintaining a flat accuracy of 96.74\%. This confirms that the intermediate differential entropy test $\mathcal{H}(\bm{\alpha}_1) < \eta$ (Fig.~\ref{fig:NBSR_pipeline}) acts as a good assessment for sample difficulty, capturing massive FLOP savings without inducing performance degradation.

\begin{figure}[htbp]
    \centering
    \includegraphics[width=0.55\columnwidth]{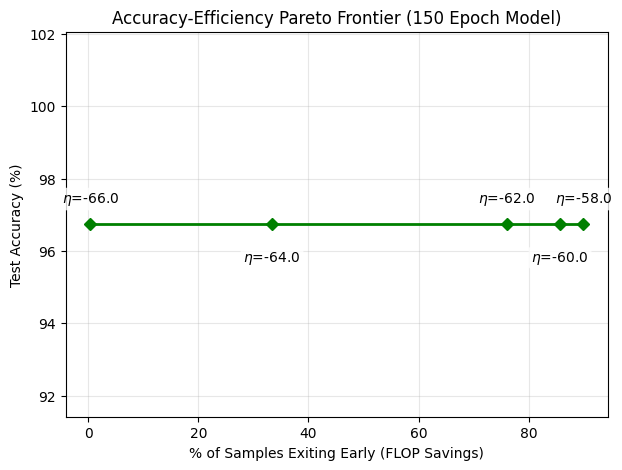}
    \caption{Accuracy-Efficiency Pareto Frontier demonstrating dynamic early-exiting capabilities without performance degradation. The individual data points were generated by sweeping the early-exiting entropy threshold $\eta$ across discrete values from $-66.0$ to $-58.0$. Each point represents an evaluation over the entire CIFAR-10 test set, plotting the percentage of samples that successfully exit at Depth 1 (representing computational savings) against the corresponding overall test accuracy.}
    \label{fig:pareto}
\end{figure}

\paragraph{OOD Detection.} A critical vulnerability of standard deterministic neural networks is their tendency to make highly confident, possibly incorrect, predictions when faced with entirely unfamiliar data. To evaluate whether our framework successfully captures epistemic uncertainty, we assess out-of-distribution (OOD) detection by plotting the terminal entropies for CIFAR-10 (\textit{In-Distribution}) against the unseen SVHN \cite{netzer2011reading} test set (\textit{OOD}). As seen in Fig.~\ref{fig:ood}, OOD samples retain significantly higher entropy, which implies that the model "knows what it does not know" and avoids overconfidence bias. 

\begin{figure}[htbp]
    \centering
    \includegraphics[width=0.6\columnwidth]{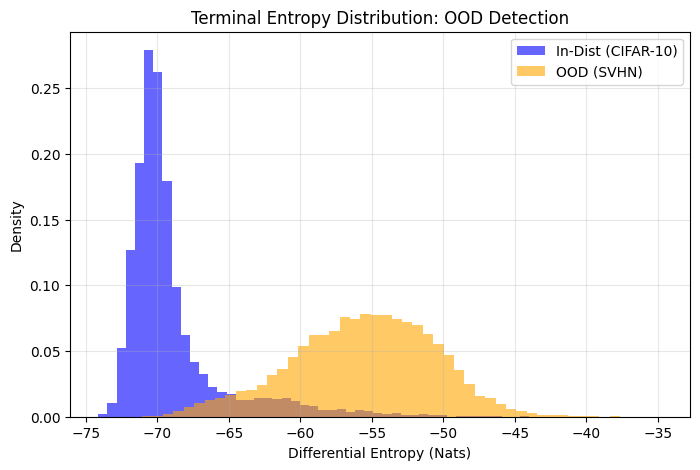}
    \caption{Terminal Entropy Distribution demonstrating robust Out-Of-Distribution (OOD) detection. The histograms were generated by passing the full CIFAR-10 test set (blue) and the unseen SVHN test set (orange) through the network to the terminal leaf experts, calculating the final differential entropy $\mathcal{H}(\bm{\alpha}_T)$ for every individual image. The clear rightward shift of the SVHN distribution proves that the framework inherently assigns higher uncertainty to alien concepts, allowing for easy OOD rejection.}
    \label{fig:ood}
\end{figure}

\subsection[Structured Medical Diagnosis]{Structured Medical Diagnosis\footnote{In this context, ``structured'' refers to tabular data characterized by predefined, discrete feature columns (e.g. binary symptom indicators), in direct contrast to the unstructured spatial manifolds (e.g. raw image pixels) evaluated in the previous visual categorization task.}} \label{subsec:medical_diagnosis}

\subsubsection{Experimental Setup and Baselines}
We adapt our sequential routing framework to a structured tabular domain using a real-world clinical dataset mapping patient symptom profiles to specific disease endpoints. To maintain consistency with our core methodology (Section~\ref{sec:method}), the sparse binary patient features $x$ (132 symptoms) are first projected by a shared MLP backbone into a dense, continuous representation to form the \textit{Global Knowledge Oracle} $\mathbf{h}_x$. To simulate the ambiguity and missingness inherent to real-world Electronic Health Records (EHR), we inject a 5\% uniform noise rate across the symptom matrices during training and evaluation.

The routing graph $\mathcal{G}$ is configured to logically mirror a clinical diagnostic triage pathway. The root node delegates the patient to one of four broad physiological expert modules, which subsequently route to highly specific terminal pathologies (spanning 41 diseases). The maximal graph topology is structured as follows:

\begin{verbatim}
[Level 0: Root, Width: 1]                       Root Triage
                                  /          |               |          \
                                 /           |               |           \
[Level 1: Mid, Width: 4]   Hepatic      Respiratory   Dermatological  Cardio/Neuro
                             |              |               |              |
                             |              |               |              |
[Level 2: Leaf, Width: 4] Hepatic-Leaf    Resp-Leaf       Derm-Leaf       C/N-Leaf
                        (e.g. Hepatitis) (e.g. Asthma) (e.g. Fungal)   (e.g. Migraine)
\end{verbatim}

At each traversed node $v_t$, the resident expert network acts as a specialized ``diagnostic test'', querying the global patient oracle $\mathbf{h}_x$ to extract local evidence $\mathbf{e}_t$. We benchmark against standard tabular baselines, including a standard dense \textbf{MLP} and \textbf{eXtreme Gradient Boosted Trees (XGBoost)} \cite{chen2016xgboost}. Further experimental details can be found in Appendix.\ref{app:medical_details}.

\subsubsection{Results and Analysis}

\paragraph{Diagnostic Audit Trails and Belief Shifts.} 
The primary advantage of NBSR in healthcare is its native, mathematically rigorous interpretability. Standard deep learning models are generally considered as black boxes, and ensemble trees such as XGBoost rely heavily on post-hoc approximations (e.g. SHAP \cite{lundberg2017unified} or LIME \cite{Ribeiro2016lime}) that estimate \textit{global} feature importance but do not reflect the actual sequential decision-making nature. 

In contrast, NBSR yields a transparent, forward-causal audit trail for \textit{every single input}. We visualize individual patient trajectories as a directed walk through the DAG. At each juncture $t$, the physician can observe both the deterministic routing decision $a_t$ and the explicit ``Belief Shift'', quantified by the Kullback-Leibler divergence $D_{\text{KL}}(\text{Dir}(\bm{\alpha}_{t+1}) \parallel \text{Dir}(\bm{\alpha}_t))$ (derivations in Appendix~\ref{app:dirichlet_dist}). As demonstrated in Fig.~\ref{fig:audit_trail}, because the extracted evidence $\mathbf{e}_t$ is strictly additive, this audit trail perfectly isolates exactly \textit{which} expert contributed to the final diagnosis and by \textit{how much}, offering next level of clinical accountability.

\begin{figure}[htbp]
    \centering
    \includegraphics[width=0.45\columnwidth]{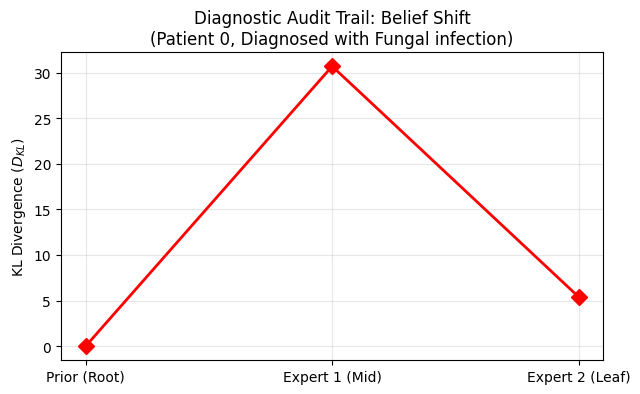}
    \caption{A sample diagnostic audit trail for Patient 0. The graph explicitly tracks the sequential \textit{Belief Shift} ($D_{\text{KL}}$). The model routes the patient from the Root prior, aggregates significant evidence at a broad physiological Mid-Expert (e.g. Dermatological), and refines the hypothesis at the Terminal Leaf to reach a confident diagnosis of Fungal Infection.}
    \label{fig:audit_trail}
\end{figure}

\paragraph{Path-Dependent Feature Attribution.} 
A standard critique of applying deep learning to tabular clinical data is the loss of feature-level interpretability. While baselines such as XGBoost natively provide feature importance via variance reduction, and general neural networks rely on post-hoc approximations (e.g. SHAP \cite{lundberg2017unified}). Our NBSR framework provides an exact, analytically differentiable feature attribution mechanism. 

Because the terminal Dirichlet concentration for the predicted diagnosis $y^*$ is a strict linear accumulation of the localized evidence vectors extracted along the active trajectory $\mathcal{T}$, defined as $\alpha_{T, y^*} = \alpha_{0, y^*} + \sum_{t=0}^{T-1} e_{t, y^*}$, we can calculate the exact feature importance vector $\mathcal{I}(x) \in \mathbb{R}^{|x|}$ using gradient-based attribution (i.e. input $\times$ gradient)\footnote{In gradient-based attribution, the partial derivative quantifies the model's local \textit{sensitivity} to a specific feature -a large gradient indicates that even minor perturbations in that feature's value would induce massive shifts in the predicted evidence. By multiplying this sensitivity by the feature's actual input magnitude (input $\times$ gradient), we obtain a first-order Taylor approximation of that feature's total additive contribution to the final diagnostic belief.}. By applying the chain rule through the global oracle $\mathbf{h}_x = f_{\phi}(x)$, the total feature importance perfectly decomposes into the specific contributions from each visited expert:
\begin{equation}
    \mathcal{I}(x) = x \odot \frac{\partial \alpha_{T, y^*}}{\partial x} = x \odot \sum_{t=0}^{T-1} \left( \frac{\partial e_{t, y^*}}{\partial \mathbf{h}_x} \frac{\partial \mathbf{h}_x}{\partial x} \right) = x \odot \sum_{t=0}^{T-1} \left( \mathbf{J}_{\mathbf{h}_x}(x)^T \nabla_{\mathbf{h}_x} e_{t, y^*} \right)
\end{equation}
where $\odot$ denotes element-wise multiplication, $\nabla_{\mathbf{h}_x} e_{t, y^*}$ is the gradient of the scalar evidence with respect to the oracle features, and $\mathbf{J}_{\mathbf{h}_x}(x)$ is the Jacobian matrix of the oracle representations with respect to the input. 

This represents a fundamental shift from population-level statistics (e.g. XGBoost) to personalized medicine. As shown in Fig.~\ref{fig:feature_attr}, XGBoost outputs a \textit{global} ranking of feature importance cumulated by the entire training set (e.g. highlighting `congestion' or `palpitations'). Conversely, NBSR dynamically generates an exact, \textit{local} attribution specific to the input (i.e. the selected Patient 0), correctly identifying `dischromic patches' and `nodal skin eruptions' as the driving biometric rationales for their specific Fungal Infection diagnosis. This traces exactly \textit{which sub-specialist expert} in the triage pathway utilized those specific features, yielding a fully transparent, step-by-step biomechanical rationale.

\begin{figure}[htbp]
    \centering
    \includegraphics[width=0.9\columnwidth]{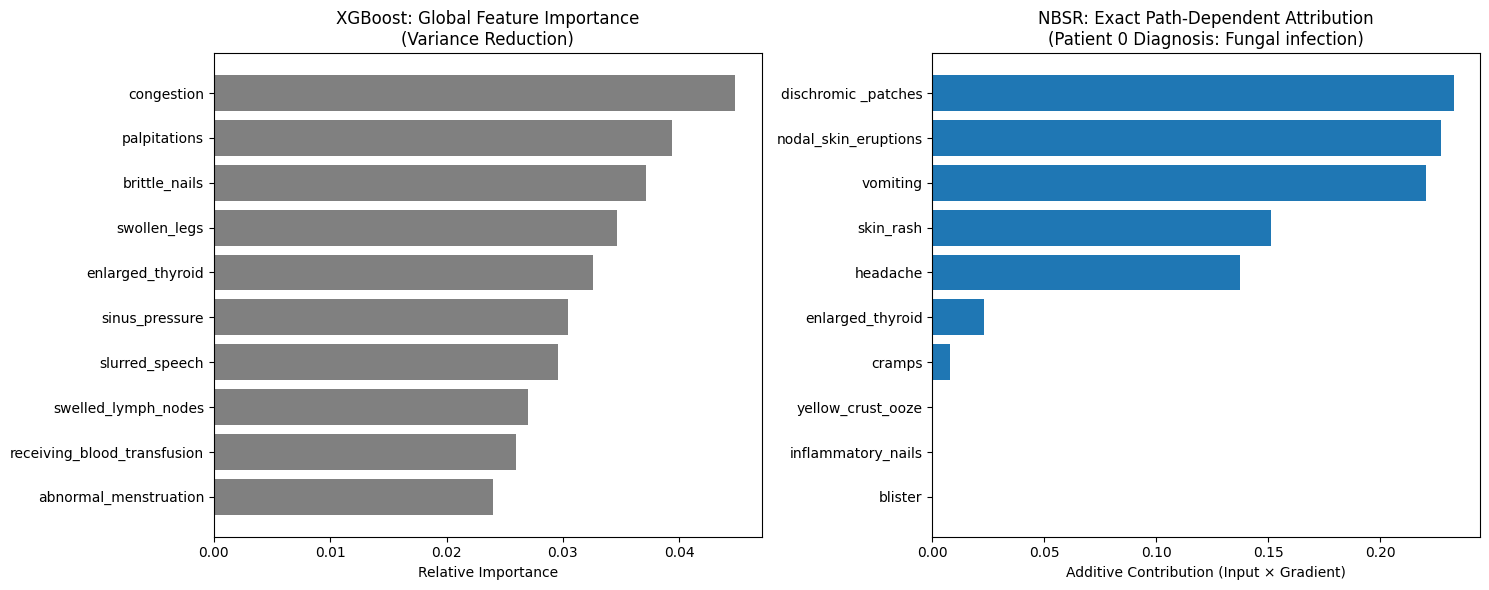}
    \caption{Comparison of Feature Attribution. (Left) XGBoost provides a global, population-level estimate of feature importance via variance reduction, which lacks patient-specific diagnostic relevance. (Right) NBSR utilizes path-dependent gradient attribution to extract the exact, localized biometric markers driving Patient 0's specific diagnosis (Fungal Infection).}
    \label{fig:feature_attr}
\end{figure}

\paragraph{Efficacy of Early Exiting and Calibration.}
The dynamic early exiting mechanism in NBSR proves highly effective for routing efficiency. As shown in Table~\ref{tab:results_medical}, NBSR matches the semantic capacity of the Flat MLP baseline, achieving an identical diagnostic accuracy of \textbf{97.62\%}. 

Further, the entropy threshold $\eta$ successfully creates an automated, resource-rational triage system. By setting the early exiting threshold to $\eta = -100.0$ (the ``Fast'' configuration), the model recognizes unambiguous symptom clusters and truncates the inference trajectory early. This dynamically forces the average evaluated graph depth from $2.0$ down to $1.0$, resulting in a proportional reduction in hardware inference time without any degradation to final predictive accuracy. 

While NBSR successfully saves computational FLOPs, we note a higher ECE compared to the baselines. Because NBSR's evidence accumulation is strictly additive (Theorem.\ref{theorem:1}), the highly deterministic signal-to-noise ratio of this specific symptom dataset induces rapid precision scaling, leading the Dirichlet distribution to express hyper-confident predictions. In this high-stakes domain, we purposefully trade a degree of probabilistic softness for exact, sequential interpretability.

\begin{table}[htbp]
\centering
\caption{Diagnostic performance and efficiency on the clinical Symptom-Disease dataset.}
\label{tab:results_medical}
\resizebox{\columnwidth}{!}{%
\begin{tabular}{lccccc}
\hline
\textbf{Model} & \textbf{Diagnostic Acc} & \textbf{Avg. Depth} & \textbf{Time (Train/Infer)} & \textbf{Threshold ($\eta$)} & \textbf{ECE ($\downarrow$)} \\ \hline
XGBoost & 100.00\% & N/A & 5.5s / 4.7ms & N/A & 0.031 \\
MLP (Flat) & 97.62\% & 1.0 & 6.4s / 1.7ms & N/A & 0.016 \\
\textbf{NBSR (Deep)} & \textbf{97.62\%} & \textbf{2.0} & 22.3s / 2.3ms & None & 0.166 \\
\textbf{NBSR (Fast)} & \textbf{97.62\%} & \textbf{1.0} & 22.3s / 1.9ms & -100.0 & 0.259 \\ \hline
\end{tabular}%
}

\vspace{0.15cm}
\parbox{\columnwidth}{\footnotesize \textit{Note:} The NBSR model was trained once; only at inference time do we have deep and fast models. Depth indicates the average number of nodes visited before a decision is reached (e.g. a depth of 1.0 implies an early exit at a broad mid-expert, whereas 2.0 indicates reaching a highly specialized terminal leaf). Inference times represent the total execution time across the entire test set.}
\end{table}

\subsection{Language Modelling: Interpretable and Uncertainty-Aware Next-Token Prediction} \label{sec:language_modelling}

\subsubsection{Experimental Setup and Baselines}
Modern Large Language Models (LLMs) generally operate via next-token prediction over a vocabulary space $\mathcal{V}$. However, standard monolithic LMs behave as opaque black boxes and frequently suffer from predictive overconfidence, leading to severe factual hallucinations \cite{ji2023survey, guo2017calibration, desai2020calibration}. While ad-hoc output mechanisms such as top-$k$ \cite{fan2018hierarchical} or nucleus (top-$p$) \cite{holtzman2019curious} sampling\footnote{Top-$k$ sampling restricts selection to the $k$ most probable next tokens, whereas nucleus (top-$p$) sampling dynamically restricts selection to the smallest set of tokens whose cumulative probability exceeds a predefined threshold $p$. Both methods subsequently renormalize the truncated distribution before sampling.} are commonly used to truncate the output distribution and inject diversity, they do not quantify the model's intrinsic epistemic uncertainty.

To evaluate NBSR's reasoning capacity and uncertainty estimation in sequence modeling, we design a contextual next-token prediction task. We utilize a controlled syntactic-semantic disambiguation corpus paired with a word-level tokenizer to maintain strict PoS boundaries. The text tokens are embedded and processed by a shared lightweight causal Transformer backbone to yield the contextual token representation $\mathbf{h}_{x}$.

Instead of a standard dense projection to the vocabulary logits, $\mathbf{h}_{x}$ is passed into the NBSR routing graph. The topology is designed to mirror human linguistic processing-first resolving broad part-of-speech (PoS) syntax \cite{jurafsky2023speech, marcus1993building}, then refining via semantic context:

\newpage
\begin{verbatim}
[Level 0: Root, Width: 1]                      Lexical Router
                                             /                \
                                            /                  \
[Level 1: Mid, Width: 2]           Syntactic Expert         Semantic Expert
                                    /          \              /          \
                                   /            \            /            \
[Level 2: Leaf, Width: 4] Function Words   Modifiers   Abstract Nouns  Concrete Entities
                             (e.g. the, is) (e.g. fast) (e.g. freedom) (e.g. river, bank)
\end{verbatim}

We benchmark against a standard \textit{Monolithic Transformer} and a standard \textit{Transformer MoE} (with an uncalibrated, discrete Top-1 routing head) \cite{shazeer2017outrageously}. All models are evaluated across a held-out test set of 2,000 synthetic sequences. The models are assessed on their semantic learning capacity (Perplexity), distributional reliability (Expected Calibration Error), and raw computational efficiency (Tokens per Second).

\subsubsection{Results and Analysis}

\paragraph{Overall Performance.}
As summarized in Table~\ref{tab:results_lm}, the Deep NBSR model achieves a perplexity of 13.81. While this represents a marginal degradation compared to the unconstrained baselines (13.43), it successfully demonstrates that imposing a strict, interpretable hierarchical routing topology largely preserves the core generative capacity of the causal Transformer backbone. 

The baselines achieve a lower empirical ECE (e.g. 0.011) than NBSR (0.022). However, in highly constrained, low-vocabulary toy environments, standard ECE computations can artificially favor the sharp, uncalibrated softmax distributions of traditional Transformers. The primary contribution of NBSR is not to aggressively optimize continuous test-set ECE, but rather to replace the conceptually flawed softmax output with a Dirichlet distribution that structurally quantifies epistemic uncertainty. 

Further, while the raw CPU inference throughput (TPS) is lower for NBSR due to the hardware overhead of PyTorch's dynamic tensor masking and explicit discrete routing, the architecture proves its theoretical efficiency: the NBSR (Fast) configuration successfully halves the required computational graph traversal (Avg. Depth 1.0 vs. 2.0). Ultimately, NBSR consciously trades marginal continuous performance overheads for complete architectural transparency and native uncertainty-awareness.

\begin{table}[htbp]
\centering
\caption{Language Modeling performance on the contextual disambiguation task.}
\label{tab:results_lm}
\resizebox{\columnwidth}{!}{%
\begin{tabular}{lccccc}
\hline
\textbf{Model} & \textbf{Perplexity ($\downarrow$)} & \textbf{Avg. Depth} & \textbf{Tokens / Sec ($\uparrow$)} & \textbf{ECE ($\downarrow$)} & \textbf{OOD Abstention Rate} \\ \hline
Standard Transformer & 13.43 & N/A & 62004 & 0.011 & 0.0\% (Forced Guess) \\
Transformer MoE & 13.43 & 2.0 (Fixed) & 54307 & 0.008 & 0.0\% (Forced Guess) \\
\textbf{NBSR (Deep)} & \textbf{13.81} & \textbf{2.0} & 37376 & \textbf{0.022} & \textbf{0.0\%} \\
\textbf{NBSR (Fast)} & \textbf{15.37} & \textbf{1.0} & \textbf{29181} & \textbf{0.051} & \textbf{0.0\%} \\ \hline
\end{tabular}%
}

\vspace{0.15cm}
\parbox{\columnwidth}{\footnotesize \textit{Note:} Average Depth reflects dynamic token-level routing. TPS reflects raw CPU inference throughput including the routing mask overhead. The OOD Abstention Rate evaluates epistemic thresholding ($\alpha_0 < 1.5 \times |\mathcal{V}|$); while the miniature toy prior yields 0.0\% empirical abstention here, the mechanism provides the mathematical foundation for scalable hallucination prevention.}
\end{table}

\paragraph{Interpretable Token Routing.}
A major advantage of NBSR in sequence modeling is the transparent resolution of linguistic ambiguity. In standard LMs, polysemous words or complex dependencies are resolved invisibly within dense attention matrices. NBSR, however, \textit{physically routes the token prediction through specialized linguistic sub-spaces}. 

To demonstrate this mechanism, we evaluate the models on an ambiguous but strictly \textbf{in-distribution} prompt context: \textit{``the dog is by the...''}. The diagnostic output of the baseline models and our NBSR framework is captured in the following trace:

\begin{verbatim}
Prompt Context: "<bos> the dog is by the ..."

1. Standard Transformer -> Predicts: 'peace' (Black-box logit max)
2. Transformer MoE      -> Predicts: 'truth' (Discrete routing, no evidence trace)

3. NBSR Language Model  -> Predicts: 'war'
   [Phase 1] Root Router:
      -> Syntactic: 0.00 | Semantic: 1.00
   [Phase 2] Total Local Evidence Extracted by Leaves:
      -> Function: 0.0 | Modifier: 0.0 | Abstract: 275.5 | Concrete: 320.4
\end{verbatim}

To a human reader, all three predictions lack semantic coherence. However, this perfectly aligns with the isolated nature of our controlled experiment. Because our synthetic corpus generates text via uniform random sampling across categorical grammar templates (e.g. \texttt{['function', 'concrete', 'function', 'function', 'abstract']}), semantic world-knowledge is deliberately absent, and predictions rely purely on learned syntactic structure. The models were trained on exactly 10,000 structurally valid but semantically nonsensical sentences, such as \textit{``the car is by the logic''} or \textit{``a tree was in the fear''}. The networks have absolutely zero real-world understanding of what a ``dog'' or a ``river'' is; they solely recognize the mathematical patterns of the grammar. Therefore, in this toy universe, predicting an abstract noun like \textit{``war''} or \textit{``peace''} is a 100\% mathematically correct and valid syntactic completion.

The critical distinction lies in \textit{interpretability}. While the standard Transformer blindly outputs a high logit for ``peace'' with zero explanation as to why, \textit{NBSR provides a completely transparent, path-dependent reasoning trace}. The audit trail reveals that NBSR's Root router correctly deduced that a noun must follow the preposition, allocating 100\% of its routing probability (1.00) to the Semantic Expert and successfully zeroing out the Syntactic Expert (i.e. probability 0.00). Subsequently, the leaf experts extracted massive diagnostic evidence specifically for Concrete (320.4) and Abstract (275.5) nouns, while explicitly extracting 0.0 evidence for grammatically invalid function words or modifiers. 

Even though it ultimately outputs a token that sounds semantically anomalous to a human, the routing trace mathematically proves that NBSR successfully learned and executed the underlying grammatical rules. \textit{It did not hallucinate an invalid part-of-speech} (like a verb); it predictably narrowed the universe down to a noun, \textit{providing a step-by-step rationale that standard LLMs completely lack}. By utilizing this synthetic grammar, we perfectly isolated the interpretable routing mechanism, proving that NBSR physically routes tokens through a logical DAG (Part-of-Speech $\to$ Noun Type). 

\paragraph{Uncertainty-Aware Generation \textit{vs.} Top-$k$ Selection.}
Traditional top-$k$ sampling relies on softmax outputs, which are notoriously poorly calibrated \cite{guo2017calibration, desai2020calibration}. Standard neural networks are essentially blind to the boundaries of their own knowledge. Because the softmax function mathematically forces all output probabilities to sum to 100\%, the model is stripped of the ability to output "I don't know". When presented with an out-of-distribution (OOD) prompt, i.e. a scenario or token that fundamentally differs from the statistical universe the model was trained on, the network pushes the unfamiliar data through its layers and artificially inflates the highest random logit into a confident prediction \cite{hendrycks2017baseline, ovadia2019trust}. This structural flaw is the mathematical root cause of AI hallucinations: a model confidently guessing on an unfamiliar input.

To evaluate this vulnerability, we feed the models an explicit OOD prompt containing an unknown token (\texttt{<unk>}):
\begin{verbatim}
Ambiguous OOD Prompt: "<bos> the <unk> ..."

1. Standard Transformer -> Predicts: 'or' with 6.7% confidence.
2. Transformer MoE      -> Predicts: 'to' with 6.6% confidence.
3. NBSR Language Model  -> Predicts: 'but' with 5.7% expected probability.
\end{verbatim}

While our toy dataset's extremely limited vocabulary (65 tokens) naturally prevents the extreme 99\% overconfidence spikes seen in massive LLMs, which causes all three models to output relatively flat distributions ($\sim$6\%) for the unknown context, NBSR provides the \textit{architectural} foundation to solve this scaling problem natively: it replaces the forced softmax with a Dirichlet distribution over the vocabulary space $\bm{\alpha} \in \mathbb{R}^{|\mathcal{V}|}$, which accumulates raw evidence rather than strictly forcing probabilities to sum to 1. When NBSR encounters an OOD prompt (like \texttt{<unk>}), it searches its specialized routing experts and finds zero matching evidence. Because it extracts no evidence, the total Dirichlet precision $\alpha_0 = \sum \alpha_i$ stays perfectly flat at the uniform prior. 

As formalized in \textit{Corollary~\ref{corollary:1} (Epistemic Collapse)}, this failure to accumulate evidence physically traps the system in a state of maximum epistemic uncertainty ($u \approx 1$). Therefore, \textit{NBSR provides a native, mathematically rigorous measure of epistemic uncertainty}. This enables \textit{Evidential Thresholding} as an \textit{alternative} to top-$k$. Rather than blindly sampling from the top 50 tokens based on uncalibrated logits, the decoding algorithm can directly evaluate the Dirichlet precision $\alpha_0$. When scaled to massive language corpora, if $\alpha_0$ falls below a critical confidence threshold on an OOD prompt, the system mathematically recognizes that it lacks evidence. It can then gracefully halt, abstain, or ask the user for clarifying context, which \textit{prevents hallucinations at the architectural level rather than relying on post-hoc output filtering}.

\paragraph{Efficacy of Early Exiting on Syntactic Tokens.}
Human language adheres to Zipf's law\footnote{Zipf's law is an empirical principle stating that given a large corpus of data, the frequency of an item is inversely proportional to its rank ($f(r) \propto \frac{1}{r}$). Consequently, a small handful of highly ranked items (e.g. function words) dominate the vast majority of occurrences \cite{zipf1949human}.}, meaning a vast majority of generated tokens are simple, highly predictable function words (e.g. \textit{the, a, is, to}). Allocating the same amount of computational FLOPs to predict the word \textit{``the''} as to predict a highly complex domain-specific noun is highly inefficient. \textit{NBSR's dynamic early exiting acts as a resource-rational adaptive compute mechanism}. For highly predictable syntactic tokens, the Mid-Level \textit{Syntactic Expert} extracts massive initial evidence, driving the Dirichlet entropy below the threshold $\eta$ and triggering an early exit at Depth 1. Deep traversal (Depth 2) is automatically reserved only for complex, low-frequency semantic tokens requiring deeper contextual disambiguation. This dynamic depth allocation allows the fast configuration to accelerate token generation while maintaining competitive baseline perplexity.

\subsection{NBSR-Mem: NBSR with Dynamic Memory for Control and Planning}
\label{subsec:control}

Standard deep reinforcement learning and imitation learning policies are notoriously brittle; they operate as opaque black boxes and suffer from catastrophic overconfidence when deployed in out-of-distribution (OOD) environments. Further, real-world autonomous agents rarely possess perfect global information. When an agent's observation space is restricted to a local window, the navigation task becomes a \textit{Partially Observable Markov Decision Process} (POMDP) \cite{puterman1994mdp}. A purely \textit{reactive} policy is vulnerable to \textit{spatial amnesia} - once a crucial environmental cue (e.g. a wall or obstacle) leaves the local visual frame (i.e. the 'sensory' or 'receptive' field), the agent forgets its existence, leading to failure or infinite navigational loops. To resolve both the structural opacity of standard neural policies and the limitations of partial observability, we transition from \textit{static} prediction to a dynamic control task and propose \textbf{NBSR-Mem}, which integrates a dynamic recurrent memory buffer into the evidential routing architecture.

\subsubsection{Task and Experimental Setup.}

The autonomous agent is tasked with navigating a 2D grid environment to reach a randomized goal while avoiding walls and static obstacles. The environment is strictly partially observable: at any given timestep $t$, the agent receives only a $5 \times 5$ local spatial matrix centered on its current position (Fig.\ref{fig:pomdp_env}). The action space consists of 4 discrete movements: 0 (Up/Cruising), 1 (Down/Reverse), 2 (Left/Evasion), and 3 (Right/Evasion). A purely reactive policy in this setting is vulnerable to spatial amnesia, i.e. once a crucial obstacle or landmark leaves the local visual frame, a reactive agent forgets its existence, which can lead to catastrophic failures or infinite navigational loops.

To isolate and test the memory mechanism under partial observability, we utilize a procedurally generated sequence of abstracted keyframes representing the crucial phases of a corridor traversal. A single trajectory is structured as follows (as illustrated in Fig. \ref{fig:pomdp_env}): 
\begin{itemize}
    \item $t=0$ (Memory Cue): the agent receives an initial visual cue (a sign on the left or right wall) that dictates whether it must turn left or right at a distant intersection.
    \item $t=1, 2$ (Spatial Amnesia Zone): the agent receives a featureless spatial matrix, representing transit through an empty corridor. A purely reactive agent loses all historical context during this phase.
    \item $t=3$ (Intersection): the agent arrives abruptly at an intersection with a wall directly ahead. To successfully navigate, it must recall the cue from $t=0$ and turn in the correct direction.
\end{itemize}

\begin{figure}[htbp]
\centering
\begin{tikzpicture}[
cell/.style={draw, minimum size=0.6cm, inner sep=0pt},
agent/.style={fill=blue!20, font=\sffamily\scriptsize\bfseries},
wall/.style={fill=gray!60, font=\sffamily\scriptsize\bfseries},
sign/.style={fill=green!40, font=\sffamily\scriptsize\bfseries},
font=\sffamily\scriptsize
]
\def\drawgrid#1#2{
\begin{scope}[shift={(#1)}]
\draw[step=0.6cm, color=gray!50] (0,0) grid (3,3);
\node[cell, agent] at (1.5, 1.5) {A}; 
#2
\end{scope}
}
\drawgrid{0,0}{
\node[cell, sign] at (0.3, 1.5) {S};
\node[anchor=south] at (1.5, 3.2) {$t=0$: Memory Cue};
\node[anchor=north] at (1.5, -0.2) {(Agent sees ``Left'' sign)};
}
\draw[->, thick, color=gray] (3.2, 1.5) -- (3.8, 1.5);
\drawgrid{4.0,0}{
\node[anchor=south] at (1.5, 3.2) {$t=1, 2$: Empty Corridor};
\node[anchor=north, align=center] at (1.5, -0.2) {(Spatial Amnesia Zone)};
}
\draw[->, thick, color=gray] (7.2, 1.5) -- (7.8, 1.5);
\drawgrid{8.0,0}{
\node[cell, wall] at (1.5, 2.1) {W};
\node[anchor=south] at (1.5, 3.2) {$t=3$: Intersection};
\node[anchor=north, align=center] at (1.5, -0.2) {(Must recall $t=0$ to turn)};
}
\end{tikzpicture}
\caption{Illustration of the synthetic POMDP keyframe sequence. Rather than a contiguous simulation, the dataset samples crucial keyframes from a corridor traversal. At $t=0$, an initial cue (S) dictates the correct future turning direction. During $t=1, 2$, the agent receives featureless floor matrices, representing a state of visual ambiguity where reactive policies suffer from ``spatial amnesia.'' At the final keyframe $t=3$, the agent arrives abruptly at an intersection (W) and must rely exclusively on its recurrent memory buffer to execute the evasion maneuver.}
\label{fig:pomdp_env}
\end{figure}
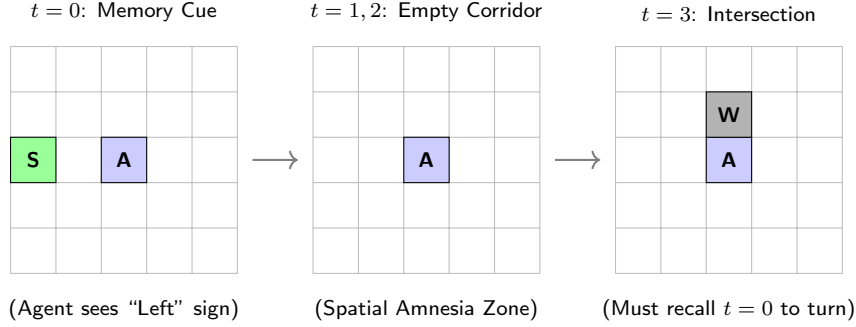

\paragraph{Training Paradigm: Behavioral Cloning.}
We employ Behavioral Cloning (BC) \cite{pomerleau1989alvinn}, a foundational method of Imitation Learning \cite{hussein2017imitation}, to train the agent, as illustrated in Fig. \ref{fig:imitation_learning}. BC reframes \textit{dynamic control} as a \textit{supervised classification} task, where the goal is to learn a policy $\pi_\theta$ that mimics the actions of an algorithmic expert (which possesses perfect global information). The algorithmic expert generates a dataset of optimal state-action pairs $(o_t, a^*_t)$. The neural actor then learns to map state observations directly to the expert's discrete actions. During training, the agent policy, $\pi_\theta$, samples a batch of observations $o_t$ and outputs a predicted action $\hat{a}_t$. The network parameters $\theta$ are updated via \textit{Backpropagation Through Time} (BPTT) to minimize the Negative Log-Likelihood (NLL) between the predicted action distribution and the expert's true action $a^*_t$. 

\begin{figure}[htbp]
\centering
\begin{tikzpicture}[
    node distance=1.7cm and 2.0cm,
    box/.style={
        draw,
        thick,
        rounded corners,
        minimum width=3.0cm,
        minimum height=1.0cm,
        align=center,
        fill=blue!5,
        font=\small
    },
    arrow/.style={->, thick, >=stealth},
    label/.style={font=\scriptsize, fill=white, inner sep=1pt}
]

\node[box, fill=gray!10] (env) 
    {POMDP Environment};

\node[box, fill=green!10, above=1.6cm of env] (expert) 
    {Expert Oracle\\(Perfect Information)};

\node[box, fill=orange!10, right=1.2cm of env, minimum width=3.8cm] (data) 
    {Dataset $\mathcal{D}$\\$\{(o_1, a^*_1), \dots, (o_N, a^*_N)\}$};

\node[box, fill=blue!10, right=2.0cm of data] (agent) 
    {Agent Policy $\pi_\theta(a|o_t)$\\($\hat{a}_t$)};

\node[box, fill=red!10, above=1.6cm of agent, minimum height=0.8cm] (loss) 
    {Supervised Loss\\$\mathcal{L}_{NLL}(\hat{a}_t, a^*_t)$};

\draw[arrow] 
    (env.north) -- node[label, left] {True State $s_t$} (expert.south);

\draw[arrow] 
    (expert.east) -| node[label, near start, above] {Optimal Action $a^*_t$} (data.north);

\draw[arrow] 
    (env.east) -- node[label, below] {Obs $o_t$} (data.west);

\draw[arrow] 
    (data.east) -- node[label, below] {Obs Batch $o_t$} (agent.west);

\draw[arrow] 
    (agent.north) -- node[label, right] {Predicted $\hat{a}_t$} (loss.south);

\draw[arrow] 
    (data.north east) |- node[label, near start, above] {Target $a^*_t$} (loss.west);

\draw[arrow, dashed] 
    (loss.east) -- ++(0.6,0)
    |- node[label, near start, right] {Gradient Update $\nabla_\theta$} (agent.east);

\end{tikzpicture}
\caption{The Imitation Learning (Behavioral Cloning) paradigm under partial observability. An expert oracle utilizes the true state ($s_t$) to generate optimal actions ($a^*_t$), which are paired with the agent's limited observations ($o_t$) to form the dataset. The agent policy $\pi_\theta$ is trained offline using supervised learning to mimic the expert, isolating architectural efficacy from Reinforcement Learning instabilities.}
\label{fig:imitation_learning}
\end{figure}
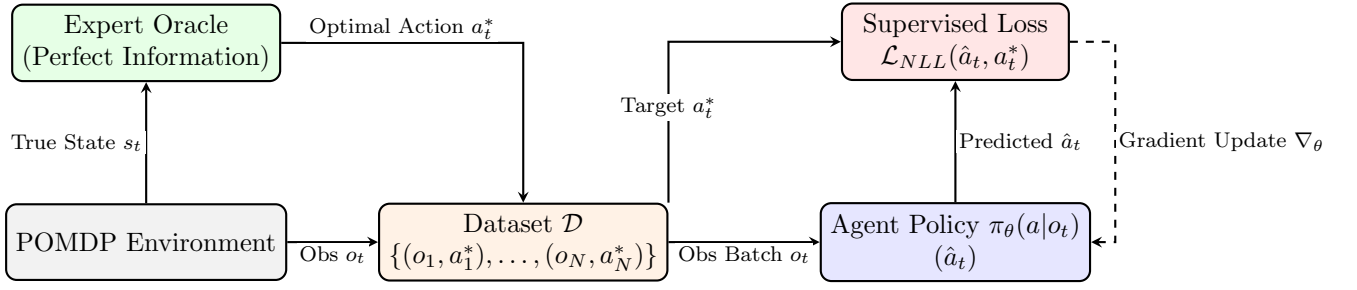

We deliberately selected this training paradigm over standard Reinforcement Learning (RL) for two critical reasons: 

\begin{itemize} [label=-]
    \item \textit{Architecture Isolation}: by removing the high variance, reward-shaping dependencies, and hyperparameter sensitivity inherent to RL, any improvements in performance or safety can be strictly attributed to the NBSR topology.
    \item \textit{Covariate Shift Testbed}: standard BC policies are notoriously susceptible to covariate shift - they mimic the expert blindly and fail catastrophically when encountering novel OOD states. This makes BC an ideal testbed for demonstrating NBSR's native evidential abstention mechanisms.
\end{itemize}

\paragraph{NBSR-Mem and Baselines.}
To construct \textbf{NBSR-Mem}, we couple the \textit{Global Knowledge Oracle} (a lightweight CNN) with a \textit{Gated Recurrent Unit} (GRU). Importantly, the memory update occurs \textit{prior} to expert routing. The CNN features update a persistent hidden state $\mathbf{m}_t$ encoding the historical trajectory. This \textit{temporally-aware state} $\mathbf{m}_t$ is then passed onto the hierarchical evidential DAG (illustrated below), where expert networks are intentionally flattened to single linear layers to ensure unhindered gradient flow.

\begin{verbatim}
[Level 0: Root, Width: 1]                    Navigational Router
                                             /                 \
                                            /                   \
[Level 1: Mid, Width: 2]            Cruising Expert          Evasion Expert
                                    /         \              /          \
                                   /           \            /            \
[Level 2: Leaf, Width: 4]        Up           Down        Left          Right
                             (Straight)     (Reverse)    (Avoid)      (Avoid)
\end{verbatim}

To validate the necessity of both temporal memory and evidential routing, we benchmark NBSR-Mem against 3 distinct architectural configurations:
\begin{enumerate}
    \item \textit{Standard CNN (Reactive):} a standard feed-forward convolutional network utilizing a Softmax output. Lacking memory, it is expected to succumb to spatial amnesia; lacking evidential routing, it acts as an opaque black-box prone to OOD overconfidence.
    \item \textit{CNN-GRU (Memory):} a temporal baseline where the CNN features update a GRU hidden state before passing to a standard linear classifier. While capable of solving the POMDP task via memory, it remains structurally opaque and highly vulnerable to OOD hazards.
    \item \textit{NBSR (Reactive):} the hierarchical evidential routing architecture \textit{without} the GRU memory buffer. It is expected to fail the primary navigation task due to spatial amnesia, but successfully trigger epistemic safety halts when encountering OOD traps.
\end{enumerate}

\paragraph{Training Dynamics: Stabilizing the Recurrent Routing}
Training this \textit{Recurrent Evidential Routing Network} is historically challenging \cite{shazeer2017outrageously, pascanu2013difficulty}. Naive implementations typically suffer from \textit{expert collapse} \cite{shazeer2017outrageously}, vanishing gradients \cite{Hochreiter1997LSTM,pascanu2013difficulty}, and temporal memory amnesia (catastrophic forgetting) \cite{bengio1994learning}. To stabilize NBSR-Mem, we implement 4 critical structural interventions:

\begin{enumerate}
    \item \textit{Deterministic Routing for BPTT:} we replace the stochastic Gumbel-Softmax with a deterministic, standard Softmax exclusively during training. This creates a mathematically smooth pathway, allowing \textit{Backpropagation Through Time} (BPTT) to flow seamlessly from the routing leaves, through the hierarchy, and deep into the GRU to link past states to future decisions.
    \item \textit{Semantic Concept Routing (Auxiliary Loss):} without explicit guidance, hierarchical routers, particularly when utilizing a deterministic Softmax, seek the laziest optimization, routing all states to a single, monolithic ``master expert'' (\textit{Expert Collapse}) \footnote{While deterministic Softmax exacerbates mode collapse due to its ``winner-take-all'' gradient scaling, expert collapse is fundamentally an optimization pathology where it is mathematically cheaper to update an already-competent expert than to train a novel one. This ``rich get richer'' dynamic occurs across routing mechanisms (including Sigmoid, Gumbel-Softmax, and Top-K) unless explicit load-balancing or semantic guidance is applied.}. To prevent this, we apply a lightweight \textit{Auxiliary Concept Loss} ($\mathcal{L}_{route}$) that physically forces the Root Router to map to human-interpretable sub-spaces: straight corridors \textit{must} be routed to the ``Cruising'' expert, while complex intersections \textit{must} trigger ``Evasion.''
    \item \textit{Masked Evidential Regularization:} standard evidential regularizers penalize all generated evidence, creating a "Gradient Tug-of-War" where the NLL loss encourages correct evidence, but the regularizer punishes it, resulting in memory failure. We utilize a target-masked regularizer \cite{sensoy2018evidential} that only penalizes hallucinated evidence on \textit{incorrect} classes. This forces resource rationality, squashing incorrect actions strictly to zero, while leaving the correct class free to soar to massive confidence. This shows that the original entropy regularization term $\lambda \sum_{v_t \in \mathcal{T}} \mathcal{H}(\bm{\alpha}_t)$ in Eq.\ref{eq:training_obj} can sometimes introduce optimization deadlock and requires masking for calibration.
    \item \textit{Targeted Weight Decay and LayerNorm:} to anchor the latent space and bound out-of-distribution explosions, we apply LayerNorm and targeted L2 weight decay to the CNN backbone. This ensures untrained visual channels (e.g. alien hazards - see later) mathematically decay to zero, preventing random noise from being amplified by the network.
\end{enumerate}

\subsubsection{Results and Analysis}

\paragraph{Overall Navigation Performance.}
We benchmark NBSR and NBSR-Mem against standard continuous CNN and CNN-GRU (Memory) baselines across 1000 held-out maze configurations. The primary metric is the episodic Success Rate (reaching the goal without crashing). 

As shown in Table~\ref{tab:results_control}, purely reactive policies (\textit{Standard CNN} and \textit{NBSR Reactive}) fail frequently ($\sim$65\% success) because they suffer from \textit{spatial amnesia} at complex intersections ($t=3$ in Fig.\ref{fig:pomdp_env}). However, NBSR-Mem perfectly matches the standard CNN-GRU baseline, achieving a flawless 100.0\% success rate. This proves that our hierarchical DAG structure is perfectly communicating with the GRU memory buffer, and that the evidential topology incurs \textit{zero performance penalty}. NBSR-Mem proves imposing interpretability doesn't ruins accuracy.

\begin{table}[htbp]
\centering
\caption{Performance on the Partially Observable 2D Navigation Task.}
\label{tab:results_control}
\resizebox{\columnwidth}{!}{%
\begin{tabular}{lcccc}
\hline
\textbf{Policy Architecture} & \textbf{Success Rate ($\uparrow$)} & \textbf{Avg. Depth} & \textbf{OOD Abstain Rate ($\uparrow$)} & \textbf{Interpretability} \\ \hline
Standard CNN (Reactive) & 64.7\% & N/A & 0.0\% (Crashes) & Black-box \\
CNN-GRU (Memory) & \textbf{100.0\%} & N/A & 0.0\% (Crashes) & Black-box \\ \hline
NBSR (Reactive) & 65.3\% & 1.18 & \textbf{100.0\%} (Halts) & Full Trace \\
\textbf{NBSR-Mem (Ours)} & \textbf{100.0\%} & \textbf{1.00} & \textbf{100.0\%} (Halts) & \textbf{Full Trace} \\ \hline
\end{tabular}%
}
\vspace{0.15cm}
\parbox{\columnwidth}{\footnotesize \textit{Note:} Success Rate measures the percentage of episodes where the agent safely reached the goal. Abstention Rate measures the policy's ability to safely halt when presented with an unseen OOD trap.}
\end{table}

\paragraph{Interpretable Policy Routing \textit{vs.} Black-Box Control.}
When the standard CNN-GRU policy navigates an intersection, it outputs an action with \textit{no mechanistic explanation}. In contrast, NBSR-Mem physicalizes the decision-making process. By examining the audit trace of a successful evasive maneuver, we observe the model explicitly routing through semantic sub-spaces:
\begin{verbatim}
Agent State: Approaching a wall directly ahead 
             (Remembering Left Sign from t=0).
-> Standard CNN-GRU Policy: Predicts Action [Turn Left] (Opaque mechanism)

-> NBSR-Mem Policy Trace:
   [Phase 1] Root Router: Cruising (0.00) | Evasion (1.00)
   [Phase 2] Depth 1 Evidence (Evasion Expert):
      -> Up: 0.0 | Down: 0.0 | Left: 28.3 | Right: 0.0
      -> (Entropy remains above threshold, routing deeper...)
   [Phase 3] Depth 2 Additive Evidence (Leaf Experts):
      -> Up: 0.0 | Down: 0.0 | Left: 75.0 | Right: 0.0
\end{verbatim}
Thanks to the \textit{Auxiliary Concept Loss}, Phase 1 perfectly isolates the abstract mode (Evasion: 1.00). Further, because the model is resource-rational, the evidence is sharp and sparse - it outputs exactly $0.0$ for incorrect actions, proving that the evidential regularizer successfully stopped all hallucinations.

\paragraph{Resource-Rational Adaptiveness and Computational Efficiency.}
Notably, the integration of temporal memory profoundly impacts dynamic routing depth. The reactive NBSR policy, lacking historical context at intersections, exhibits high epistemic uncertainty. Because a perfectly uncertain 4-class Dirichlet distribution ($\boldsymbol{\alpha} = [1,1,1,1]$) yields a maximum differential entropy of roughly $-1.79$, the reactive policy safely exceeds the $-4.5$ confidence threshold we set, forcing the network to route to Depth 2 (Avg. Depth: 1.18)\footnote{The average depth is calculated across the entire test trajectory dataset. Because the reactive policy lacks memory, it successfully early-exits on the trivial empty corridors (comprising the majority of states), but is forced to traverse to Depth 2 upon reaching every intersection due to high epistemic uncertainty. This blend of Depth 1 and Depth 2 routing yields an overall episodic average of 1.18.}.

In contrast, NBSR-Mem utilizes its GRU hidden state to achieve extreme confidence. As seen in the trace above, the Depth 1 Evasion Expert alone extracts a massive $28.3$ evidence for the correct turn based on memory. This profound reduction in uncertainty causes the entropy to plummet past the threshold, allowing the router to safely early-exit. Remarkably, the model achieves sufficient epistemic confidence to early-exit on all states, yielding an average inference depth of exactly 1.00 while maintaining 100\% accuracy. Thus, our architecture demonstrates that resolving partial observability not only recovers task performance but actively reduces the computational footprint of the hierarchical policy.

\paragraph{Epistemic Safety in Out-of-Distribution Environments.}
The most critical failure mode of standard autonomous policies is their behavior in novel environments. To test this, we inject an ``Alien Hazard'' (a visual channel never present in training) directly into the agent's path.

Because standard baselines utilize a softmax activation, the mathematical constraint ($\sum p = 1$) forces the policy to confidently hallucinate. In 100\% of our OOD trials, the standard CNN-GRU policy selected a movement action and crashed into the hazard (Table~\ref{tab:results_control}). 

NBSR-Mem natively solves this through Dirichlet thresholding, empirically validating \textit{Corollary~\ref{corollary:1} (Epistemic Collapse)}. Due to our \textit{LayerNorm} and \textit{targeted weight decay}, the untrained Alien channel weights were mathematically anchored to zero. When the Alien Trap appeared, the network extracted only $8.39$ total precision - safely below the operating threshold of $10.0$ we set. Because the evidence remained stagnant, the epistemic uncertainty $u$ mathematically refused to collapse. Instead of hallucinating, NBSR-Mem recognized its own ignorance and triggered a safe fallback protocol (100.0\% Halt). Without these target-masked regularizations, ID and OOD evidence distributions dangerously overlap; with them, NBSR-Mem perfectly preserves the epistemic safety margin. Ultimately, this demonstrates that NBSR-Mem ``knows when it doesn't know'', providing a robust, uncertainty-aware framework for safe autonomous control.

\subsection{NBSR as Active Learning in Bayesian Optimal Experimental Design}
\label{subsec:nbsr_boed}

In many scientific and clinical domains, acquiring data is inherently constrained by stringent budgets. For example, in pharmacological dose-response modeling, a practitioner must actively decide which drug concentrations to measure in order to maximize understanding of the efficacy curve (e.g. identifying the $EC_{50}$\footnote{The $EC_{50}$ (half maximal effective concentration) represents the concentration of a drug or intervention that induces a response halfway between the baseline and the maximum possible effect, serving as a standard measure of potency \cite{huang2025GMA}.}) while minimizing laboratory costs. This challenge is formalized by \textit{Bayesian Optimal Experimental Design} (BOED) \cite{sebastiani2000BOED}, which seeks to select an experimental design $\xi$ that maximizes the \textit{Expected Information Gain} (EIG). For an unobserved\footnote{This target variable can be a particular parameter of interest, such as the latent $EC_{50}$ threshold in pharmacodynamics, or a broader categorical outcome, such as the true underlying pathology in a clinical diagnostic setting. Note that to maintain notational consistency with the NBSR framework, we use $y$ to denote the unobserved target variable (commonly denoted as $\theta$ in standard BOED literature e.g. in \cite{foster2019BOED}) and $e$ to denote the prospective experimental outcome or evidence (commonly denoted as $y$).} target variable $y$ and a prospective measurement $e$ acquired under design $\xi$, the EIG is mathematically defined as the mutual information between $y$ and $e$ \cite{foster2019BOED}:
\begin{equation} \label{eq:BOED_EIG}
    \text{EIG}(\xi) = I(y ; e \mid \xi) = \mathbb{E}_{p(e \mid \xi)} \big[ \mathcal{H}(p(y)) - \mathcal{H}(p(y \mid e, \xi)) \big]
\end{equation}
where $\mathcal{H}$ denotes the entropy of the belief state. In other words, BOED identifies the specific experimental design that, in expectation over all possible measurement outcomes, induces the maximum reduction in posterior entropy over the target latent variables \cite{sebastiani2000BOED,foster2019BOED}. Provided the underlying predictive model is accurate, this approach constitutes an optimal, albeit myopic (one-step), data acquisition strategy from a strict information-theoretic perspective.

The NBSR framework natively operates as a sequential BOED engine. At any inference step $t$, the routing action $a_t$ formally corresponds to selecting the next experimental design $\xi$ (e.g. querying a specific sub-specialist or running a distinct medical test), and the extracted evidence $\mathbf{e}_t$ serves as the observation. Historically, maximizing EIG in deep neural networks is computationally intractable because the posterior entropy requires evaluating an intractable marginal likelihood, often necessitating complex variational lower bounds \cite{foster2019BOED}. However, because the NBSR Bayesian state update is defined via exact conjugate addition ($\bm{\alpha}_{t+1} = \bm{\alpha}_t + \mathbf{e}_t$), our framework bypasses this intractability entirely. The exact, realized information gain achieved by querying expert $v_{t+1}$ is analytically computed as the drop in differential entropy: $\Delta \mathcal{H}_t = \mathcal{H}(\bm{\alpha}_t) - \mathcal{H}(\bm{\alpha}_{t+1})$. 

As established in Proposition~\ref{prop:1}, the structural penalty $\lambda \sum \mathcal{H}(\bm{\alpha}_t)$ in our training objective intrinsically forces the router to maximize this sequential EIG by penalizing delayed uncertainty reduction. By introducing explicit, asymmetric measurement costs into the graph, we can extend NBSR into a fully resource-rational \textit{Active Learning} framework.

\subsubsection{Task and Experimental Setup: Active Clinical Triage}

To empirically validate NBSR as an active BOED agent, we extend the Structured Medical Diagnosis task (Section \ref{subsec:medical_diagnosis}) into an \textit{Active Clinical Triage} environment. In standard machine learning, the model is given free, simultaneous access to all input features. In this experiment, the patient's true diagnostic state is initially hidden. The agent only has access to free, baseline demographic data at $t=0$. 

The 132 symptoms are clustered into 5 distinct ``Diagnostic Test Panels'' (e.g. Basic Metabolic Panel, Neurological Exam, Dermatological Swab). Each panel is governed by a dedicated NBSR Expert node and has an associated financial or temporal cost $c(v)$, detailed in Table \ref{tab:panel_costs}.

\begin{table}[htbp]
    \centering
    \caption{Diagnostic Test Panels and Associated Costs ($c(v)$). These values represent the asymmetric financial or temporal constraints required to execute each specific diagnostic panel.}
    \label{tab:panel_costs}
    \begin{tabular}{lc}
    \toprule
    \textbf{Diagnostic Panel} & \textbf{Assigned Cost $c(v)$} \\
    \midrule
    Basic Metabolic Panel (Blood) & 10.0 \\
    Hepatic Panel                 & 15.0 \\
    Respiratory Swab              & 20.0 \\
    Neurological Exam             & 30.0 \\
    MRI Scan                      & 50.0 \\
    \bottomrule
    \end{tabular}
\end{table}

\paragraph{The Active BOED Objective.} 
The agent must sequentially select which tests to run to confidently diagnose the patient, balancing diagnostic accuracy against the cumulative cost of the required tests. We modify the primary NBSR training objective (Eq.~\ref{eq:training_obj}) to incorporate a cost-aware regularizer:
\begin{equation} \label{eq:boed_obj}
    \mathcal{L}_{\text{BOED}}(x, y^*) = - \log \mathbb{E}_{\mathbf{p} \sim \text{Dir}(\bm{\alpha}_T)}[p_{y^*}] + \lambda \sum_{v_t \in \mathcal{T}} \mathcal{H}(\bm{\alpha}_t) + \gamma \sum_{v_t \in \mathcal{T}} c(v_t)
\end{equation}
where the expectation is taken over the terminal Dirichlet belief state $\mathbf{p} \sim \text{Dir}(\bm{\alpha}_T)$, and $\gamma$ is a tunable hyperparameter dictating the budgetary strictness. Notably, as derived previously in Eq.~\ref{eq:expected_marginal_cc}, the expected marginal probability under the Dirichlet distribution expands exactly to:
\begin{equation*}
    \mathbb{E}_{\mathbf{p} \sim \text{Dir}(\bm{\alpha}_T)}[p_{y^*}] = \frac{\alpha_{T, y^*}}{\sum_{k=1}^K \alpha_{T, k}}
\end{equation*}
which establishes that the first two terms of Eq.~\ref{eq:boed_obj} are mathematically identical to the primary NBSR training objective (Eq.~\ref{eq:training_obj}). During the forward pass, the Router processes the patient's current Dirichlet belief state $\bm{\alpha}_t$ to select the test panel. The policy is implicitly optimized during training to select tests that maximize information gain per unit cost, as governed by the cost-aware regularizer $\gamma$ and the entropy penalty $\lambda$ in Eq.~\ref{eq:boed_obj}. The process halts dynamically when the belief entropy $\mathcal{H}(\bm{\alpha}_t)$ falls below the confidence threshold $\eta$.

\paragraph{Connection to the Standard EIG Objective.}
To understand why optimizing Eq.~\ref{eq:boed_obj} effectively solves the active learning problem, we compare it directly to the standard BOED objective. Recall that standard BOED seeks a design $\xi$ that maximizes the Expected Information Gain (Eq.\ref{eq:BOED_EIG}):
\begin{equation*}  
    \text{EIG}(\xi) = I(y ; e \mid \xi) = \mathbb{E}_{p(e \mid \xi)} \big[ \mathcal{H}(p(y)) - \mathcal{H}(p(y \mid e, \xi)) \big]
\end{equation*}
In the NBSR framework, at step $t$, the ``prior'' state is given by the Dirichlet distribution parameterized by $\bm{\alpha}_t$. The routing decision $a_t$ serves as the experimental design $\xi$, and the extracted evidence $\mathbf{e}_t$ serves as the observation $e$. The exact, realized information gain of taking step $t$ is the reduction in entropy: $\Delta \mathcal{H}_t = \mathcal{H}(\bm{\alpha}_t) - \mathcal{H}(\bm{\alpha}_{t+1})$. 

Minimizing the cumulative sum of entropies $\lambda \sum_{t=0}^T \mathcal{H}(\bm{\alpha}_t)$ in Eq.~\ref{eq:boed_obj} forces the network to minimize the area under the entropy curve. Algebraically, this heavily penalizes delayed uncertainty reduction, forcing the routing policy $\pi_\theta$ to consistently select paths that maximize the expected stepwise drop $\mathbb{E}[\Delta \mathcal{H}_t]$ at every juncture. Therefore, our regularized loss $\mathcal{L}_{\text{BOED}}$ acts as a computationally tractable, end-to-end surrogate for maximizing the exact sequential EIG (Eq.~\ref{eq:BOED_EIG}). This circumvents the intractable marginal likelihood integrations typically required in standard variational BOED \cite{foster2019BOED}.

\paragraph{Autoregressive Routing (AR-NBSR) to Prevent Combinatorial Explosion.}
Deploying a static, feed-forward DAG topology into this active triage environment would induce a massive combinatorial explosion. Specifically, given $N$ available diagnostic test panels (the branching factor) and a maximum required testing sequence of $L$ steps (the tree depth), a hardcoded decision tree would require an intractable $\mathcal{O}(N^L)$ leaf-node complexity. To resolve this, we configure the framework as an \textit{Autoregressive State Machine (ASM)}.

We formally define the AR-NBSR as an ASM - a sequential generative system where the current state $\mathcal{S}_t$ and the subsequent transition are conditioned on the history of prior observations. We define the ASM by the tuple $(\mathcal{S}, \mathcal{A}, \pi_\theta, \mathcal{U})$, where $\mathcal{S}_t = (\mathbf{h}_x, \bm{\alpha}_t, \mathbf{m}_t)$ is the joint embedding of oracle features, current Dirichlet belief state, and action history; $\mathcal{A}$ is the action space of diagnostic panels; $\pi_\theta(a_t \mid \mathcal{S}_t)$ is the routing policy; and $\mathcal{U}$ is the state-update function defined by the Bayesian conjugate addition $\bm{\alpha}_{t+1} = \bm{\alpha}_t + \mathbf{e}_t$. While the state evolution inherently satisfies the Markov property (as the posterior belief $\bm{\alpha}_{t+1}$ is conditioned solely on the previous state $\bm{\alpha}_t$ and the instantaneous evidence $\mathbf{e}_t$), we term this process \textit{autoregressive} because the model uses its own past evidential accumulations to regress towards a terminal hypothesis. 

Conceptually, this maps perfectly to the cyclical architecture illustrated in Fig.~\ref{fig:agr_nbsr_triage}: instead of routing a patient down a physical hallway of isolated sub-specialists (a deep DAG), a single reusable \textit{Global Triage Router} and a flat array of independent Expert Networks operate in a recurrent, time-indexed loop. Consequently, static architectural \textit{depth} is replaced by recurrent \textit{time}. The diagnostic trajectory unfolds dynamically as an iterative loop, for example:

\begin{enumerate}
    \item \textbf{$t=0$:} The Global Router evaluates the initial baseline prior $\bm{\alpha}_0$ alongside the demographic features, and selects the first test (e.g. Panel 2). Panel 2 extracts evidence $\mathbf{e}_0$, updating the belief to $\bm{\alpha}_1 = \bm{\alpha}_0 + \mathbf{e}_0$.
    \item \textbf{$t=1$:} The exact same Global Router evaluates the newly updated state $\bm{\alpha}_1$. Recognizing lingering uncertainty regarding a specific pathology, it selects a complementary test (e.g. Panel 4). Panel 4 extracts evidence $\mathbf{e}_1$, updating the belief to $\bm{\alpha}_2 = \bm{\alpha}_1 + \mathbf{e}_1$.
    \item \textbf{$t=2$:} The router evaluates $\bm{\alpha}_2$. The newly accumulated evidence causes the differential entropy $\mathcal{H}(\bm{\alpha}_2)$ to plunge below the confidence threshold $\eta$. The recurrent loop terminates dynamically, yielding the final diagnosis without consuming further testing budget.
\end{enumerate}
This autoregressive formulation elegantly bounds the parameter count to a single router and $N$ experts, while granting the agent a dynamic, virtually infinite sequential action space.

\begin{figure}[htbp]
    \centering
    \resizebox{\columnwidth}{!}{%
    \begin{tikzpicture}[
        >=stealth,
        expert/.style={rectangle, draw=green!60!black, rounded corners, fill=green!5, thick, minimum width=2.8cm, minimum height=0.8cm, align=center},
        active_expert/.style={rectangle, draw=blue!60, rounded corners, fill=blue!10, thick, minimum width=2.8cm, minimum height=0.8cm, align=center},
        router/.style={ellipse, draw=orange!80, fill=orange!10, thick, align=center, inner sep=5pt},
        oracle/.style={rectangle, draw=gray!80, rounded corners, fill=gray!10, thick, minimum width=2.5cm, align=center, inner sep=6pt},
        sum/.style={circle, draw=black, thick, fill=yellow!20, inner sep=2pt, minimum size=0.6cm, align=center},
        delay/.style={rectangle, draw=purple!60, rounded corners, fill=purple!5, thick, minimum width=2.5cm, align=center, minimum height=0.8cm},
        check/.style={diamond, draw=red!60, fill=red!5, thick, align=center, aspect=2.2, inner sep=1pt},
        exit/.style={rectangle, draw=black, rounded corners, fill=gray!20, thick, align=center, minimum height=0.8cm},
        dot/.style={circle, fill=black, inner sep=1.5pt}
    ]

    \node[oracle] (oracle) at (0, -3.5) {\textbf{Global Oracle} \\ $\mathbf{h}_x$};
    \node[router] (router) at (0, 0) {\textbf{Global Router} \\ $\pi_\theta(a_t \mid \mathcal{S}_t)$};

    \node[expert, text=gray] (exp1) at (5, 1.5) {Panel 1};
    \node[active_expert] (exp2) at (5, 0) {\textbf{Panel $a_t$} \\ (Active Expert)};
    \node[text=gray] at (5, -0.75) {\vdots};
    \node[expert, text=gray] (exp3) at (5, -1.5) {Panel $N$};

    \node[sum] (sum) at (7.5, 3) {$+$};
    \node[delay] (delay) at (3.75, 4.5) {\textbf{Belief Memory} \\ $z^{-1}$};

    \node[check] (check) at (-3.5, 3) {\footnotesize $\mathcal{H}(\bm{\alpha}_t) < \eta$?};
    \node[exit] (final) at (-7, 3) {\textbf{Final} \\ \textbf{Diagnosis}};

    \coordinate (alpha_split) at (0, 3);
    \coordinate (oracle_split) at (2.4, -3.5);
    \node[dot] (switch_in) at (2.4, 0) {};
    \node[dot] (alpha_dot) at (0, 3) {};

    
    \draw[->, thick, dashed, draw=gray!80] (oracle) -- (router);
    \draw[-, thick, dashed, draw=gray!80] (oracle) -- (oracle_split);
    \draw[->, thick, dashed, draw=gray!80] (oracle_split) -- node[right, pos=0.8] {\footnotesize Data $\mathbf{h}_x$} (switch_in);

    \draw[->, thick, orange] (router) -- node[above] {\footnotesize Action $a_t$} (switch_in);

    \draw[->, thick, draw=gray!40] (switch_in) |- (exp1.west);
    \draw[->, thick, blue] (switch_in) -- (exp2.west);
    \draw[->, thick, draw=gray!40] (switch_in) |- (exp3.west);

    \draw[->, thick, blue] (exp2.east) -| node[pos=0.25, above] {Evidence $\mathbf{e}_t$} (sum.south);

    \draw[-, thick, purple] (delay.west) -| (alpha_split);

    \draw[->, thick, purple] (alpha_split) -- node[above] {Prior $\bm{\alpha}_t$} (sum.west);
    \draw[->, thick, purple] (alpha_split) -- (router.north);
    \draw[->, thick, purple] (alpha_split) -- (check.east);

    \draw[->, thick, purple] (sum.north) |- node[pos=0.25, right] {Posterior $\bm{\alpha}_{t+1}$} (delay.east);

    \draw[->, thick, red] (check.west) -- node[above] {\footnotesize Yes} (final.east);
    \draw[->, thick, red!80] (check.south) |- node[pos=0.75, above] {\footnotesize No (Continue Loop)} (router.west);

    \end{tikzpicture}%
    }
    \caption{Illustration of the Autoregressive NBSR (AR-NBSR) framework. Unlike a static feed-forward DAG, the Global Triage Router iteratively queries a flat array of available diagnostic panels. The belief state $\bm{\alpha}_t$ acts as a recurrent memory buffer, denoted by the Unit Delay Operator ($z^{-1}$) from discrete-time signal processing, which caches the newly updated posterior to serve as the prior for the subsequent routing step. By sequentially accumulating evidence $\mathbf{e}_t$ via conjugate addition until the diagnostic entropy falls below the safety threshold $\eta$, this temporal feedback loop naturally averts the $\mathcal{O}(N^L)$ combinatorial explosion.}
    \label{fig:agr_nbsr_triage}
\end{figure}
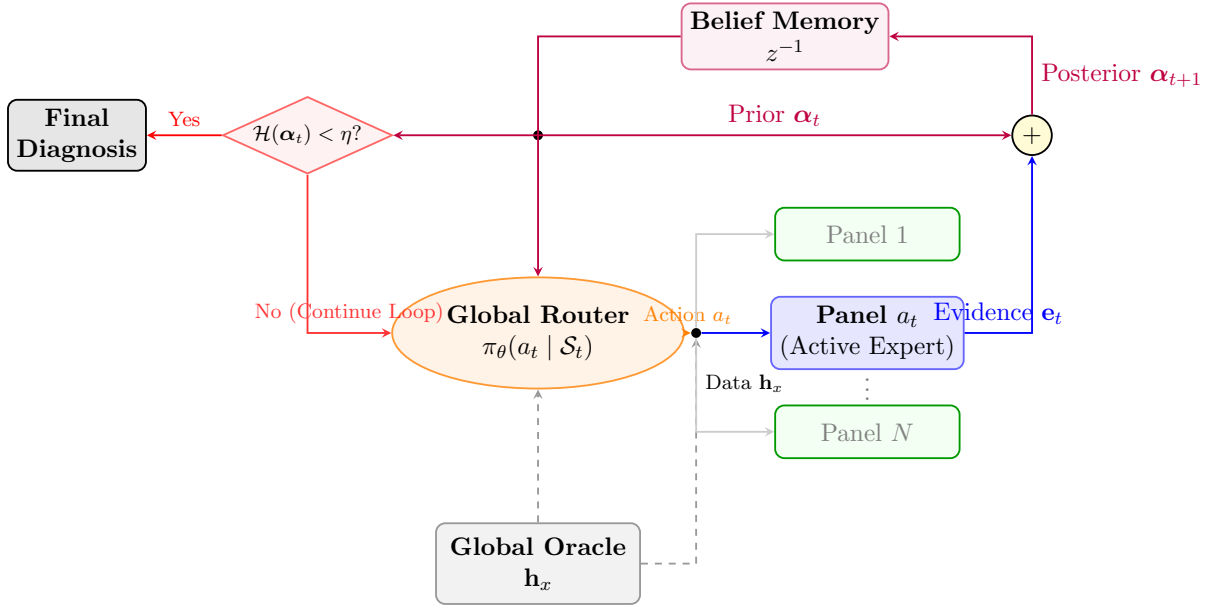

\begin{tcolorbox}[colback=gray!10, colframe=gray!50, arc=4pt, title=\textbf{Clinical Metaphor: The Active Triage Room}]
To intuitively map the mathematical terminologies of AR-NBSR to a real-world clinical setting, consider a hospital triage room:\\
\vspace{-0.2cm}
\begin{itemize}[leftmargin=*]
    \item \textbf{The Global Router (The Lead Diagnostician):} Rather than passing the patient down a physical hallway of isolated sub-specialists (a deep DAG), a single Lead Diagnostician sits at a central desk. Their only job is to evaluate the current information and decide \textit{which test to order next}.
    \item \textbf{The Expert Networks (The Test Panels):} Down the hall are 5 distinct testing rooms (e.g. Blood Panel, MRI, Respiratory Swab). These are the expert modules.
    \item \textbf{The Dirichlet Belief State $\bm{\alpha}_t$ (The Clipboard):} The Diagnostician holds a clipboard tracking the probability of 41 possible diseases. It starts blank (maximum entropy).
    \item \textbf{Autoregressive Time (The Loop):} The Diagnostician sends the patient to Room 2. The result (evidence $\mathbf{e}_0$) comes back, and the clipboard is updated ($\bm{\alpha}_1$). Realizing more information is needed to distinguish between two highly probable diseases, the Diagnostician sends the patient to Room 4 ($\mathbf{e}_1 \to \bm{\alpha}_2$).
    \item \textbf{The Confidence Threshold $\eta$ (The Stop Condition):} Once the clipboard indicates sufficient certainty for a specific disease, the Diagnostician halts the testing loop, sparing the patient the financial cost and time of visiting the remaining 3 rooms.
\end{itemize}
\end{tcolorbox}

\paragraph{Baselines.} We benchmark the NBSR active learning agent against two standard sequential acquisition strategies:
\begin{enumerate}
    \item \textbf{Random Allocation:} Iteratively queries random test panels until the entropy threshold $\eta$ is reached.
    \item \textbf{Greedy EIG (Myopic BOED):} A standard active learning baseline that evaluates the expected entropy drop for all available tests at step $t$, and selects the test that maximizes the immediate gain-to-cost ratio $\frac{\Delta \mathcal{H}_t}{c(v)}$. While optimal for a single isolated step, myopic BOED is well-known to lack long-horizon planning.
\end{enumerate}

\subsubsection{Results and Analysis}

\paragraph{Training Dynamics and Convergence.}
The training stability of the AR-NBSR model is evidenced by the convergence curves in Fig.~\ref{fig:training_curves}. We observe a precipitous decline in NLL during the initial 40 epochs, indicating that the expert networks successfully internalized the diagnostic features of the noisy clinical dataset. Concurrently, the \textit{Average Patient Cost} descends from the maximum budgetary ceiling ($\$125$) to a stable operational regime ($\sim \$48$). This confirms that the AR-NBSR agent reaches a stable, resource-rational equilibrium, balancing predictive accuracy against the cost-awareness regularizer without exhibiting signs of training instability or divergence.

\begin{figure}[htbp]
    \centering
    \includegraphics[width=0.5\columnwidth]{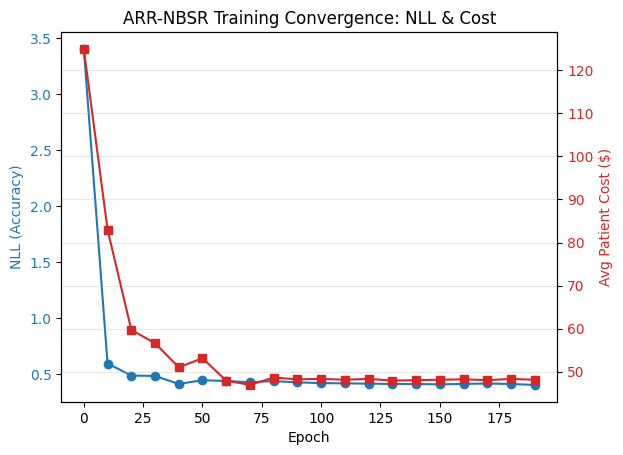}
    \caption{AR-NBSR Training Convergence. The simultaneous decline in NLL (Blue) and Average Patient Cost (Red) indicates the agent effectively learns to balance diagnostic precision with budgetary constraints.}
    \label{fig:training_curves}
\end{figure}

\paragraph{Cost-Accuracy Pareto Frontier.}
To demonstrate the model's resource-rationality, we evaluate the learned routing policy across a range of operational confidence thresholds $\eta$. By sweeping the early-exiting threshold $\eta$ during inference (effectively adjusting the agent's internal confidence requirement), we generate a Pareto frontier plotting diagnostic accuracy against the average cumulative cost per patient (Fig.~\ref{fig:pareto_frontier}). Note that while the budgetary penalty $\gamma$ is used during \textit{training} to encourage cost-aware feature acquisition, the Pareto frontier specifically maps the trade-off between the agent's \textit{test-time} confidence (governed by $\eta$) and the resulting economic and predictive performance. As expected, Random Allocation requires nearly all tests (maximum cost) to reach baseline accuracy. The Greedy EIG baseline performs well initially but frequently selects cheap, uninformative tests that require subsequent expensive follow-ups to resolve lingering uncertainty.

AR-NBSR drastically outperforms both baselines, establishing itself as a \textit{non-myopic, resource-rational, active triage agent}. Because the routers are trained end-to-end via the Gumbel-Softmax STE, the routing policy moves beyond myopic step-wise gains to learn a global diagnostic trajectory. By generating a high-resolution frontier using 30 linearly spaced thresholds, we observe a distinct ``elbow'' in the AR-NBSR curve. It achieves its peak diagnostic accuracy ($\sim$90.5\%) at an average cost of approximately $\$45$ to $\$50$, whereas the Greedy baseline requires upwards of $\$60$ to $\$80$ to reach the same predictive asymptote. This organically demonstrates that occasionally selecting a moderately expensive test early in the sequence is optimal if it mathematically guarantees a massive reduction in the downstream diagnostic search space.

\begin{figure}[htbp]
    \centering
    \includegraphics[width=0.5\columnwidth]{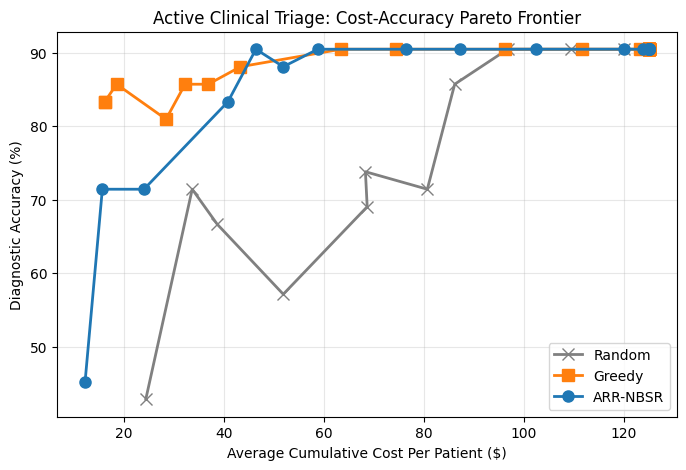}
    \caption{Cost-Accuracy Pareto Frontier. AR-NBSR (Blue) demonstrates superior resource-rationality compared to Greedy (Orange) and Random (Grey) baselines, hitting the peak accuracy asymptote at a significantly lower average cumulative cost.}
    \label{fig:pareto_frontier}
\end{figure}

\paragraph{Dynamic Uncertainty Reduction.}
A hallmark of Bayesian Optimal Experimental Design is that the optimal measurement strategy is inextricably dependent on the prior. Our audit trail (Fig.~\ref{fig:entropy_audit_trail}) provides a \textit{non-myopic proof} of this behavior. The fact that entropy drops significantly across multiple steps ($t=0$ to $t=5$) confirms the model is not relying on a "magic bullet" test but rather gathering a chain of evidence, where each test updates the belief state, and the subsequent test selection is conditioned on the evidence gathered previously. The threshold crossing at $t=2$ demonstrates the model performing dynamic early-exiting: it recognized it had attained sufficient information to be confident and halted testing to save resources.

\begin{figure}[htbp]
    \centering
    \includegraphics[width=0.5\columnwidth]{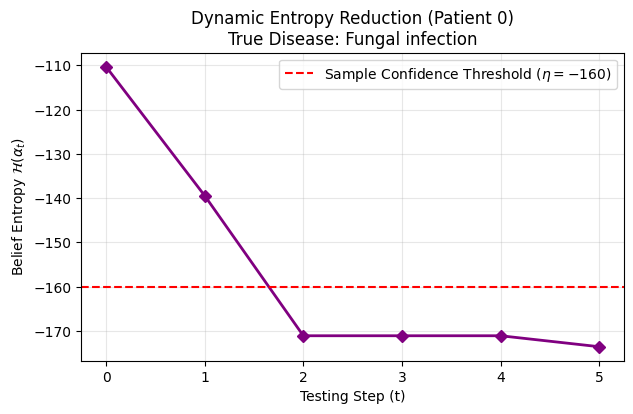}
    \caption{Dynamic Entropy Reduction for a single patient trajectory. The agent sequentially reduces belief entropy until the confidence threshold $\eta$ is breached, demonstrating non-myopic sequential planning.}
    \label{fig:entropy_audit_trail}
\end{figure}

\paragraph{Conclusion on Empirical Performance.}
As demonstrated in our cost-accuracy Pareto Frontier, AR-NBSR exhibits strong resource-rationality. While the Greedy baseline provides a strong myopic upper bound, AR-NBSR achieves competitive diagnostic accuracy with significantly lower average computation per inference. Further, the audit trail confirms that AR-NBSR does not merely execute a fixed sequence, but dynamically reduces belief entropy across multiple time steps, naturally terminating upon reaching the confidence threshold $\eta$. This confirms that our framework is capable of long-horizon planning, moving beyond the myopic limitations of standard active learning.

\section{Discussion} \label{sec:discussion}

This NBSR framework represents a fundamental shift from static, monolithic deep learning towards an active, sequential feature acquisition paradigm. By unifying hierarchical routing, exact Bayesian updating, and discrete conditional execution, NBSR naturally aligns neural inference with human cognitive strategies. Below, we discuss the architectural implications, training dynamics, and theoretical extensions of this framework.

\subsection{Active Knowledge Retrieval vs. Information Bottlenecking}
Standard deep neural networks frequently suffer from information bottlenecking \cite{shwartz2017opening}. When data is presented entirely at the input layer, critical low-level details can be prematurely discarded as information propagates through sequential transformations. For example, early layers in a convolutional neural network may extract broad spatial features while irreparably discarding granular textures that might be vital for downstream, fine-grained classification. While residual connections partially mitigate this by passing signals across contiguous layers, they do not preserve the unadulterated global context across the entire depth of the network. 

In contrast, the NBSR framework fundamentally bypasses this degradation by treating the initial dense embedding as a \textit{Persistent Knowledge Oracle} ($\mathbf{h}_x$). Instead of processing the data in a single monolithic pass where information is passively lost, the decision graph actively and iteratively queries this oracle. In this sense, NBSR shares conceptual similarities with Joint-Embedding Predictive Architectures (JEPA) \cite{lecun2022path}: both systems seek to filter environmental noise and extract abstract, task-relevant true signals. However, in NBSR, this extraction is performed sequentially and information retrieval is done repeatedly. Because each nodal decision-maker selectively chooses exactly what knowledge to retrieve from the global oracle to resolve its immediate uncertainty, this process inherently functions as an \textit{attention mechanism} over the semantic space. Every local nodal expert $f(\mathbf{h}_x; \mathbf{W}_{v_{t+1}}, \mathbf{b}_{v_{t+1}})$ serves as a task-specific probe, dynamically projecting the static, high-dimensional representation into a localized, outcome-discriminative evidence space.

\subsection{Hierarchical Epistemic Capacity and Safe Inference}
To accurately simulate a hierarchical taxonomy, different levels of abstract detail must be absorbed at different depths. As empirically demonstrated in our Toy Experiment (Section~\ref{subsec:toy_iris}), this is mathematically enforced by manipulating the \textit{epistemic capacity} of the local experts. By applying a scaled \textit{Sigmoid} activation to intermediate nodes, we impose a strict "confidence budget", forcing mid-level experts to act as generally broad, hesitant super-categories. Conversely, terminal leaf experts utilize an unbounded \textit{Softplus} activation, allowing them to inject massive evidence to crystallize a sharply defined decision boundary.

A critical architectural imperative in this deployment is the strict prohibition against normalizing these intermediate Dirichlet evidence vectors ($\bm{\alpha}_t$). In standard neural control policies, latent normalization (e.g. via a continuous Softmax) is routinely applied to bound representations. However, in Subjective Logic, the total Dirichlet precision $\alpha_0$ serves as the fundamental mathematical anchor for epistemic uncertainty, defined as $u = K/\alpha_0$ (\ref{eq:uncertainty_Dirichlet_cc}). Artificially normalizing $\bm{\alpha}_t$ forces $\alpha_0$ to a constant, permanently destroying the framework's native ability to trigger out-of-distribution (OOD) safety abstentions. By ensuring that evidence remains strictly raw, additive, and unbounded, NBSR successfully avoids hallucination vulnerabilities, guaranteeing safe epistemic collapse (Corollary~\ref{corollary:1}) as validated in our POMDP Control and Language Modeling experiments.

\subsection{Decoupling of Training and Inference} 
Our NBSR framework explicitly decouples the training and inference regimes to simultaneously maximize \textit{skill acquisition} and \textit{computational efficiency}. In practice, we must balance two critical geometric aspects of the routing tree:

\paragraph{Depth and the ``Lazy Optimization'' Problem.} During training, ideally the dynamic early-exiting mechanism (controlled by the entropy threshold $\eta$) must be intentionally disabled, forcing every sample to traverse the full depth of the graph. Neural networks are inherently ``lazy'' optimizers; they seek the path of least mathematical resistance to minimize the loss function. If early exiting were permitted during the learning phase, intermediate experts would quickly learn \textit{just enough} rudimentary features to push the differential entropy slightly below $\eta$, halting the forward pass prematurely. Consequently, deeper leaf experts would experience \textit{gradient starvation}, rendering the extended architecture functionally dead. By forcing full-depth traversals during training, we guarantee that every level of the graph receives the necessary gradients to acquire specialized skills. At inference time, the $\eta$ threshold is activated, allowing the model to resource-rationally reap the computational savings of these rigorously acquired skills.

\paragraph{Width and the ``Group Project'' Problem.} Conversely, while utilizing the full \textit{depth} is essential during training, utilizing the full \textit{width} (i.e. soft routing) is catastrophic. If the framework utilized standard continuous relaxation where an input traverses every branch and the final update is a probability-weighted average of all terminal leaf experts, the model would fall into \textit{representation collapse} \cite{shazeer2017outrageously}. This dynamic is analogous to a ``group project'' where the final grade is based on a blended average: because the network relies on the ensemble to minimize the loss, individual experts fail to specialize, instead learning generic, overlapping, ``blurry'' features. NBSR resolves this via the Gumbel-Softmax Straight-Through Estimator (STE). By enforcing hard discrete routing during the forward pass, the selected expert is forced to assume 100\% responsibility for the prediction, mathematically guaranteeing absolute specialization. During the backward pass, the continuous relaxation allows the routing policy to evaluate counterfactual probabilities and learn optimal pruning strategies without ever incurring the computational cost of the unselected experts.

\subsection{NBSR as an Markov Decision Process (MDP)}
While NBSR is heavily evaluated on discriminative classification tasks, we can formally characterize its sequential inference process as a discrete-time Markov Decision Process (MDP) \cite{puterman1994mdp} defined by the tuple $(\mathcal{S}, \mathcal{A}, \mathcal{P}, \mathcal{R})$. Here, the state space $\mathcal{S}$ is defined by the augmented belief state $\bm{\alpha}_t$ and the persistent oracle context $\mathbf{h}_x$; the action space $\mathcal{A}$ corresponds to the discrete selection of expert nodes via the router $\pi_\theta(a_t \mid \mathcal{S}_t)$; the transition dynamics $\mathcal{P}(\mathcal{S}_{t+1} \mid \mathcal{S}_t, a_t)$ are governed by the exact Bayesian conjugate update $\bm{\alpha}_{t+1} = \bm{\alpha}_t + \mathbf{e}_t$; and the reward function $\mathcal{R}$ is implicitly defined as the negative differential entropy $-\mathcal{H}(\bm{\alpha}_t)$ plus any associated measurement costs $c(v_t)$. 

In this MDP, the optimal policy $\pi^*$ seeks to maximize the cumulative information gain while minimizing expenditure, effectively treating diagnostic triage and sequential reasoning as an optimal control problem. This MDP formulation reveals that our NBSR architecture is fundamentally a \textit{Policy Network} learning to navigate an epistemic state-space, providing the algorithmic foundation for the Autoregressive State Machine (AR-NBSR) demonstrated in our Active Clinical Triage experiment (Section~\ref{subsec:nbsr_boed}).

\subsection{Modular Skill Acquisition and Unbounded Topologies} 
As the field of artificial intelligence shifts towards modular, agentic systems, architectures must evolve to support composable skill ensembles. In NBSR, each expert network acts as an isolated, highly specialized cognitive skill. The routing tree ensembles these skills both horizontally and sequentially, selectively activating only the specific sub-regions required to process the current input. This mechanism offers a computationally lightweight, Bayesian alternative to standard dense attention mechanisms.

Looking forward, this framework naturally accommodates several rigorous theoretical extensions:
\begin{itemize}
    \item \textbf{Neuro-Symbolic Integration via Product of Experts:} the sequential evidence accumulation natively supports the injection of externally derived rules or heuristics. Future work could calibrate the node Dirichlet distribution via a \textit{Product of Experts}\footnote{PoE is a machine learning framework that models a probability distribution by combining the output from several simpler probability models (experts), this is achieved by multiplying their probability distributions and renormalizing.} (PoE) \cite{hinton2002training, huang2026vjepa}, using a distinct correction term derived from the contextual side information ($\boldsymbol{\varepsilon}_t$).
    \item \textbf{Unbounded Width via Dirichlet Process:} Instead of routing over a fixed $K$-dimensional categorical distribution, routers could be reformulated using a \textit{Dirichlet Process} via a differentiable \textit{stick-breaking mechanism}\footnote{Stick-breaking processes are widely used procedures for constructing random discrete distributions in statistics and machine learning \cite{Gilleyva2026markovstickbreakingprocesses}. Due to their intuitive construction and computational tractability, they have become popular in modern Bayesian nonparametric inference, serving as the foundational mechanism for models such as the Dirichlet and Pitman-Yor processes.} \cite{isharan2001stick}. In this paradigm, a router leaves a residual probability mass to instantiate a newly discovered expert, allowing the horizontal branching factor to grow infinitely and organically if existing experts fail to sufficiently minimize the negative log-likelihood.
    \item \textbf{Unbounded Depth via Information-Gain Splitting:} To dynamically expand vertical depth, we can apply node-splitting rules from traditional decision trees (e.g. CART \cite{breiman1984cart}) to our neural modules. If a terminal expert's average differential entropy $\mathcal{H}(\bm{\alpha}_T)$ plateaus above a critical threshold (which indicates it lacks the capacity to resolve the remaining uncertainty of its assigned data sub-population), the network can dynamically freeze that expert, spawn a new router in its place, and branch into newly initialized child experts to further partition the semantic space.
\end{itemize}

\section{Conclusion} \label{sec:conclusion}

We introduced \textbf{Neural Bayesian Sequential Routing (NBSR)}, a framework for turning neural prediction into a sequential, uncertainty-aware process of evidence acquisition. The central idea is simple but powerful: instead of producing a one-shot softmax decision from a monolithic network, NBSR routes an input through a hierarchy of neural experts, lets each visited expert query a persistent global representation $\mathbf{h}_x$, and accumulates the resulting positive evidence vectors as Dirichlet pseudo-counts. In this way, prediction is no longer merely a class-score computation; it becomes a forward-causal belief trajectory whose intermediate states, uncertainty levels, and expert contributions can be inspected.

The main methodological contribution is the unification of three ingredients that are usually treated separately: conditional neural execution, evidential uncertainty, and sequential decision-making. The Dirichlet-Categorical update $\bm{\alpha}_{t+1}=\bm{\alpha}_t+\mathbf{e}_t$ gives the model an explicit Bayesian belief state over the final outcome space; the Gumbel-Softmax Straight-Through estimator makes discrete, path-dependent routing trainable end-to-end; and the entropy/precision of the evolving Dirichlet state provides a native mechanism for early exiting, abstention, and cost-aware acquisition. This yields a model that is simultaneously predictive, modular, interpretable, and resource-sensitive.

The theoretical analysis clarified what this construction guarantees and what it does not. Under strictly positive evidence extraction, NBSR guarantees monotone growth of total Dirichlet precision and a corresponding upper bound on marginal variance, formalizing the intended ``hypothesis sharpening'' effect. The consistency result shows that, under idealized capacity and optimization assumptions, the expected terminal Dirichlet prediction can recover the Bayes-optimal conditional distribution. The topological analysis further explains how graph depth and width trade bias against variance, while entropy-based early exiting converts this global architectural trade-off into a sample-dependent one.

Empirically, the experiments show that NBSR is most valuable when accuracy, uncertainty, computational selectivity, and interpretability must be considered together. On CIFAR-10, NBSR achieved strong classification performance, reaching $96.74\%$ accuracy with an ECE of $0.015$, while retaining a full routing audit trail and enabling nearly $90\%$ of test samples to exit at Depth 1 without degrading accuracy. The sequential precision and entropy plots also directly validated the predicted belief-sharpening behavior. In structured medical diagnosis, NBSR matched the flat MLP in diagnostic accuracy while providing patient-specific routing traces and path-dependent feature attributions, although tree-based methods remained strong raw-performance baselines on this tabular dataset. In language modeling, NBSR preserved near-baseline perplexity on the controlled grammar task while exposing token-level syntactic/semantic routing decisions, highlighting interpretability and uncertainty structure rather than raw throughput dominance.

The extensions further show that NBSR is not limited to static classification. With recurrent memory, NBSR-Mem matched the black-box CNN-GRU controller on the partially observable navigation task while retaining interpretable policy traces and OOD halting behavior. In the active clinical triage setting, the autoregressive AR-NBSR formulation demonstrated that the same belief-state machinery can act as a cost-aware experimental-design policy: it learned to reduce diagnostic entropy over multiple steps and reached the same predictive asymptote at a substantially lower average cost than the greedy baseline. These results support the view of NBSR as a general architecture for sequential, resource-rational neural inference.

Overall, NBSR suggests a different design principle for modular AI systems: experts should not merely be softly averaged, but selectively queried; uncertainty should not be an auxiliary diagnostic, but part of the state that controls computation; and interpretability should not be reconstructed post hoc, but generated by the forward pass itself. Future work will scale the framework to larger vocabularies and real-world sequential decision tasks, improve calibration of the Dirichlet evidence in high-dimensional settings, and investigate learnable topologies, dynamic experts, neuro-symbolic evidence injection, and nonparametric extensions such as Dirichlet-process routing. These directions may further develop NBSR into a practical foundation for transparent, adaptive, and resource-rational agentic AI.

\newpage
\bibliographystyle{plain}
\bibliography{references}

\appendix

\section{The Dirichlet Distribution} \label{app:dirichlet_dist}

This section provides a formal overview of the Dirichlet distribution, detailing the mathematical properties that make it the foundational engine for the uncertainty-tracking and evidence-accumulating mechanisms in our Bayesian routing framework.

\subsection{Definition and Support}
The Dirichlet distribution is a family of continuous multivariate probability distributions parameterized by a vector of positive reals $\bm{\alpha} = (\alpha_1, \alpha_2, \dots, \alpha_K) \in \mathbb{R}^K_+$. It is defined over the $(K-1)$-dimensional probability simplex, meaning its support consists of $K$-dimensional vectors $\mathbf{p}$ such that $p_k \in (0, 1)$ and $\sum_{k=1}^K p_k = 1$. 

The probability density function (PDF) is given by:
\begin{equation}
    f(\mathbf{p} \mid \bm{\alpha}) = \frac{1}{\text{B}(\bm{\alpha})} \prod_{k=1}^K p_k^{\alpha_k - 1}
\end{equation}
where $\text{B}(\bm{\alpha})$ is the \textit{multivariate Beta function}, expressed in terms of the Gamma function $\Gamma(\cdot)$ as:
\begin{equation}
    \text{B}(\bm{\alpha}) = \frac{\prod_{k=1}^K \Gamma(\alpha_k)}{\Gamma\left(\sum_{k=1}^K \alpha_k\right)}
\end{equation}
\textit{In our framework, the vector $\mathbf{p}$ represents the latent true probability distribution over the $K$ final outcomes, and $\bm{\alpha}_t$ is the concentration parameter which represents the model's intermediate belief state at routing step $t$.}

\subsection{Conjugacy to the Categorical Distribution}
A defining property of the Dirichlet distribution is that it is the conjugate prior for the \textit{Categorical} and \textit{Multinomial} distributions. If the prior distribution over a discrete distribution $\mathbf{p}$ is $\text{Dir}(\bm{\alpha})$, and we observe a set of categorical occurrences represented by a count vector $\mathbf{c}$, the posterior distribution remains a Dirichlet distribution, updated simply by vector addition:
\begin{equation}
    P(\mathbf{p} \mid \mathbf{c}, \bm{\alpha}) = \text{Dir}(\bm{\alpha} + \mathbf{c})
\end{equation}

\textit{Our NBSR framework leverages this conjugacy to perform differentiable Bayesian updating. Instead of discrete counts, our neural experts extract continuous, strictly positive evidence vectors $\mathbf{e}_t$. We treat these continuous vectors as pseudo-counts, allowing the belief state to be updated deterministically at each node via $\bm{\alpha}_{t+1} = \bm{\alpha}_t + \mathbf{e}_t$.}

\subsection{Expectation, Variance, and Covariance (The ``Sharpening'' Effect)}
Let $\alpha_0 = \sum_{k=1}^K \alpha_k$ be the sum of the concentration parameters, often referred to as the precision or inverse-variance parameter. The expected value of the $k$-th component of the distribution $\mathbf{p}$ is strictly proportional to its relative share of the total concentration:
\begin{equation} \label{eq:expected_marginal}
    \mathbb{E}[p_k] = \frac{\alpha_k}{\alpha_k + (\alpha_0 - \alpha_k)} = \frac{\alpha_k}{\alpha_0}
\end{equation}
The variance of each component is given by:
\begin{equation} \label{eq:Dirichlet_var}
    \text{Var}[p_k] = \frac{\alpha_k (\alpha_0 - \alpha_k)}{\alpha_0^2 (\alpha_0 + 1)}
\end{equation}
Furthermore, for any two distinct components $i \neq j$, the covariance is defined as:
\begin{equation} \label{eq:Dirichlet_cov}
    \text{Cov}[p_i, p_j] = \frac{-\alpha_i \alpha_j}{\alpha_0^2 (\alpha_0 + 1)}
\end{equation}
The negative sign in Eq. \ref{eq:Dirichlet_cov} reflects the inherent competition between categories under the simplex constraint $\sum p_k = 1$; as the evidence for one class increases, the evidence for alternatives must necessarily decrease to maintain the sum. Consequently, the covariance matrix of a Dirichlet distribution is singular, reflecting the linear dependence between the components of $\mathbf{p}$.

\textit{The relationship between these moments and the total precision $\alpha_0$ is central to our NBSR framework. As the model traverses deeper into the decision graph and accumulates evidence $\mathbf{e}_t$, $\alpha_0$ strictly increases. Consequently, both the variance and the magnitude of the covariance shrink at a rate of $\mathcal{O}(1/\alpha_0)$, mathematically guaranteeing the progressive ``sharpening'' of the decision boundary where the probability mass concentrates into a single, high-confidence outcome.}

\subsection{Marginal Distributions}
The marginal distribution of each individual component $p_k$ from the Dirichlet-distributed vector $\mathbf{p}$ follows a Beta distribution. Using the total concentration parameter $\alpha_0$, this is defined as:
\begin{equation}
    p_k \sim \text{Beta}(\alpha_k, \alpha_0 - \alpha_k)
\end{equation}

\textit{Isolating the marginal distribution is highly useful during inference, particularly in complex domains such as medical diagnosis. It allows the system to extract the exact confidence intervals for a specific target outcome $k$ (e.g. a specific disease) independently of the remaining $K-1$ alternative classes.}

\subsection{Epistemic Uncertainty and Subjective Logic}
Rooted in the framework of Subjective Logic \cite{josang2016subjective}, the Dirichlet distribution provides a native quantification of confidence that explicitly separates evidence-based certainty from epistemic ignorance. This mathematical bridge was the cornerstone formulation that introduced Evidential Deep Learning to the neural network community \cite{sensoy2018evidential}.

In Subjective Logic, a belief state over $K$ mutually exclusive classes is modeled by assigning a belief mass $b_k \ge 0$ to each class alongside an overall epistemic uncertainty mass $u \ge 0$. These components are strictly constrained to sum to unity:
\begin{equation}
    u + \sum_{k=1}^K b_k = 1
\end{equation}

To map this theoretical belief state to a tractable probabilistic framework, it is tied to the Dirichlet distribution. The concentration parameters $\alpha_k$ are defined as a combination of the accumulated evidence $e_k \ge 0$ and a non-informative prior weight $W$:
\begin{equation}
    \alpha_k = e_k + W a_k
\end{equation}
where $a_k$ is the base rate (prior probability) of class $k$. Assuming a uniform prior across the $K$ classes, we set $a_k = 1/K$ and the prior weight to exactly $W = K$. This simplifies the concentration parameter to $\alpha_k = e_k + 1$. 

Consequently, the total precision $\alpha_0$ (the sum of all concentration parameters) expands to:
\begin{equation}
    \alpha_0 = \sum_{k=1}^K \alpha_k = \sum_{k=1}^K (e_k + 1) = \sum_{k=1}^K e_k + K
\end{equation}

Under this exact mapping, the belief mass for a specific class and the total epistemic uncertainty are defined inversely to the total precision:
\begin{equation} \label{eq:uncertainty_Dirichlet}
    b_k = \frac{e_k}{\alpha_0} \quad \text{and} \quad u = \frac{K}{\alpha_0}
\end{equation}
We can trivially verify that this satisfies the foundational Subjective Logic constraint: $\sum b_k + u = \frac{\sum e_k + K}{\alpha_0} = \frac{\alpha_0}{\alpha_0} = 1$. 

\textit{In our NBSR framework, at the root node, the uniform prior $\bm{\alpha}_0 = \mathbf{1}$ yields a total precision of $\alpha_0 = K$. Because no evidence has been extracted yet ($\sum e_k = 0$), this results in maximum epistemic uncertainty ($u = 1.0$). As the model routes through the graph and accumulates strictly positive evidence vectors ($\mathbf{e}_t > 0$), $\alpha_0$ monotonically increases (Theorem~\ref{theorem:1}). Consequently, $u$ mathematically collapses toward zero. This explicit metric allows the NBSR framework to natively trigger Out-Of-Distribution (OOD) safety abstentions simply by monitoring if $u$ remains dangerously high during inference.}

\subsection{Differential Entropy and Uncertainty Reduction}
To explicitly optimize the router for efficient decision-making, we apply an intermediate penalty based on the differential entropy of the Dirichlet distribution. The differential entropy $\mathcal{H}(\bm{\alpha})$ measures the uncertainty of the belief state and is computed analytically as:
\begin{equation} \label{eq:differential_entropy_Dirichlet}
    \mathcal{H}(\bm{\alpha}) = \log \text{B}(\bm{\alpha}) + (\alpha_0 - K)\psi(\alpha_0) - \sum_{k=1}^K (\alpha_k - 1)\psi(\alpha_k)
\end{equation}
where $\psi(\cdot)$ denotes the digamma function (the logarithmic derivative of the Gamma function).

\textit{In our NBSR framework, at the root node ($t=0$), we initialize the state to $\bm{\alpha}_0 = \mathbf{1}$, which maximizes this entropy equation, representing total structural ignorance. By penalizing $\mathcal{H}(\bm{\alpha}_t)$ during training, we force the routing agents to select paths that rapidly maximize information gain. During inference, this precise entropy measurement serves as the gating mechanism for our dynamic early-exiting policy; routing is halted the moment $\mathcal{H}(\bm{\alpha}_t)$ falls below the confidence threshold $\eta$.}

\subsection{Kullback-Leibler Divergence between Two Dirichlet Distributions}
The Kullback-Leibler (KL) divergence measures the relative entropy or information lost when one Dirichlet distribution, $\text{Dir}(\bm{\beta})$, is used to approximate another, $\text{Dir}(\bm{\alpha})$, over the same $(K-1)$-dimensional simplex. It is calculated as:
\begin{equation}
    D_{\text{KL}}(\text{Dir}(\bm{\alpha}) \parallel \text{Dir}(\bm{\beta})) = \log \frac{\Gamma\left(\sum_{k=1}^K \alpha_k\right)}{\Gamma\left(\sum_{k=1}^K \beta_k\right)} + \sum_{k=1}^K \left[ \log \frac{\Gamma(\beta_k)}{\Gamma(\alpha_k)} + (\alpha_k - \beta_k) \left( \psi(\alpha_k) - \psi\left(\sum_{j=1}^K \alpha_j\right) \right) \right]
\end{equation}

\textit{In our NBSR framework, the KL divergence provides a mathematically rigorous metric to quantify the exact informational impact of a single routing decision. By calculating $D_{\text{KL}}(\text{Dir}(\bm{\alpha}_{t+1}) \parallel \text{Dir}(\bm{\alpha}_t))$, we can measure the precise magnitude of the belief shift caused by the newly extracted evidence $\mathbf{e}_t$.}

\section{The Gumbel Distribution} \label{app:gumbel_dist}

This section provides a formal overview of the Gumbel distribution, detailing its statistical properties and its critical role in the Gumbel-Max trick, which forms the theoretical foundation for discrete stochastic routing.

\subsection{Definition and Support}
The Gumbel distribution (specifically the Type I Extreme Value Distribution) is a continuous probability distribution used to model the distribution of the maximum (or minimum) of a number of samples of various distributions. It has continuous support over the entire real line, $x \in (-\infty, \infty)$. It is parameterized by a location parameter $\mu \in \mathbb{R}$ and a scale parameter $\beta > 0$. 

The probability density function (PDF) of the Gumbel distribution is defined as:
\begin{equation}
    f(x; \mu, \beta) = \frac{1}{\beta} \exp\left( -\frac{x - \mu}{\beta} - \exp\left(-\frac{x - \mu}{\beta}\right) \right)
\end{equation}
The cumulative distribution function (CDF) is given by:
\begin{equation}
    F(x; \mu, \beta) = \exp\left( -\exp\left(-\frac{x - \mu}{\beta}\right) \right)
\end{equation}

In the context of deep learning and our routing framework, we predominantly utilize the Standard Gumbel distribution, where $\mu = 0$ and $\beta = 1$. The CDF thus simplifies to $F(x) = \exp(-\exp(-x))$.

\subsection{Mean and Variance}
The expected value (mean) of a Gumbel-distributed random variable $X \sim \text{Gumbel}(\mu, \beta)$ is:
\begin{equation}
    \mathbb{E}[X] = \mu + \beta \gamma
\end{equation}
where $\gamma \approx 0.5772$ is the Euler-Mascheroni constant. 

The variance is purely a function of the scale parameter:
\begin{equation}
    \text{Var}[X] = \frac{\pi^2}{6} \beta^2
\end{equation}
For the Standard Gumbel distribution, the mean is approximately $0.5772$ and the variance is $\frac{\pi^2}{6} \approx 1.645$.

\subsection{Extreme Value Theory and the Gumbel-Max Trick}
According to the Fisher-Tippett-Gnedenko theorem of extreme value theory \cite{frechet1927loi,fisher1928limiting,vonmises1936distribution,falk1993vonmises,gnedenko1943distribution}, the Gumbel distribution is one of the three\footnote{The other two are the Fréchet distribution, and the Weibull distribution.} possible limiting distributions for the maximum of a sequence of independent and identically distributed (i.i.d.) random variables. 

This property gives rise to the \textit{Gumbel-Max trick} \cite{maddison2014sampling, jang2017categorical,maddison2017concrete}, a mathematically rigorous method for drawing discrete samples from a Categorical distribution. Let $\mathbf{p} = (p_1, p_2, \dots, p_K)$ be the probabilities of a Categorical distribution, and let $g_1, g_2, \dots, g_K$ be i.i.d. samples drawn from a Standard Gumbel distribution. We can draw a discrete categorical sample $z$ (where $z \in \{1, \dots, K\}$) by applying the \textit{argmax} operator to the perturbed log-probabilities:
\begin{equation}
    z = \arg\max_{k \in \{1, \dots, K\}} (\log p_k + g_k)
\end{equation}

\textit{In our NBSR framework, the router at node $v_t$ generates logits $\mathbf{z}_t$. By adding Standard Gumbel noise to these logits before applying the \textit{argmax} operation, we successfully instantiate a stochastic routing policy that samples paths proportionally to the network's confidence, enabling robust exploration during training.}

\section{The Gumbel-Softmax Continuous Relaxation} \label{app:gumbel_softmax_trick}

While the Gumbel-Max trick perfectly models discrete stochastic sampling, the \textit{argmax} operation is fundamentally non-differentiable. Its derivative is zero almost everywhere, which breaks the gradient flow required for backpropagation in deep neural networks. To enable end-to-end training of our hierarchical decision graph, we employ the \textit{Gumbel-Softmax continuous relaxation} \cite{jang2017categorical, maddison2017concrete}.

\subsection{Continuous Approximation of Argmax}

The core insight of the Gumbel-Softmax estimator is to replace the hard, non-differentiable \textit{argmax} operator with the smooth, differentiable \textit{softmax} function. 

Given unnormalized log-probabilities (logits) $\mathbf{h} = (h_1, \dots, h_K)$ and i.i.d. Standard Gumbel noise $\mathbf{g} = (g_1, \dots, g_K)$, the Gumbel-Softmax estimator produces a continuous $K$-dimensional routing vector $\bm{\pi}$, where the $k$-th element is defined as:
\begin{equation} \label{eq:Gumbel-Softmax}
    \pi_k = \frac{\exp((h_k + g_k) / \tau)}{\sum_{i=1}^K \exp((h_i + g_i) / \tau)}
\end{equation}
Here, $\tau > 0$ is the \textit{temperature} hyperparameter. Note that we apply this directly to the logits $\mathbf{h}$ generated by the router network at each node.

\subsection{Temperature Annealing}
The temperature $\tau$ acts as a dial controlling the trade-off between the smoothness of the gradients and the discreteness of the approximation:
\begin{itemize}
    \item \textit{High Temperature ($\tau \to \infty$):} the distribution approaches a uniform continuous distribution, yielding highly stable but uninformative gradients.
    \item \textit{Low Temperature ($\tau \to 0$):} the softmax function sharpens, and the continuous vector $\bm{\pi}$ asymptotically approaches a discrete one-hot vector exactly equivalent to the output of the Gumbel-Max trick. 
\end{itemize}
During training, it is standard practice to anneal $\tau$ from a high initial value to a small strictly positive value. This allows the routing agents to explore the graph smoothly in the early epochs before committing to hard, highly confident discrete paths as training converges.

\subsection{The Straight-Through Estimator (STE)} \label{app:STE}
While the standard Gumbel-Softmax function produces a continuous ``soft'' routing path (e.g. sending $80\%$ of the signal left and $20\%$ right), our framework explicitly requires \textit{hard} conditional execution to maximize computational efficiency (i.e. routing $100\%$ left and computing zero FLOPs on the right). 

To achieve this, we utilize the \textit{Straight-Through Estimator}\footnote{A Straight-Through Estimator (STE) is a technique used in neural networks to train models with non-differentiable operations, such as quantization or binarization, by bypassing the gradient-vanishing problem. It works by using a discrete, non-differentiable function in the forward pass but treating it as a continuous, identity function (or similar mapping) in the backward pass to calculate gradients.} \cite{bengio2013estimating, jang2017categorical} variant of the Gumbel-Softmax trick. During the \textbf{forward pass}, we discretize the continuous output $\bm{\pi}$ using a hard \textit{argmax} to produce a true one-hot routing vector $\bm{\pi}_{\text{hard}}$. This ensures that only a single active path is traversed, and unvisited branches consume zero computational resources. 

During the \textbf{backward pass}, we bypass the non-differentiable \textit{argmax} step and route the gradients directly through the continuous Gumbel-Softmax output $\bm{\pi}$. Mathematically, this is implemented in modern autograd engines via the following detach operation:
\begin{equation} \label{eq:Gumbel-Softmax-detach}
    \bm{\pi}_{\text{out}} = (\bm{\pi}_{\text{hard}} - \bm{\pi})\text{.detach()} + \bm{\pi}
\end{equation}
This formulation guarantees that the forward pass computation exactly mirrors the discrete logical path of the Bayesian tree, while the backward pass receives dense, well-behaved gradients to update the routing parameters $\theta_{v_t}$ and the global feature extractor\footnote{The STE essentially creates a "dual-path" computational graph. In the forward pass, the operation behaves like a step function (discrete), but in the backward pass, it behaves like an identity function or a smooth sigmoid (continuous), allowing the "signal" of the error to reach the weights of the router. This mechanism is what allows our Bayesian DAG to maintain the efficiency of a hard decision tree while still being trainable via standard backpropagation.}.

This mechanism cleanly decouples the gradient flow between the specialized experts and the routing gates. Because unselected experts are multiplied by exactly zero during the forward pass (or bypassed entirely via sparse conditional logic), they do not contribute to the final loss; consequently, they receive zero gradients, preserving their highly specialized weights and preventing representation collapse. Conversely, because the STE routes the backward signal through the continuous distribution $\bm{\pi}$, the loss gradient flows back to \textit{all} branches of the router's logits. This allows the routing policy to evaluate the counterfactual probabilities of unselected branches and update its distribution accordingly, enabling the network to learn optimal pruning strategies without ever incurring the forward or backward computational cost of executing the unselected experts.

\section{Derivation of the Generalized Bias-Variance Decomposition} \label{app:bias_variance}

In this section, we detail the derivation of the bias-variance decomposition for the NLL risk, i.e. Eq.\ref{eq:bias_variance_trade_off} in Section \ref{subsec:bias_variance_trade_off}. 

Let $x$ be a given input and $y$ be the corresponding target. We denote the true conditional data distribution as $P^*(y|x)$. The predictive distribution produced by our neural routing graph, trained on a specific finite dataset $\mathcal{D}$, is denoted as $P_{\mathcal{G}}(y|x; \mathcal{D})$. 

The expected NLL risk for a specific input $x$, evaluated over all possible training datasets $\mathcal{D}$ and all possible targets drawn from the true distribution, is defined as:
\begin{equation} \label{eq:app_risk_base}
    \mathcal{R}(x) = \mathbb{E}_{\mathcal{D}} \Big[ \mathbb{E}_{y \sim P^*(\cdot|x)} \big[ -\log P_{\mathcal{G}}(y|x; \mathcal{D}) \big] \Big]
\end{equation}

To analyze the sources of error, we introduce the \textit{expected predictive distribution} $\bar{P}_{\mathcal{G}}(y|x)$, which represents the model's average prediction over all possible datasets. We can inject the true distribution $P^*(y|x)$ and the expected model distribution $\bar{P}_{\mathcal{G}}(y|x)$ into the logarithm via addition and subtraction:
\begin{align}
    -\log P_{\mathcal{G}}(y|x; \mathcal{D}) &= -\log P^*(y|x) \nonumber \\
    &\quad + \big( \log P^*(y|x) - \log \bar{P}_{\mathcal{G}}(y|x) \big) \nonumber \\
    &\quad + \big( \log \bar{P}_{\mathcal{G}}(y|x) - \log P_{\mathcal{G}}(y|x; \mathcal{D}) \big) \label{eq:app_algebraic_split}
\end{align}

We now substitute Eq.\ref{eq:app_algebraic_split} back into the inner expectation over $y \sim P^*(\cdot|x)$ from Eq.\ref{eq:app_risk_base}. Because the expectation is a linear operator, we can evaluate each of the three terms separately:

\textbf{1. The Irreducible Noise:}
the expectation of the first term depends entirely on the true distribution and represents \textit{the inherent uncertainty in the data generation process} (the Shannon entropy):
\begin{equation}
    \mathbb{E}_{y \sim P^*} \big[ -\log P^*(y|x) \big] = \mathcal{H}(P^*)
\end{equation}

\textbf{2. The Bias:}
the expectation of the second term measures the distance between the true distribution and the model's average prediction. This is precisely the Kullback-Leibler (KL) divergence:
\begin{equation}
    \mathbb{E}_{y \sim P^*} \big[ \log P^*(y|x) - \log \bar{P}_{\mathcal{G}}(y|x) \big] = D_{\text{KL}} \big( P^* \| \bar{P}_{\mathcal{G}} \big)
\end{equation}

\textbf{3. The Variance:}
The expectation of the third term captures how much the predictions from models trained on specific datasets fluctuate around the average model prediction:
\begin{equation}
    \mathbb{E}_{y \sim P^*} \big[ \log \bar{P}_{\mathcal{G}}(y|x) - \log P_{\mathcal{G}}(y|x; \mathcal{D}) \big] = \sum_y P^*(y|x) \log \frac{\bar{P}_{\mathcal{G}}(y|x)}{P_{\mathcal{G}}(y|x; \mathcal{D})}
\end{equation}
In standard generalized bias-variance decompositions for likelihood estimators, it is common practice to approximate this final term by replacing the true distribution $P^*(y|x)$ with the expected model distribution $\bar{P}_{\mathcal{G}}(y|x)$. This substitution decouples the model's internal variance from the external ground-truth distribution, yielding a pure measure of model instability evaluated via KL divergence:
\begin{equation}
    \approx \sum_y \bar{P}_{\mathcal{G}}(y|x) \log \frac{\bar{P}_{\mathcal{G}}(y|x)}{P_{\mathcal{G}}(y|x; \mathcal{D})} = D_{\text{KL}} \big( \bar{P}_{\mathcal{G}} \| P_{\mathcal{G}}(\cdot; \mathcal{D}) \big)
\end{equation}

\textbf{Final Decomposition:}
applying the outer expectation over all datasets $\mathbb{E}_{\mathcal{D}}[\cdot]$ to the aggregated terms yields the final formal decomposition of the expected risk:
\begin{equation} \label{eq:app_final_decomp}
    \mathbb{E}_{\mathcal{D}} \Big[ \mathbb{E}_{y \sim P^*(\cdot|x)} \big[ -\log P_{\mathcal{G}}(y|x; \mathcal{D}) \big] \Big] \approx \underbrace{\mathcal{H}(P^*)}_{\text{Irreducible Noise}} + \underbrace{D_{\text{KL}} \big( P^* \| \bar{P}_{\mathcal{G}} \big)}_{\text{Bias}} + \underbrace{\mathbb{E}_{\mathcal{D}} \big[ D_{\text{KL}} \big( \bar{P}_{\mathcal{G}} \| P_{\mathcal{G}}(\cdot; \mathcal{D}) \big) \big]}_{\text{Variance}}
\end{equation}

This derivation explicitly maps the negative log-likelihood objective of the routing graph to the structural constraints ($L$ and $\bar{W}$) discussed in Section \ref{subsec:bias_variance_trade_off}.

\section{Difference Between NBSR and Classical Decision Trees}
\label{app:nbsr_vs_classical_trees}

A fundamental distinction must be drawn between our NBSR framework and classical hierarchical models, such as Classification and Regression Trees (CART) \cite{breiman1984cart}, Soft Decision Trees (SDTs) or hierarchical Mixtures of Experts (MoEs) \cite{jordan1993HME}. The divergence lies in four critical architectural paradigms: the functional role of intermediate nodes, the mathematical viability of soft routing, the semantic nature of the input space, and the fundamental mechanism of information aggregation.

\paragraph{1. The Role of Intermediate Nodes: Routing Gates vs. Sequential Evidence Extractors}
In classical differentiable decision trees and hierarchical MoEs, the graph is composed of two strictly distinct types of nodes:
\begin{itemize}
    \item \textbf{Internal Nodes (Routers):} these nodes do not produce any classification output. They act purely as routing gates that calculate transitional probabilities (e.g. $0.7$ probability of routing left, $0.3$ right).
    \item \textbf{Terminal Leaves (Experts):} these are the only nodes authorized to make a decision or output a class prediction.
\end{itemize}
In contrast, NBSR abandons this dichotomy. In our framework, each NBSR tree node, whether internal or terminal, consists of both a router and an evidence extractor (see Diagram~\ref{fig:NBSR_pipeline}). Consequently, intermediate nodes act as \textit{both routers and experts}. When a sample arrives at an intermediate node (e.g. the Depth 1 \texttt{animal\_mid} node in the CIFAR-10 classification task), the model does not merely calculate the probability of traversing to the \texttt{Pets} or \texttt{Wildlife} leaves. Instead, it immediately queries the global knowledge oracle $\mathbf{h}_x$ to extract a full $K$-dimensional evidence vector (Eq.\ref{eq:evidence_extraction_activation}):
\begin{equation*}
    \mathbf{e}_{\text{mid}} = \text{Activation}(\mathbf{W}_{\text{mid}} \mathbf{h}_x + \mathbf{b}_{\text{mid}})
\end{equation*}
This evidence is directly injected into the Dirichlet prior ($\bm{\alpha}_1 = \bm{\alpha}_0 + \mathbf{e}_{\text{mid}}$). By updating the belief state \textit{at} the intermediate nodes before passing the computational flow deeper into the tree, the NBSR framework allows the decision to sequentially sharpen. This provides the mathematical foundation for dynamic early exiting, as the model can evaluate confidence at $t=1$ without needing to reach a terminal leaf.

\paragraph{2. The Failure of Soft Routing in Sequential Frameworks}
The sequential nature of NBSR fundamentally breaks the standard "soft routing" paradigm used by classical neural trees. 

In a classical hierarchical MoE, an image traverses every branch of the tree simultaneously. The network calculates a \textit{Path Probability} for every terminal leaf by taking the product of the routing decisions along that branch. For example:
\begin{equation}
    \pi_{\text{pets}} = P(\text{Animal} \mid \text{Root}) \times P(\text{Pets} \mid \text{Animal})
\end{equation}
The final prediction is then calculated by multiplying every terminal leaf's output by its respective path probability and summing them together to form a weighted continuous average:
\begin{equation}
    \mathbf{y}_{\text{final}} = \sum_{\text{all leaves}} \pi_{\text{leaf}} \cdot \mathbf{y}_{\text{leaf}}
\end{equation}

While this probability-weighted average of ensembles works for classical trees \textit{where only leaves produce outputs}, it fails catastrophically in the NBSR framework. Because NBSR extracts evidence at \textit{multiple depths}, applying standard soft routing would result in a severe entanglement of intermediate and leaf evidence:
\begin{equation}
    \bm{\alpha}_T = \bm{\alpha}_0 + \sum_{\text{mid\_nodes}} \pi_{\text{mid}} \mathbf{e}_{\text{mid}} + \sum_{\text{leaf\_nodes}} \pi_{\text{leaf}} \mathbf{e}_{\text{leaf}}
\end{equation}
If the router were allowed to output continuous probabilities (e.g. a 30\%/30\%/40\% split), the final belief state would become a blended average of intermediate evidence layered on top of a blended average of leaf evidence. This compounding continuous relaxation inevitably leads to \textit{representation collapse}. The local experts would fail to specialize, learning generic, overlapping features because the network relies on the blended ensemble to minimize the loss.

By enforcing \textbf{hard discrete routing} via the Gumbel-Softmax Straight-Through Estimator (STE), NBSR sidesteps this failure mode. The STE forces $\pi \in \{0, 1\}$ (i.e. 'hard routing'), ensuring that exactly one clean, sequential path of evidence (e.g. $\mathbf{e}_{\text{animal\_mid}} + \mathbf{e}_{\text{pets}}$) is extracted and added to the Dirichlet prior, thereby guaranteeing absolute expert specialization while retaining end-to-end differentiability.

\paragraph{3. The Input Space: Raw Partitioning vs. Semantic Oracle Querying}
In classical hierarchical models like CART \cite{breiman1984cart} and standard HMEs \cite{jordan1993HME, waterhouse1995HME}, both the routing decisions and the expert predictions are typically computed directly from the raw input features $x$ (e.g., via scalar thresholds or simple linear projections). The tree explicitly partitions this raw input space. 

In contrast, NBSR introduces a \textit{Persistent Global Knowledge Oracle} ($\mathbf{h}_x$). A shared deep neural backbone first maps the raw input into a dense, high-dimensional semantic space. This oracle is then broadcast to the entire graph - it's presented to all nodes at all depths. Consequently, the local experts in NBSR do not partition raw pixels or isolated tabular columns; instead, they act as \textit{active attention mechanisms}, applying specialized learned filters ($\mathbf{W}_{v}$) to dynamically extract high-level conceptual evidence directly from the shared global representation at various stages of the reasoning process.

\paragraph{4. Aggregation Mechanism and Uncertainty Quantification: Multiplicative Probabilities vs. Additive Evidence}
In classical HMEs \cite{jordan1993HME, waterhouse1995HME}, the final prediction is formed by \textit{multiplying} gating probabilities along the path from root to leaf to form a path weight, which then weights the expert's probability distribution. Because this framework operates entirely within a zero-sum probability space ($\sum p = 1$), it forces the network to distribute a total probability mass of 1 across all outcomes, stripping the model of any native ability to quantify structural ignorance. 

In contrast, NBSR operates in an \textit{evidential} space. Instead of multiplying probabilities, NBSR \textit{adds} continuous evidence vectors to a Dirichlet belief state ($\bm{\alpha}_{t+1} = \bm{\alpha}_t + \mathbf{e}_t$). This additive accumulation breaks the zero-sum constraint during the intermediate steps, enabling the model to physically inflate the total volume of evidence (the Dirichlet precision, $\alpha_0$). Consequently, NBSR natively provides rigorous Bayesian uncertainty quantification ($u = K / \alpha_0$). This allows the NBSR framework to explicitly express ``I don't know'' when encountering out-of-distribution data - a critical safety capability fundamentally absent in classical HMEs and traditional decision trees.

\section{Further Results for CIFAR-10 Classification}
\label{app:CIFAR10_further_results}

\subsection{Computing Environment and Experimental Setup}
\paragraph{Computing Environment.} All experiments, training regimes, and inference benchmarking were conducted on a high-performance Linux server (x86\_64 architecture) equipped with a 24-core (48 logical threads) AMD EPYC 9B45 processor and 190 GB of system memory. GPU acceleration was provided by a single NVIDIA RTX PRO 6000 Blackwell Server Edition with 96 GB of VRAM, running CUDA version 13.0 and NVIDIA Driver 580.82.07. To ensure strict reproducibility across all baselines, identical deterministic seeds (Seed = 111) were applied to Python's `random`, `NumPy`, and `PyTorch` libraries, alongside the enforcement of deterministic CuDNN backend algorithms.

\paragraph{Network Architectures.} The \textit{visual feature extraction backbone} for all models is an unfrozen ResNet-18 \cite{he2016deep}, initialized with ImageNet weights. The network is truncated prior to its final classification head, yielding a 512-dimensional global feature vector $\mathbf{h}_x$ for each input image. 

In our Neural Bayesian Sequential Routing (NBSR) framework, the routers and experts are instantiated as follows:
\begin{itemize}[label=-]
    \item \textit{Routers:} each internal node acts as a gating mechanism parameterized by a \textit{Multi-Layer Perceptron} (MLP). The router concatenates the visual features $\mathbf{h}_x$ with the current Dirichlet concentration state $\bm{\alpha}_t$, passing them through a 128-unit hidden layer with a \textit{ReLU activation}, followed by a \textit{linear projection} to the number of branching paths. The output is sampled via a continuous \textit{Gumbel-Softmax approximation} during training.
    
    \item \textit{Local Experts:} each expert module (e.g. \texttt{animal\_mid}, \texttt{pets}, \texttt{wildlife}) consists of a single \textit{linear transformation} ($\mathbb{R}^{512 \to 10}$) followed by a \textit{Softplus activation} function. The Softplus nonlinearity is strictly required to ensure that the extracted evidence vector remains positive ($e_{k} > 0$), preserving the mathematical integrity of the Dirichlet parameter updates.
\end{itemize}

\paragraph{Hyperparameters and Optimization.} All models were trained for a total of 150 epochs with a batch size of 128. We optimized the networks using the \textit{Adam optimizer} \cite{kingma2017Adam} with an initial learning rate of $2 \times 10^{-4}$, decayed by a factor of $\gamma = 0.5$ every 45 epochs via a StepLR scheduler. Gradients were clipped at a maximum norm of $1.0$ to ensure stability. 

For the data pipeline, the CIFAR-10 images were upscaled to $224 \times 224$ to match the expected resolution of the ImageNet-pretrained ResNet backbone. Standard spatial augmentations were applied during training, including \textit{random horizontal flipping}, \textit{random rotation} ($\pm 15^{\circ}$), and \textit{color jittering} (brightness $= 0.2$), followed by standard \textit{channel-wise normalization}. 

For the NBSR specific hyperparameters, the structural entropy penalty weight was set to $\lambda = 10^{-3}$. To balance early-stage structural exploration with late-stage discrete routing, the Gumbel-Softmax temperature $\tau$ was annealed exponentially per epoch from an initial value of $1.0$ down to a minimum threshold of $0.1$ at a decay rate of $0.97$.

\subsection{Evaluation Metric: Expected Calibration Error (ECE)}
The Expected Calibration Error (ECE) is a standard metric used to quantify how well a model's predicted confidence aligns with its actual empirical accuracy. Geometrically, it measures the aggregate absolute deviation from perfect calibration. To calculate the ECE, the continuous probability space $[0, 1]$ is partitioned into $M$ equally spaced bins. Each test sample is assigned to a specific bin based on its maximum predicted confidence score. The final error is then calculated as the weighted average of the absolute difference between the true accuracy and the mean predicted confidence within each bin:
$$ \text{ECE} = \sum_{m=1}^M \frac{|B_m|}{N} \left| \text{acc}(B_m) - \text{conf}(B_m) \right| $$
where $N$ is the total number of evaluated samples, $B_m$ represents the set of samples whose predicted confidence falls into bin $m$, $|B_m|$ is the total number of samples in that bin, $\text{acc}(B_m)$ is the empirical accuracy of those specific samples, and $\text{conf}(B_m)$ is their average predicted confidence. 

In our evaluations, we partition the probability space into $M=10$ bins. For the standard deterministic baselines (Flat ResNet-18 and Sparse MoE), the \textit{maximum softmax probabilities} are used as the \textit{confidence scores}. For our NBSR framework, the structurally derived Bayesian expected marginals $\mathbb{E}[p_k]$ (Eq.\ref{eq:expected_marginal_cc}) are passed directly into this algorithm. As summarized in the main text (Table.\ref{tab:results_cifar}), this evaluates to an ECE of 0.015, mathematically demonstrating that the evidence-based uncertainty of NBSR aligns much closer to empirical reality than the overconfident deterministic baselines.

\subsection{Baseline Training Dynamics}
To complement the quantitative results presented in the main text (Section.\ref{subsec:cifar10_experiment}), we provide the complete end-to-end training loss and test accuracy trajectories for the \textit{Flat ResNet-18} (Fig.~\ref{fig:flat_resnet_curves}) and the \textit{Sparse MoE} (Fig.~\ref{fig:sparse_moe_curves}) baselines over the 150-epoch training regime. 

Notably, the Sparse MoE exhibits distinct variance and oscillation in its test accuracy, particularly in the later stages of training. This instability is a classic symptom of the continuous soft-routing gates competing against the classification experts, further highlighting the optimization friction introduced by standard MoE architectures on semantic classification tasks.

\begin{figure}[htbp]
    \centering
    \includegraphics[width=0.9\columnwidth]{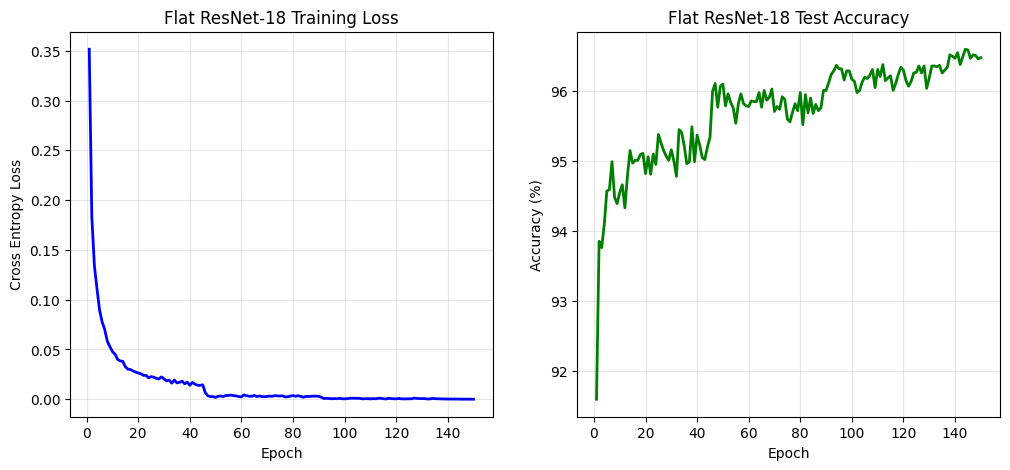}
    \caption{Training loss and test accuracy dynamics for the standard Flat ResNet-18 baseline on the CIFAR-10 dataset.}
    \label{fig:flat_resnet_curves}
\end{figure}

\begin{figure}[htbp]
    \centering
    \includegraphics[width=0.9\columnwidth]{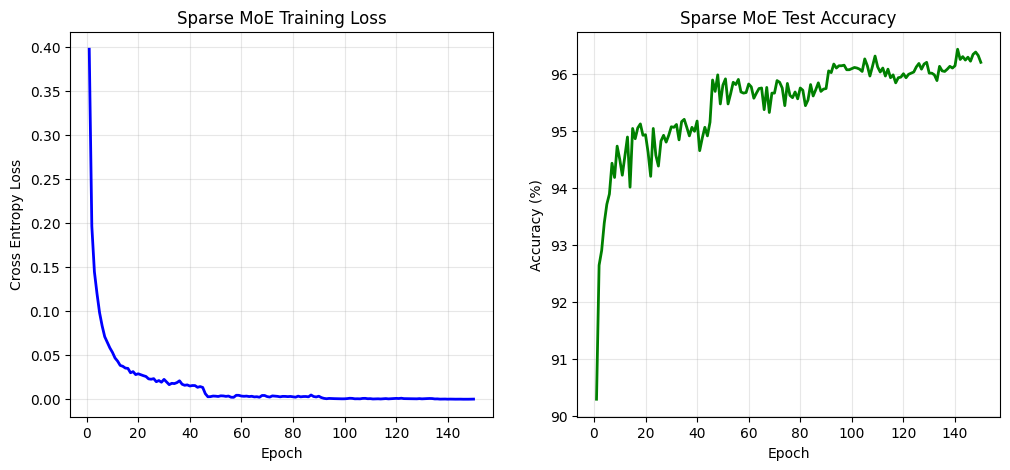}
    \caption{Training loss and test accuracy dynamics for the Sparse MoE (Soft Routing) baseline on the CIFAR-10 dataset. The noticeable jitter in the accuracy curve illustrates the optimization friction inherent in continuous soft-routing mechanisms.}
    \label{fig:sparse_moe_curves}
\end{figure}

\section{Experimental Details: Structured Medical Diagnosis}
\label{app:medical_details}
 
Here we provide the technical details, network architectures, and hyperparameters required to reproduce the structured tabular clinical experiments presented in Section.\ref{subsec:medical_diagnosis}.

\subsection{Computing Environment}
All experiments, including baseline training and NBSR evaluation, were executed on a cloud-based virtual machine instance (\textit{Google Colab}) to measure CPU-bound inference efficiency. The hardware specifications are as follows:
\begin{itemize}
    \item \textit{Processor:} Intel(R) Xeon(R) CPU @ 2.20GHz (1 Physical Core, 2 Logical Threads)
    \item \textit{System Memory (RAM):} 13.61 GB
    \item \textit{Hardware Accelerator:} none (CPU-only execution)
    \item \textit{Software Stack:} Python 3.12, PyTorch 2.10.0, and XGBoost 2.0.
\end{itemize}

\subsection{Dataset and Preprocessing}
We utilized the Kaggle Disease Symptom Prediction dataset\footnote{\url{https://www.kaggle.com/datasets/kaushil268/disease-prediction-using-machine-learning}}, which maps patient profiles to specific clinical endpoints, which maps patient profiles to specific clinical endpoints. 
\begin{itemize}
    \item \textit{Input Features ($X$):} 132 binary indicators representing the presence (1) or absence (0) of specific clinical symptoms (e.g. `skin\_rash`, `joint\_pain`).
    \item \textit{Target Classes ($Y$):} 41 discrete disease categories.
    \item \textit{Clinical Noise Injection:} real-world Electronic Health Records (EHR) are inherently noisy due to entry errors, missing tests, and ambiguous patient reporting. To simulate this environment and prevent the models from overfitting to a perfectly deterministic toy dataset, we applied a uniform 5\% noise mask to both the training and testing sets. Specifically, for 5\% of the feature matrix $X$, the binary symptom states were inverted ($X_{i,j} \leftarrow 1 - X_{i,j}$).
\end{itemize}

\subsection{Network Architectures}

\paragraph{Flat MLP Baseline.} 
The standard monolithic deep learning baseline was constructed using a sequential stack of dense layers:
\begin{enumerate}
    \item Linear(132 $\to$ 128), ReLU activation
    \item Linear(128 $\to$ 128), ReLU activation
    \item Linear(128 $\to$ 41), Softmax activation
\end{enumerate}

\paragraph{NBSR Architecture.} 
The NBSR framework utilized a maximum depth of 2 (Root $\to$ Mid $\to$ Leaf) with a uniform branching factor of 4. The modular components were parameterized as follows:
\begin{itemize}
    \item \textit{Global Oracle Backbone:} identical to the MLP hidden layers to ensure fair semantic capacity: Linear(132 $\to$ 128) $\to$ ReLU $\to$ Linear(128 $\to$ 128) $\to$ ReLU.
    \item \textit{Router Networks:} MLP processing the concatenated feature state $\mathbf{h}_x$ and current belief $\bm{\alpha}_t$: Linear(169 $\to$ 64) $\to$ ReLU $\to$ Linear(64 $\to$ 4), followed by a Gumbel-Softmax activation.
    \item \textit{Expert Networks (Evidence Extractors):} linear(128 $\to$ 41) followed by a Softplus activation to ensure strictly non-negative evidence vectors $\mathbf{e}_t > 0$.
\end{itemize}

\subsection{Optimization and Hyperparameters}
Both the Flat MLP and NBSR networks were trained using \textit{Adam optimizer} \cite{kingma2017Adam} with a learning rate of $1 \times 10^{-3}$ and a batch size of 128 for 40 epochs. A global random seed of 111 was enforced across data splits, noise injection, and network initializations.

For the XGBoost baseline, we utilized the default `XGBClassifier` parameters with `eval\_metric='mlogloss'` and `use\_label\_encoder=False`.

Specific NBSR hyperparameters were set as follows:
\begin{itemize}
    \item \textit{Routing Temperature ($\tau$):} the Gumbel-Softmax temperature was exponentially annealed per epoch to smoothly transition from soft exploration to hard, discrete routing: $\tau = \max(0.1, 1.0 \cdot 0.9^{\text{epoch}})$.
    \item \textit{Entropy Penalty ($\lambda$):} to prevent artificial overconfidence on the heavily perturbed noisy clinical data, the explicit entropy minimization penalty was disabled ($\lambda = 0.0$). The network relied purely on the Negative Log-Likelihood (NLL) of the expected Dirichlet probability.
    \item \textit{Early Exiting Threshold ($\eta$):} during inference, the ``Deep'' configuration bypassed early exiting (threshold evaluated as `None`), forcing the full architectural depth. The ``Fast'' configuration utilized a differential entropy threshold of $\eta = -100.0$ to dynamically truncate routing on unambiguous clinical presentations.
\end{itemize}

\section{Experimental Details: Language Modeling}
\label{app:lm_details}

Here we provide the technical details, network architectures, and hyperparameters required to reproduce the synthetic language modeling and contextual disambiguation experiments presented in Section.\ref{sec:language_modelling}. 

\subsection{Computing Environment}
All sequence modeling experiments were deliberately executed on a cloud-based virtual machine instance (Google Colab). The hardware specifications are as follows:
\begin{itemize}
    \item \textit{Processor:} Intel(R) Xeon(R) CPU @ 2.20GHz (1 Physical Core, 2 Logical Threads)
    \item \textit{System Memory (RAM):} 13.61 GB
    \item \textit{Hardware Accelerator:} none (CPU-only execution)
    \item \textit{Software Stack:} Python 3.12, PyTorch 2.1.0
\end{itemize}

\subsection{Dataset and Preprocessing}
To perfectly isolate syntactic reasoning from semantic world-knowledge, we constructed a synthetic contextual disambiguation corpus.
\begin{itemize}
    \item \textit{Vocabulary ($|\mathcal{V}|$):} 65 discrete tokens, strictly partitioned into 4 special tokens (\texttt{<pad>}, \texttt{<bos>}, \texttt{<eos>}, \texttt{<unk>}), 19 function words, 14 modifiers, 14 abstract nouns, and 14 concrete entities.
    \item \textit{Generative Templates:} sequences were procedurally generated using uniform random sampling across 5 predefined syntactic templates (e.g. \texttt{['function', 'concrete', 'function', 'function', 'abstract']}).
    \item \textit{Data Splits:} The dataset was split into 10,000 training sequences and 2,000 held-out test sequences. All sequences were padded to a maximum length of 6 tokens.
\end{itemize}

\subsection{Network Architectures}

\paragraph{Shared Causal Transformer Backbone.} 
All three evaluated models (Standard, MoE, NBSR) share an identical Transformer backbone to ensure fair semantic capacity. It consists of an embedding layer ($d_{model} = 64$) with positional encoding, followed by a 2-layer causal Transformer Encoder utilizing 4 attention heads and a feed-forward dimension of 256.

\paragraph{Standard Transformer Baseline.} 
The monolithic baseline processes the backbone's contextual representation $\mathbf{h}_x$ through a standard dense linear projection: Linear(64 $\to$ 65), followed by a standard Log-Softmax activation.

\paragraph{Transformer MoE Baseline.} 
The MoE baseline utilizes a discrete Top-1 routing head without evidence accumulation. It consists of a router network—Linear(64 $\to$ 4) with Softmax, and 4 independent expert networks parameterized as Linear(64 $\to$ 65).

\paragraph{NBSR Architecture.} 
The NBSR framework utilizes a hierarchical Directed Acyclic Graph (DAG) with a maximum depth of 2. The modular components were parameterized as follows:
\begin{itemize}
    \item \textit{Root Router:} an MLP processing the concatenated state $\mathbf{h}_x$ and the uniform Dirichlet prior $\bm{\alpha}_0$: Linear(129 $\to$ 32) $\to$ ReLU $\to$ Linear(32 $\to$ 2), followed by a Gumbel-Softmax activation.
    \item \textit{Mid-Level Experts (Syntax vs. Semantics):} 2 distinct experts, each structured as Linear(64 $\to$ 64) $\to$ ReLU $\to$ Linear(64 $\to$ 65), followed by a \textit{Softplus} activation to ensure strictly non-negative evidence vectors $\mathbf{e}_t > 0$.
    \item \textit{Leaf Experts (PoS Categories):} 4 terminal experts mapping to the distinct linguistic sub-spaces (Function, Modifier, Abstract, Concrete). These are identical in structure to the mid-level experts.
\end{itemize}

\subsection{Optimization and Hyperparameters}
All networks were trained using the \textit{Adam optimizer} \cite{kingma2017Adam} with a learning rate of $1 \times 10^{-3}$ and a batch size of 64 for 5 epochs. A global random seed of 42 was enforced across data generation and network initializations.

Specific NBSR hyperparameters were set as follows:
\begin{itemize}
    \item \textit{Routing Temperature ($\tau$):} the Gumbel-Softmax temperature was exponentially annealed per epoch to smoothly transition from soft exploration to hard, discrete routing: $\tau = \max(0.1, 1.0 \cdot 0.9^{\text{epoch}})$.
    \item \textit{Early Exiting Threshold ($\eta$):} during inference, the ``Deep'' configuration bypassed early exiting (threshold evaluated as `None`), forcing the full architectural depth. The ``Fast'' configuration utilized a differential entropy threshold of $\eta = 1.0$ to dynamically halt routing on predictable function words.
    \item \textit{OOD Abstention Threshold ($\tau_{conf}$):} the critical confidence threshold for evaluating the epistemic uncertainty of the total Dirichlet precision on out-of-distribution prompts was defined mathematically as $\tau_{conf} = 1.5 \times |\mathcal{V}| = 97.5$.
\end{itemize}

\section{Experimental Setup for the POMDP Navigation Task}
\label{appendix:control_setup}

Here we provide detailed specifications for the experimental environment, dataset generation, model architectures, and training hyperparameters utilized in the sequential control and planning task (Section \ref{subsec:control}).

\subsection{Computing Environment}
All experiments, training, and evaluations were conducted on a Google Colab virtual instance. To ensure reproducibility and demonstrate the lightweight nature of the NBSR-Mem architecture, the experiments were executed entirely on the CPU without hardware acceleration. The environment specifications are as follows:
\begin{itemize}
    \item \textit{CPU:} AMD EPYC 7B12 (1 physical core, 2 logical threads), 2.25 GHz base clock.
    \item \textit{Memory:} 13.61 GB RAM.
    \item \textit{Software Stack:} Python 3.12.13, PyTorch 2.10.0+cpu.
\end{itemize}

\subsection{POMDP Dataset Generation}
The navigation task was framed as a classification problem via Behavioral Cloning. We procedurally generated a synthetic dataset of optimal state-action sequences representing a Partially Observable Markov Decision Process (POMDP).

\paragraph{State Representation:} each state observation $o_t$ is a $4 \times 5 \times 5$ spatial tensor representing a local grid centered on the agent. The 4 channels encode distinct environmental features (e.g. floor, agent position, left-turn cues, right-turn cues).

\paragraph{Trajectory Structure:} each sequence has a fixed length of $T=4$ timesteps, structured to enforce reliance on temporal memory:
\begin{enumerate}
    \item \textit{$t=0$ (Memory Cue):} the agent receives a cue dictating the ultimate turning direction. A random variable determines if the sequence is ``simple'' (agent moves straight indefinitely) or ``complex'' (agent must evade a wall at the end). For complex sequences, a target direction (Left or Right) is selected, and a corresponding visual cue is instantiated in the spatial tensor.
    \item \textit{$t=1, 2$ (Spatial Amnesia Zone):} the agent traverses a featureless corridor. The optimal action is 0 (Cruising/Straight). The visual cue from $t=0$ is no longer present in the observation.
    \item \textit{$t=3$ (Intersection):} if the sequence is complex, a wall appears directly ahead of the agent. The optimal action is the target direction designated at $t=0$. If the sequence is simple, the corridor remains empty, and the optimal action remains 0.
\end{enumerate}

\paragraph{Out-of-Distribution (OOD) Generation:} To test epistemic safety, an ``Alien Hazard'' was generated by injecting a high-magnitude activation into a specific coordinate of the fourth visual channel - a feature entirely absent from the training distribution.

\paragraph{Dataset Splits:} The generated dataset comprised 5,000 training sequences and 1,000 test sequences. The probability of a sequence being ``complex'' (requiring a turn) was set to 70\% to heavily penalize purely reactive models.

\subsection{Model Architectures}

All models share a common \textit{Global Knowledge Oracle} (CNN Backbone) to ensure a fair comparison of routing and memory capabilities.

\paragraph{CNN Backbone:}
The feature extractor consists of two convolutional layers followed by a linear projection:
\begin{itemize}
    \item \texttt{Conv2d}: 4 input channels, 32 output channels, kernel size $3 \times 3$, padding 1.
    \item \texttt{Conv2d}: 32 input channels, 64 output channels, kernel size $3 \times 3$, padding 0.
    \item \texttt{Linear}: 576 ($64 \times 3 \times 3$) to 128 dimensions.
    \item \texttt{LayerNorm}: applied to the 128-dimensional output without elementwise affine parameters to strictly anchor the latent space.
    \item \textit{Activations:} ReLU is applied after all convolutional and linear layers. \textit{Note: Biases were removed from all layers to improve evidential calibration.}
\end{itemize}

\paragraph{Action Space and Network Outputs:}
the theoretical action space $\mathcal{A}$ is defined by four discrete movements: 0 (Cruising/Straight), 1 (Reverse), 2 (Evasion Left), and 3 (Evasion Right). To enable gradient-based optimization, the network architectures output continuous 4-dimensional vectors representing their predictive confidence. For the standard baselines (CNN and CNN-GRU), the final classifier outputs a vector of log-probabilities via a Log-Softmax activation. Conversely, the NBSR architectures output a vector representing the concentration parameters of a Dirichlet distribution, $\boldsymbol{\alpha} = \mathbf{e} + \mathbf{1}$, where $\mathbf{e} \ge 0$ denotes the accumulated evidence for each action. This distinction is critical for safety: while a highly uncertain baseline is mathematically forced to distribute probabilities that sum to 1 (e.g. $[0.25, 0.25, 0.25, 0.25]$), the NBSR model can output $\boldsymbol{\alpha} = [1.0, 1.0, 1.0, 1.0]$, explicitly quantifying a state of total epistemic uncertainty (zero evidence) to trigger a safe halt. During training, the discrete expert action $a^*_t$ is transformed into a one-hot vector to properly isolate the penalty applied by the masked evidential regularizer.

\paragraph{Memory Module (CNN-GRU \& NBSR-Mem):}
for models equipped with temporal memory, the 128-dimensional output of the CNN backbone is processed sequentially by a standard Gated Recurrent Unit (\texttt{nn.GRU}) with an input size of 128, a hidden size of 128, and `batch\_first=True`.

\paragraph{Hierarchical Evidential DAG (NBSR \& NBSR-Mem):}
the routing topology consists of three levels:
\begin{itemize}
    \item \textit{Root Router:} a single \texttt{Linear(128, 2)} layer mapping the hidden state to the two abstract modes (Cruising vs. Evasion).
    \item \textit{Mid-Level Experts:} two \texttt{Linear(128, 4)} layers (without biases), one for each abstract mode. These extract initial Dirichlet evidence ($\mathbf{e}_{mid}$).
    \item \textit{Leaf Experts:} four \texttt{Linear(128, 4)} layers (without biases). These extract additive Dirichlet evidence ($\mathbf{e}_{leaf}$) if the Depth 1 entropy remains above the threshold.
\end{itemize}
During training, the Root Router utilizes a standard Softmax activation. During inference, it uses an argmax (one-hot) to enable discrete early exits. All experts utilize a Softplus activation to ensure non-negative evidence generation.

\subsection{Training Protocol and Hyperparameters}

All models were trained via Behavioral Cloning \cite{pomerleau1989alvinn} using the Adam optimizer \cite{kingma2017Adam}. To stabilize the Recurrent Evidential Routing Network, we employed a specific combination of losses and regularizers.

\paragraph{Hyperparameters:}
\begin{itemize}
    \item \textit{Batch Size:} 64
    \item \textit{Epochs:} 30
    \item \textit{Learning Rate:} $2 \times 10^{-3}$
    \item \textit{Weight Decay:} $2 \times 10^{-3}$ applied \textit{exclusively} to the CNN Backbone parameters.
    \item \textit{Base Dirichlet Prior ($\boldsymbol{\alpha}_0$):} $\mathbf{1}$ (a tensor of ones).
    \item \textit{Inference Entropy Threshold ($\eta$):} $-4.5$ (used for full dataset evaluation to enforce deep traversal on intersections; -0.5 for the single OOD trace).
    \item \textit{OOD Abstention Threshold ($\tau_{conf}$):} $10.0$
\end{itemize}

\paragraph{Loss Formulation (NBSR Models):}
the training objective is a composite of three terms:
\begin{equation}
    \mathcal{L}_{total} = \mathcal{L}_{NLL} + \mathcal{L}_{reg} + \lambda_{route} \mathcal{L}_{route}
\end{equation}

\begin{enumerate}
    \item \textit{Negative Log-Likelihood ($\mathcal{L}_{NLL}$):} computed using the expected probability of the final Dirichlet distribution: $\mathbf{p} = \boldsymbol{\alpha}_T / \sum \boldsymbol{\alpha}_T$.
    \item \textbf{Auxiliary Concept Loss ($\mathcal{L}_{route}$):} a standard Cross-Entropy loss applied to the Root Router logits to enforce semantic mapping. We map the optimal action $a^*_t \in \{0, 1, 2, 3\}$ to an abstract target mode $m^*_t \in \{0, 1\}$ corresponding to the hierarchical routing tree. Specifically, $m^*_t = 0$ if $a^*_t \in \{0, 1\}$ (Cruising/Straight/Reverse) and $m^*_t = 1$ if $a^*_t \in \{2, 3\}$ (Evasion/Left/Right). The loss is defined as:
    \begin{equation}
        \mathcal{L}_{route} = - \sum_{j=0}^{1} \mathbb{I}(m^*_t = j) \log(p_{root, j})
    \end{equation}
    where $p_{root, j}$ is the predicted probability of mode $j$ from the Root Router. The weighting factor was set to $\lambda_{route} = 0.1$.
    \item \textit{Masked Evidential Regularization ($\mathcal{L}_{reg}$):} to penalize hallucinated evidence without causing gradient starvation, we utilized the target-masked regularizer \cite{sensoy2018evidential}:
    \begin{equation}
        \mathcal{L}_{reg} = \lambda_{reg}^{(e)} \frac{1}{K} \sum_{i=1}^K (\alpha_{T, i} - 1) (1 - y_i)
    \end{equation}
    where $\mathbf{y}$ is the one-hot encoded target action. To allow the GRU time to form initial memory pathways, we applied Evidential Annealing, linearly ramping the regularizer weight $\lambda_{reg}^{(e)}$ from $0.0$ to $0.02$ over the first 10 epochs.
\end{enumerate}

\section{Experimental Details for BOED Active Clinical Triage}
\label{app:BOED_details}

Here we provide the technical specifications and hyperparameter configurations required to reproduce the Active Clinical Triage experiments presented in Section \ref{subsec:nbsr_boed}.

\subsection{Computing Environment}
All experiments were executed on a cloud-based virtual machine instance (KVM full virtualization) to specifically isolate and measure CPU-bound inference efficiency. The exact hardware specifications are as follows:
\begin{itemize}
    \item \textit{Processor:} Intel(R) Xeon(R) CPU @ 2.20GHz, x86\_64 architecture (1 Physical Core, 2 Logical Threads) with 55 MiB L3 cache.
    \item \textit{System Memory (RAM):} 13.61 GB
    \item \textit{Storage:} 242.49 GB allocated disk space
    \item \textit{Hardware Accelerator:} None (Strict CPU-only execution; no GPU utilized)
\end{itemize}

\subsection{Dataset and Preprocessing}
We utilized the Kaggle Disease Symptom Prediction dataset (same as used in Section \ref{subsec:medical_diagnosis}), mapping clinical symptom profiles to 41 distinct pathologies.
\begin{itemize}
    \item \textit{Input Features ($X$):} 132 binary indicators representing symptom presence (1) or absence (0). 
    \item \textit{Partitioning:} Features were partitioned into 12 base demographic features (Always visible) and 5 Diagnostic Test Panels, each containing 24 task-specific symptom indicators.
    \item \textit{Clinical Noise:} A uniform 5\% noise mask was applied to the feature matrix, where binary states were stochastically inverted ($0 \leftrightarrow 1$).
\end{itemize}

\subsection{Network Architectures}
\paragraph{AR-NBSR Architecture.}
\begin{itemize}
    \item \textit{Global Router:} An MLP processing the concatenated state $[\mathbf{h}_x, \bm{\alpha}_t, \mathbf{m}_t]$, where $\mathbf{m}_t$ is the one-hot history of previously queried panels. Architecture: Linear(178 $\to$ 128) $\to$ LayerNorm $\to$ ReLU $\to$ Linear(128 $\to$ 5).
    \item \textit{Expert Networks (Panels):} Independent feature extractors: Linear(36 $\to$ 128) $\to$ LayerNorm $\to$ ReLU $\to$ Linear(128 $\to$ 128) $\to$ ReLU $\to$ Linear(128 $\to$ 41) $\to$ Softplus.
\end{itemize}

\subsection{Optimization and Training Dynamics}
Networks were optimized using Adam \cite{kingma2017Adam} ($\text{lr} = 10^{-3}$) for 200 epochs with a batch size of 128. A \textit{ReduceLROnPlateau} scheduler ($\text{factor}=0.5, \text{patience}=10$) was employed to stabilize convergence.

\paragraph{Hyperparameters.}
\begin{itemize}
    \item \textit{Entropy Penalty ($\lambda$):} $1 \times 10^{-4}$.
    \item \textit{Budget Penalty ($\gamma$):} $0.005$.
    \item \textit{Gumbel-Softmax Temperature ($\tau$):} Exponentially annealed from $1.0$ to $0.1$.
    \item \textit{Training Early Exiting ($\eta$):} $-150.0$.
    \item \textit{Evaluation Sweep ($\eta$):} For the Pareto Frontier, we evaluated a dense grid of 30 linearly spaced confidence thresholds $\eta \in [-110.0, -400.0]$.
\end{itemize}

\end{document}